\documentclass{article}
\usepackage[utf8]{inputenc}
\usepackage{amsmath} 
\usepackage{amsthm}
\usepackage{bbm}
\usepackage{xspace}
\usepackage{caption}
\usepackage{amssymb}
\usepackage{mathtools}
\usepackage{subcaption}
\usepackage{enumitem}
\usepackage[ruled,vlined]{algorithm2e}
\usepackage{xcolor}
\usepackage{graphicx}
\usepackage{nicefrac}
\usepackage{array}
\usepackage[toc,page]{appendix}
\usepackage{titlesec}
\usepackage{hyperref}  
\usepackage{accents}

\newtheorem{theorem}{Theorem}
\newtheorem{proposition}{Proposition}
\newtheorem{corollary}{Corollary}
\newtheorem{lemma}[theorem]{Lemma}
\theoremstyle{definition}
\newtheorem{definition}{Definition}
\theoremstyle{remark}

\allowdisplaybreaks

\newcommand{\calT}{\mathcal{T}}

\newcommand{\calM}{\mathcal{M}}

\newcommand{\calX}{\mathcal{X}}
\newcommand{\calY}{\mathcal{Y}}
\newcommand{\calF}{\mathcal{F}}

\newcommand{\calD}{\mathcal{D}}

\newcommand{\calI}{\mathcal{I}}

\newcommand{\quant}{Q}

\newcommand{\abs}[1]{\left\lvert#1\right\rvert}
\newcommand{\norm}[2]{\left\lVert#2\right\rVert_{#1}}

\DeclareMathOperator*{\argmax}{arg\,max}
\newcommand{\Real}{\mathbb{R}}

\newcommand{\indicator}[1]{\mathbbm{1}\curlybrack{#1}}

\newcommand{\roundbrack}[1]{\left( #1 \right)}
\newcommand{\curlybrack}[1]{\left\lbrace #1 \right\rbrace}
\newcommand{\squarebrack}[1]{\left\lbrack #1 \right\rbrack}

\newcommand{\Prob}{\mathbb{P}}
\newcommand{\Exp}[2]{\mathbb{E}_{#1}\left\lbrack#2\right\rbrack}

\usepackage[english]{babel}
\usepackage[margin=1in]{geometry}

\usepackage[square,numbers,sort&compress]{natbib}

\title{Distribution-free uncertainty quantification\\ for classification under label shift}
\author{Aleksandr Podkopaev$^{1,2}$, Aaditya Ramdas$^{1,2}$\\ 
	Department of Statistics \& Data Science$^1$\\
	Machine Learning Department$^2$\\
	Carnegie Mellon University\\
\texttt{\{podkopaev,aramdas\}@cmu.edu}}
\date{\today}

\begin{document}
\maketitle

\begin{abstract}
  Trustworthy deployment of ML models requires a proper measure of uncertainty, especially in safety-critical applications. We focus on uncertainty quantification (UQ) for classification problems via two avenues --- prediction sets using conformal prediction and calibration of probabilistic predictors by post-hoc binning --- since these possess distribution-free guarantees for  i.i.d. data. Two common ways of generalizing beyond the i.i.d. setting include handling \emph{covariate} and \emph{label} shift. Within the context of distribution-free UQ, the former has already received attention, but not the latter. It is known that label shift hurts prediction, and we first argue that it also hurts UQ, by showing degradation in coverage and calibration. Piggybacking on recent progress in addressing label shift (for better prediction), we examine the right way to achieve UQ by reweighting the aforementioned conformal and calibration procedures whenever some unlabeled data from the target distribution is available. We examine these techniques theoretically in a distribution-free framework and demonstrate their excellent practical performance. 
\end{abstract}

\tableofcontents

\section{Introduction}
It is common in classification to assume access to labeled data $\curlybrack{(X_i,Y_i)}_{i=1}^n$ where $X_i\in \calX$, $Y_i\in\calY=\curlybrack{1,\dots,K}$ denote the covariates, or features, and the labels respectively, and the pairs $(X_i,Y_i)$, $i =1,\dots,n$ are sampled i.i.d. from some unknown joint distribution $P$ over $\calX\times\calY$. Such dataset is used to learn a predictor $f$, a mapping from $\calX$ to rankings or distributions over $\calY$, by optimizing some loss/risk. However, accurate point prediction alone can be insufficient in certain applications, e.g., medical diagnosis, where trustworthy deployment of a model requires a valid measure of uncertainty associated with corresponding predictions. 

Common prediction models are mappings of the form $f:\calX\to \Delta_K$, where $\Delta_K$ refers to the probability simplex in $\mathbb{R}^K$, and a prediction on a new (test) point $X\in\calX$ is performed by picking the top-ranked class according to $f(X)$. One hopes that the output vector $f(X)$ reflects the true conditional probabilities of classes given the observed input, but this won't be true without additional distributional and modeling assumptions, that are typically strong and unverifiable in practice. In this work, we focus on two categories of post-processing procedures ---  calibration via post-hoc binning and conformal prediction --- that use held-out data (referred to as \emph{calibration} dataset) and a trained model to construct a corresponding \emph{wrapper} that provably quantifies predictive uncertainty when no distributional assumptions are made about the data generating mechanism. (This generality comes at a certain price which we discuss further.)

We work in the context of \emph{distribution-free} uncertainty quantification and, in particular, focus on producing prediction sets (Section~\ref{sec:conf_pred_sets}) and calibrated probabilities (Section~\ref{sec:calibration}), which are complementary approaches for classifier UQ. While the former aims to produce a set of labels that contains the truth with high probability, the latter aims to amend the output of a probabilistic predictor so that it has a rigorous frequentist interpretation. It is useful to view the task through the lens of how actionable the corresponding notion is in a given setup. For example, in a binary classification setup with only 4 possible prediction sets $\curlybrack{\emptyset,\curlybrack{1},\curlybrack{2},\curlybrack{1,2}}$, if we were to observe prediction sets $\curlybrack{1,2}$ for large fraction of data points, one might end up quite disappointed. Thus, calibration could be a better way of quantifying uncertainty in the binary case. However, mathematical guarantees on calibration degrade with growing number of classes, but the aforementioned prediction sets become an attractive option with more labels. To summarize, neither of two notions provide a complete answer to the question of UQ for classification on their own, but together they represent two of the more principled distribution-free approaches towards UQ that are practically efficient and theoretically grounded.

In real-world applications, the \emph{target} distribution (generating test data) might not be the same as the \emph{source} distribution (generating training data) which can both hurt a model's generalization and lead to violation of the assumptions under which even assumption-lean UQ is valid. As meaningful reasoning about uncertainty on the target domain is hopeless without any additional information about the type of distribution shift, one may hope that it is possible to make simplifying assumptions which would allow us to perform appropriate corrections and construct procedures with non-trivial guarantees. Let $P,Q$ stand for the source and target distributions defined on $\calX\times\calY$, with $p,q$ being the PDFs or PMFs associated with $P$ and $Q$ respectively. Two common assumptions about the type of shift include \emph{covariate shift}~\citep{shimodaira2000improving}: $q(x)\neq p(x)$ but $q(y\mid x)=p(y\mid x)$, and \emph{label shift}~\citep{saerens2002adjusting}: $q(y)\neq p(y)$ but $q(x\mid y)=p(x\mid y)$. Both assumptions allow for a tractable interpretation when viewing the data generating process as a causal or anti-causal model respectively. For example, label shift is a reasonable assumption in medical applications where diseases $(Y)$ cause symptoms $(X)$: it is intuitive that some sort of correction might be required when a predictor trained in ordinary conditions is deployed during extreme ones, e.g., during a pandemic. 

Classic approaches for handling the aforementioned shifts make an assumption that the target support is contained in the source support, so that the covariate or label likelihood ratios (or \emph{importance weights}) $q(x)/p(x)$ or $q(y)/p(y)$ are well-defined. In applications, true weights are never known exactly, so the construction of consistent estimators has received a lot of attention in the ML community. For label shift dominant approaches that are still computationally feasible in modern high-dimensional regimes, and that perform estimation using labeled data only from the source distribution, include: (a) Black Box Shift Estimation (BBSE)~\citep{lipton2018label} and related Regularized Learning under Label Shift (RLLS)~\citep{azizzadenesheli2018label}, (b) Maximum Likelihood Label Shift (MLLS) and its variants ~\citep{saerens2002adjusting,alexandari2020mle}.

Within the context of distribution-free UQ, covariate shift has recently received attention. Focusing on regression, \citet{tibs2019conf} generalize construction of conformal prediction intervals to handle the case of known covariate likelihood ratio, and empirically demonstrate that the modified procedure works reasonably well with a plug-in estimator for the importance weights. For binary classification, \citet{gupta2020df_calib} propose a way of calibrating probabilistic predictors under covariate shift, and quantify miscalibration of the resulting estimator.

In this work, we close an existing gap for quantifying predictive uncertainty under label shift. Building on recent results about distribution-free calibration and (split-)conformal prediction, we adapt both to handling label shift through an appropriate form of reweighting. While typical application of those frameworks requires labeled data from the target to provide guarantees, we show that under reasonable assumptions one can still reason about uncertainty on the target even if only unlabeled data is available. In contrast to covariate shift where we observe $X$ and need the covariate likelihood ratio of $X$ to reweight, under label shift we observe $X$ but need the likelihood ratio of $Y$ to reweight. We also consider an alternative way of addressing label shift by performing label-conditional conformal classification~\citep{vovk2005algorithmic,vovk2016criteria,sadinle2019label,guan2019prediction}. 

\section{Conformal classification}\label{sec:conf_pred_sets}
We begin with the notion of prediction sets as a way of  quantifying predictive uncertainty. Formally, we wish to construct an uncertainty set function $C: \calX\to 2^\calY$, such that for a new (test) data point we can guarantee that:
\begin{equation}
\label{eq:def_pred_set}
    \Prob \roundbrack{Y_{n+1}\in C(X_{n+1})}\geq 1-\alpha.
\end{equation}
Conformal prediction~\citep{vovk2005algorithmic} has received attention recently both in regression~\citep{lei2018df_pred_inference,romano2019cqr, barber2021predictive} and classification~\citep{cauchois2020knowing,romano2020classification,angelopoulos2021classification} settings. It does not require making any distributional assumptions, which comes at the price of provably providing only \emph{marginal} guarantees as stated in~\eqref{eq:def_pred_set} which should be contrasted with possibly the ultimate goal of obtaining prediction sets with guarantees conditional on a given input.

Since conditional guarantees often require making restrictive and unverifiable assumptions, we instead focus on procedures that might provably provide marginal coverage guarantees but still tend to demonstrate good conditional coverage empirically. Being flexible, conformal prediction allows to proceed with both probabilistic and scoring classifiers. Within this framework, one usually defines a non-conformity score, a higher value of which on a given data point indicates that it is more `atypical'. For example, even if a classifier outputs only the ranking of predicted classes, a rank of the true class defines a valid non-conformity score. 
Keeping in mind that our techniques extend to other types of classifiers, we nevertheless focus on probabilistic predictors in this work which are also dominant in modern machine learning.
\subsection{Exchangeable conformal}\label{subsec:exch_conformal}

Consider a sequence of candidate nested prediction sets $\curlybrack{\calF_\tau(x)}_{\tau\in\calT}$: $\calF_{\tau_1}(x)\subseteq \calF_{\tau_2}(x)\subseteq\calY$ for any $\tau_1\leq \tau_2\in\calT$, with $\calF_{\inf\calT}=\emptyset$ and $\calF_{\sup\calT}=\calY$~\citep{gupta2019nested}. 
For any point $(x,y)\in\calX\times\calY$ define
\begin{equation}
\label{eq:nested_set_radius}
    r(x,y):= \inf\curlybrack{\tau\in\calT: y\in \calF_\tau(x)},
\end{equation}
as the smallest radius of the set in a sequence $\curlybrack{\calF_\tau(x)}_{\tau\in\calT}$ that captures $y$. Within split-conformal framework, available dataset is split at random into two parts: the first is used to construct a nested sequence and the second is used to select the smallest $\tau^\star$ that guarantees validity. 


If the true class-posterior distribution $\pi_y(x) = \Prob\squarebrack{Y=y\mid X=x}$ is known, the optimal prediction set for any $x\in\calX$ with conditional coverage guarantee is based on the corresponding density level sets~\citep{vovk2005algorithmic,lei2013distribution,gupta2019nested,sadinle2019label}: one should pick the largest $\tau_\alpha(x)$ and include all labels with probabilities $\pi_y(x)$ exceeding $\tau_\alpha(x)$ so that the corresponding total probability mass is at least $1-\alpha$. When ties are present, such procedure can yield conservative sets, e.g., if for some $x\in\calX$ all classes are equally probable in a 10-class problem, then $\tau_\alpha(x)=0.1$ and the proposed set would simply be $\calY$. For the discussion that follows we assume that there are no ties or that they are broken as formally discussed in Appendix~\ref{appsubsec:tie_break}. Then, to construct the optimal prediction set, one should start with an empty one and keep including labels as long as the total probability mass of labels included before is less than $1-\alpha$. Formally,
\begin{align}
    C^{\mathrm{oracle}}_\alpha(x) & :=\curlybrack{y\in\calY: \rho_y(x;\pi)< 1-\alpha}, \label{eq:oracle_nonrand_pred_set_equiv_def}\\
    \text{where} \quad  \rho_y(x;\pi) & := \sum_{y'=1}^K \pi_{y'}(x) \indicator{\pi_{y'}(x)> \pi_y(x)} \nonumber
\end{align}
is the total probability mass of labels that are more likely than $y\in\calY$. Notice that for any $x\in \calX$ and the corresponding most likely label $y^\star$ it holds that  $\rho_{y^\star}(x;\pi)=0$. When an estimator $\widehat{\pi}$ of the true conditional distribution is used,
split-conformal framework provides a way of updating the threshold $1-\alpha$ in~\eqref{eq:oracle_nonrand_pred_set_equiv_def} in order to retain coverage guarantees. However, naive conformalization of the nested sequence suggested by the form~\eqref{eq:oracle_nonrand_pred_set_equiv_def} yields prediction sets with correct marginal coverage but typically inferior conditional coverage in practice. Due to that reason and a desire of consistency, i.e., recovering the oracle prediction sets from the conformal ones in the limit, we instead use a randomized version of~\eqref{eq:oracle_nonrand_pred_set_equiv_def} defined as
\begin{equation}\label{eq:oracle_rand_ps}
    \widetilde{C}^{\mathrm{oracle}}_\alpha(x)=\curlybrack{y: \rho_y(x;\pi)+u\cdot \pi_y(x) \leq 1-\alpha},
\end{equation}
where $u$ is a realization of $\mathrm{Unif}\roundbrack{[0,1]}$, sampled independently of anything else~\citep{vovk2005algorithmic,romano2020classification}. Note that replacing strict inequality by a non-strict does not expand the prediction set as equality happens with zero probability and that induced randomization can result in exclusion only of a single label from the set $C^{\mathrm{oracle}}_\alpha(x)$. 
The form of the optimal prediction sets~\eqref{eq:oracle_rand_ps} suggests to consider the following nested sequence:
\begin{equation}\label{eq:nested_ps_seq}
    \calF_{\tau}(x,u;\widehat{\pi}) = \curlybrack{y\in\calY: \rho_y(x;\widehat{\pi})+u\cdot \widehat{\pi}_y(x)\leq \tau},
\end{equation}
for $\tau\in\calT=[0,1]$. Then for any triple $(X,Y,U)$ the corresponding radius~\eqref{eq:nested_set_radius}, or score, is given by
\begin{align}
    r(X,Y,U;\widehat{\pi}) & = \inf\curlybrack{\tau\in \calT:\rho_Y(X;\widehat{\pi})+U\cdot \widehat{\pi}_Y(X)\leq \tau} \nonumber \\
    & = \rho_Y(X;\widehat{\pi})+U\cdot \widehat{\pi}_Y(X). \label{eq:conf_score_eqiv}
\end{align}

Adapting to label shift can be performed with other non-conformity scores proposed recently for conformal classification~\citep{cauchois2020knowing,angelopoulos2021classification}, and we further discuss the subtleties behind our choice in Appendix~\ref{appsubsec:randomization}. Assume that the dataset is split at random into two parts: training $\curlybrack{\roundbrack{X_i,Y_i}}_{i\in \calI_1}$ and calibration $\curlybrack{\roundbrack{X_i,Y_i}}_{i\in \calI_2}$, where for simplicity the calibration data points are indexed as $\calI_2= \curlybrack{1,\dots,n}$. When the data are exchangeable, the non-conformity scores $r_i=r(X_i,Y_i, U_i;\widehat{\pi})\in [0,1]$, $i\in\calI_2\cup\curlybrack{n+1}$ are exchangeable as well, which in turn implies that the prediction set
\begin{align}
    \calF_{\tau^\star} \roundbrack{x,u;\widehat{\pi}} & = \curlybrack{y\in\calY:  \rho_y(x;\widehat{\pi})+u\cdot \widehat{\pi}_y(x)\leq \tau^\star}, \nonumber\\
    \tau^\star & = \quant_{1-\alpha} \roundbrack{\curlybrack{r_i}_{i\in\calI_2}\cup \curlybrack{1}}, \label{eq:exchan_conf_pred_set}
\end{align}
does attain the right coverage guarantee\footnote{$\quant_\beta \roundbrack{F} := \inf \curlybrack{z: F(z)\geq \beta}$ is $\beta$-quantile of a distribution $F$. For a multiset $\curlybrack{z_1,\dots,z_m}$ we write $\quant_\beta \roundbrack{\curlybrack{z_1,\dots,z_m}} := \quant_\beta \roundbrack{ \frac{1}{m}\sum_{i=1}^m \delta_{z_i}}$, where $\delta_a$ is a point-mass distribution at $a$, to denote quantiles of the corresponding empirical distribution.}. This is a classic result in conformal prediction and represents a simple fact about quantiles of exchangeable random variables, stated next for completeness. 
\begin{theorem}\label{thm:guarantee_exch}
If $\curlybrack{(X_i,Y_i)}_{i=1}^{n+1}$ are exchangeable, then:
\begin{equation*}
    \Prob(Y_{n+1}\in \calF_{\tau^\star} \roundbrack{X_{n+1},U_{n+1};\widehat{\pi}} \mid \curlybrack{\roundbrack{X_i,Y_i}}_{i\in \calI_1})\geq 1-\alpha.
\end{equation*}
Further, if the non-conformity scores are almost surely distinct, then the above probability is upper bounded by $1-\alpha + 1/(n+1)$.
\end{theorem}
The proof is given in Appendix~\ref{appsubsec:conformal_proofs}. Notice that the randomized sequence~\eqref{eq:nested_ps_seq} might yield empty, and thus non-actionable prediction sets, which is the consequence of deploying randomization only. Substituting the condition in~\eqref{eq:exchan_conf_pred_set} with $\indicator{\rho_y(x;\widehat{\pi})>0}\cdot \roundbrack{\rho_y(x;\widehat{\pi})+u\cdot \widehat{\pi}_y(x)}$ ensures that
the prediction set always includes the most likely label. Such a construction trivially inherits the coverage guarantee stated in Theorem~\ref{thm:guarantee_exch}, and we refer the reader to Appendix~\ref{appsubsec:randomization} for further details.


\subsection{Label-shifted conformal}\label{subsec:conf_label_shift} 
To illustrate the necessity of accounting for label shift we consider the following toy classification task with 3 classes $\calY=\curlybrack{1,2,3}$ where class proportions are given as $p = \roundbrack{0.1, 0.6, 0.3}$ and $q = \roundbrack{0.3, 0.2, 0.5}$, and for each data point the covariates are sampled according to $X\mid Y=y\sim \mathcal{N}(\mu_y,\Sigma)$ where $\mu_1 = \roundbrack{-2;0}^\top$, $\mu_2 = \roundbrack{2;0}^\top$, $\mu_3 = \roundbrack{0;2\sqrt{3}}^\top$, $ \Sigma= \mathrm{diag}(4,4)$. First, we perform the standard routine for constructing split-conformal prediction sets for a single draw of data from the source and target distributions using the Bayes-optimal rule as an underlying predictor. We illustrate a single draw of the test data on Figure~\ref{subfig:conf_toy_1_data_target_test} and the resulting prediction sets on Figure~\ref{subfig:conf_toy_1_sets_target_test}. Next, we repeat the simulation 1000 times and track empirical coverage on the test set. Results on Figure~\ref{subfig:conf_toy_1_coverage_fixed} demonstrate the necessity of correcting for label shift as the classic conformal prediction sets introduced in Section~\ref{subsec:exch_conformal} fail to achieve the correct marginal coverage. 
\begin{figure*}
    \centering
    \begin{subfigure}{0.45\textwidth}
        \centering
       \includegraphics[width=\textwidth]{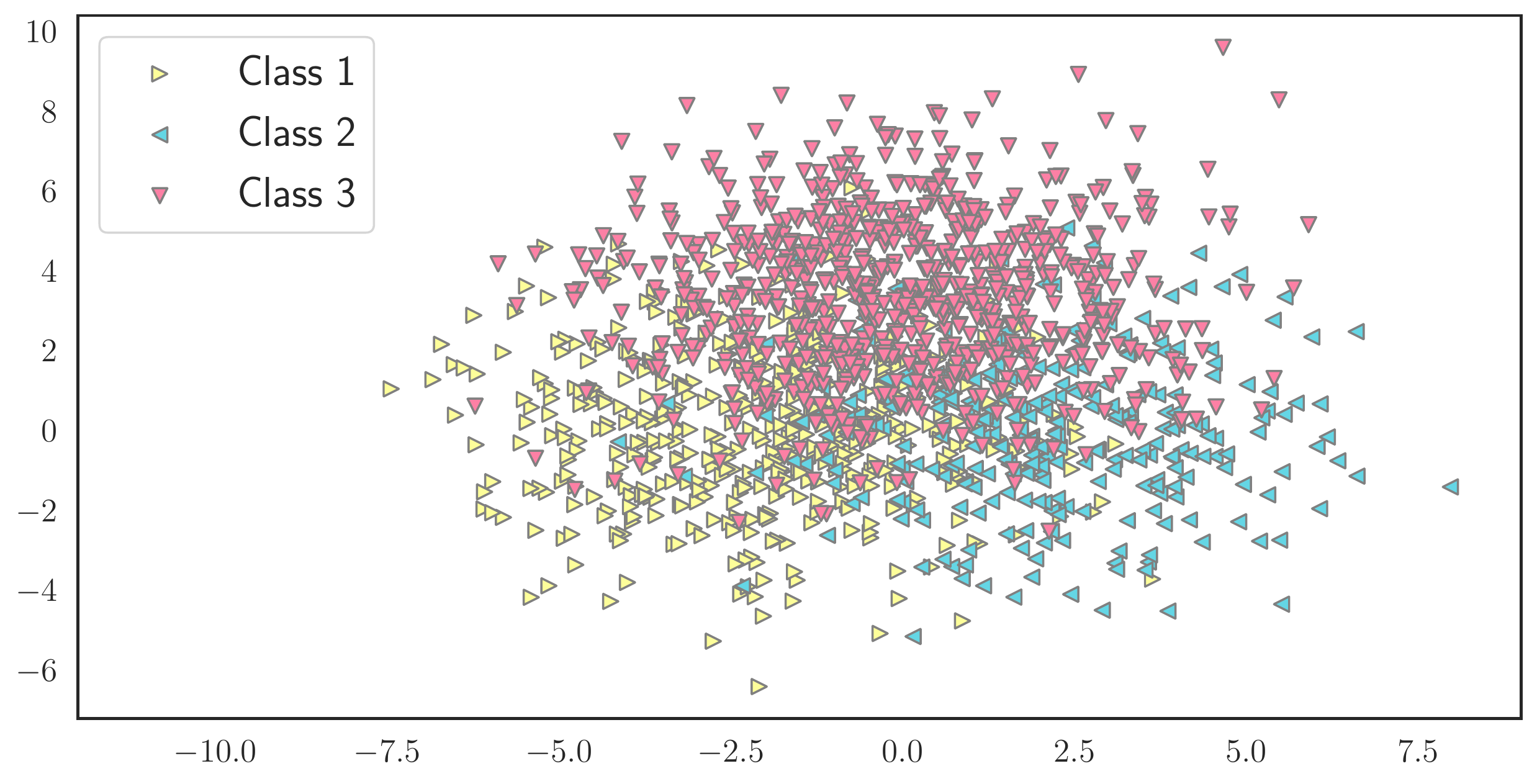}
    \caption{}
    \label{subfig:conf_toy_1_data_target_test}
    \end{subfigure}%
    ~ 
    \begin{subfigure}{0.45\textwidth}
        \includegraphics[width=\textwidth]{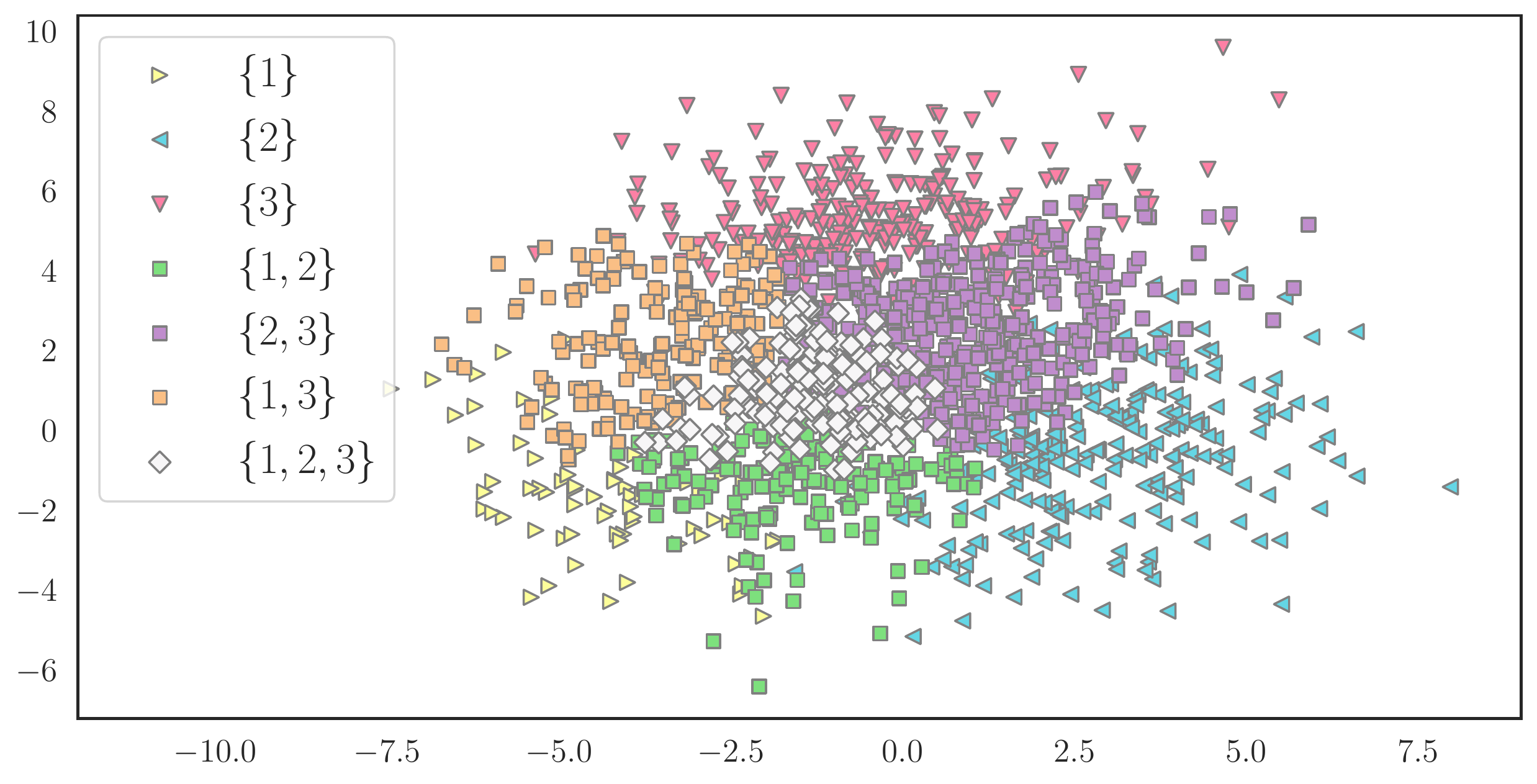}
    \caption{}
    \label{subfig:conf_toy_1_sets_target_test}
    \end{subfigure}
    ~
    \begin{subfigure}{0.45\textwidth}
        \centering
        \includegraphics[width=\textwidth]{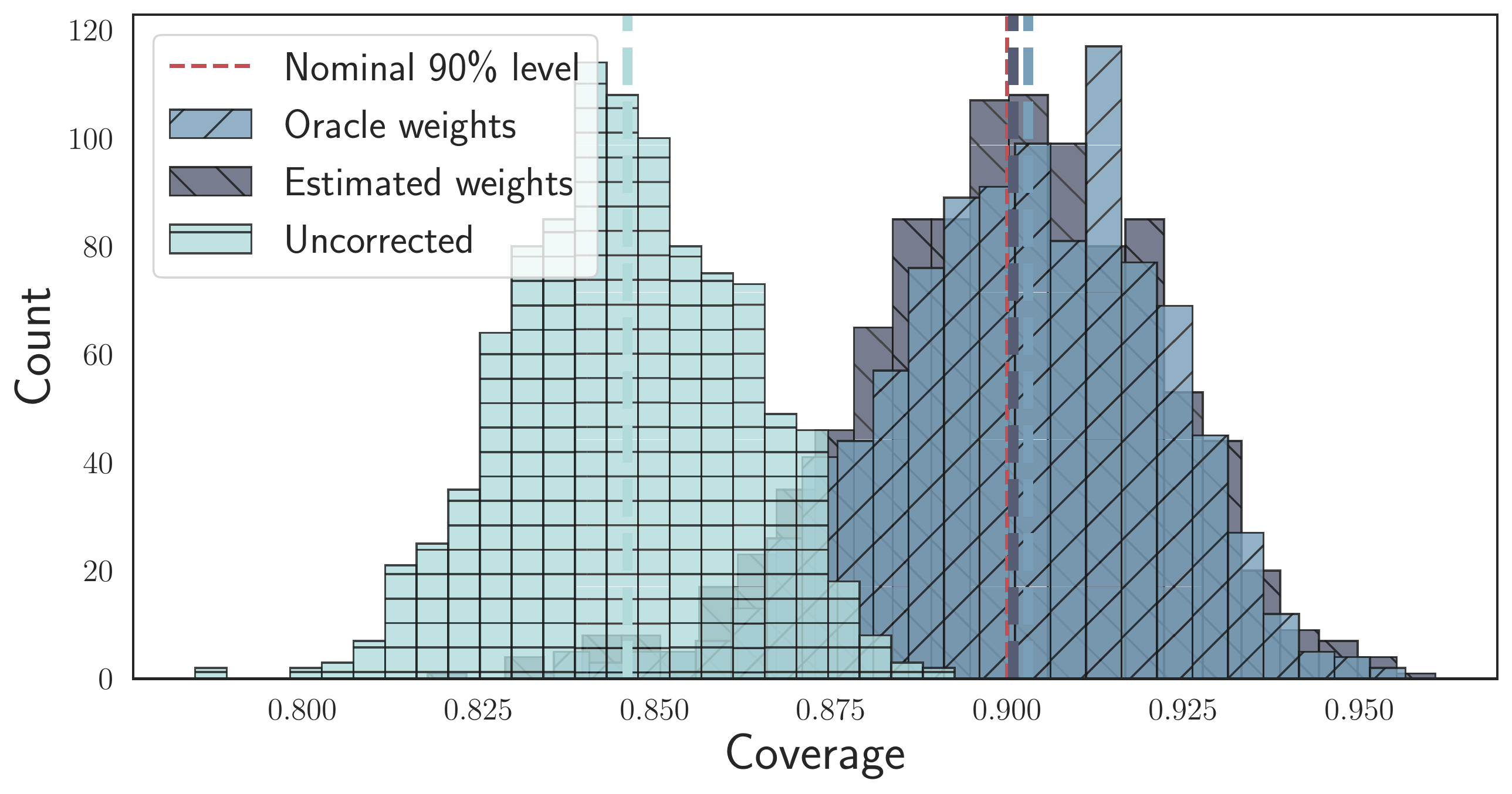}
    \caption{}
    \label{subfig:conf_toy_1_coverage_fixed}
    \end{subfigure}%
    ~ 
    \begin{subfigure}{0.45\textwidth}
        \centering
        \includegraphics[width=\textwidth]{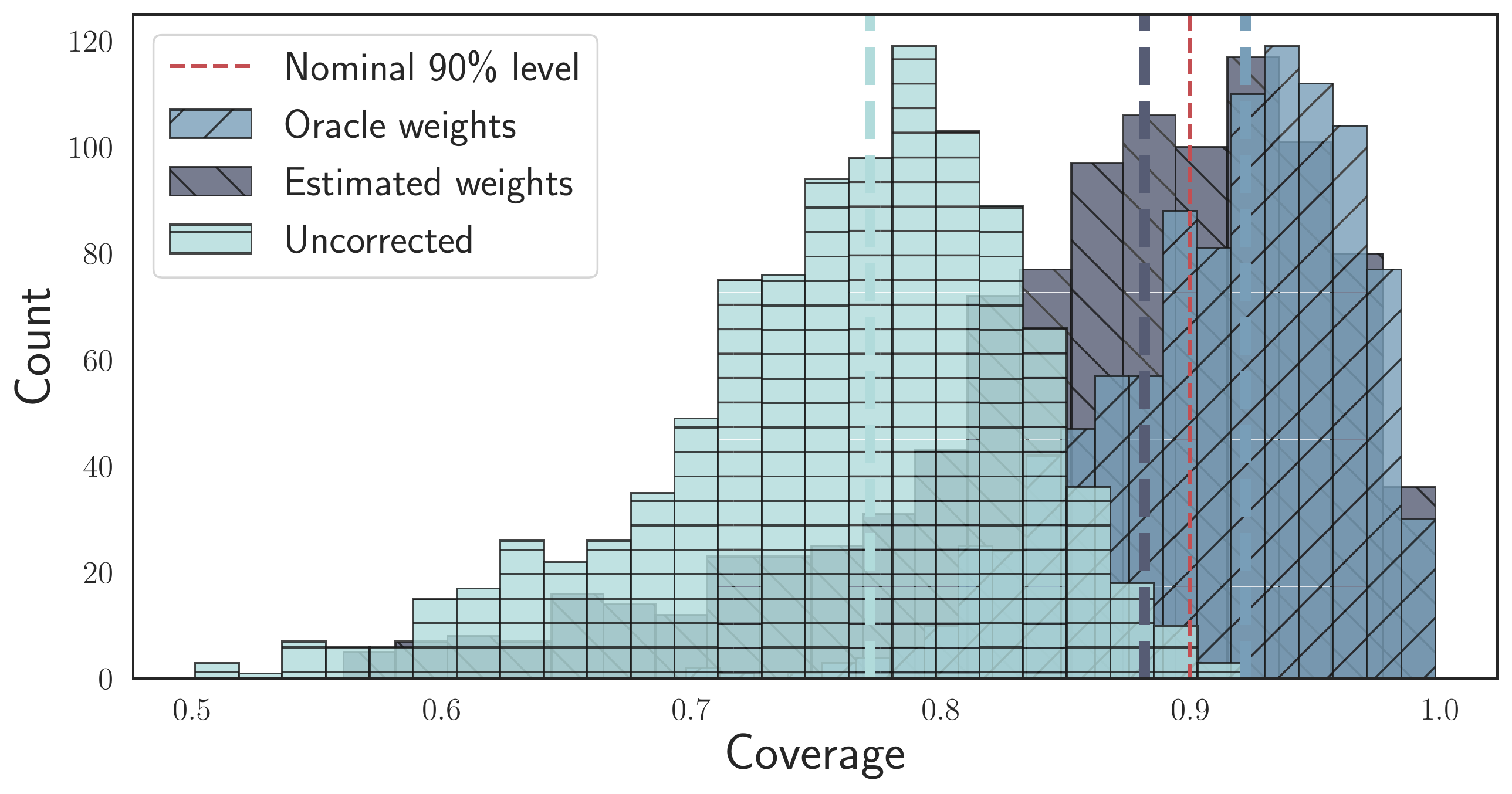}
    \caption{}
    \label{subfig:conf_toy_1_coverage_real}
    \end{subfigure}
    \caption{(\subref{subfig:conf_toy_1_data_target_test}) Test data sample for the toy simulation in Section~\ref{subsec:conf_label_shift}. (\subref{subfig:conf_toy_1_sets_target_test}) Corresponding conformal prediction sets when label shift is accounted for with oracle importance weights. (\subref{subfig:conf_toy_1_coverage_fixed}) Empirical coverage on shifted data for the toy simulation in Section~\ref{subsec:conf_label_shift}. (\subref{subfig:conf_toy_1_coverage_real}): Empirical coverage on the \texttt{wine quality} dataset. Dashed vertical lines describe the median coverage values, which are significantly worse when label shift is not accounted for, while using estimated weights mimics the oracle reasonably well.}
    \label{fig:conf_toy_example_1}
\end{figure*}

Assume that the true likelihood ratios $w(y)=q(y)/p(y)$ are known for all $y\in\calY$. In order to obtain provably valid prediction sets, we consider instead:
\begin{align}\label{eq:weighted_pred_set}
\calF^{(w)}_{\tau^\star} \roundbrack{x,u;\widehat{\pi}} & = \curlybrack{y\in\calY: \rho_y\roundbrack{x;\widehat{\pi}} +u\cdot \widehat{\pi}_y(x)\leq \tau_{w}^\star(y)}, \nonumber \\
\tau^\star_{w}(y) & = \quant_{1-\alpha}\roundbrack{\sum_{i=1}^{n} \tilde{p}_i^w(y)\delta_{r_i}+ \tilde{p}_{n+1}^w(y)\delta_1},\\
\text{where} \quad    \tilde{p}_i^w(y) &  = \frac{w(Y_{i})}{\sum_{j=1}^{n} w(Y_j)+w(y)},\quad i=1,\dots, n, \nonumber\\
    \tilde{p}_{n+1}^w(y) & = \frac{w(y)}{\sum_{j=1}^{n} w(Y_j)+w(y)}. \label{eq:threshold_weighted_case}
\end{align}
In addition to the fact that the empirical distribution used to calibrate the threshold in~\eqref{eq:weighted_pred_set} is different from the one used in exchangeable setting~\eqref{eq:exchan_conf_pred_set}, notice that the thresholds themselves now vary depending on the class label. The formal guarantee for the prediction set~\eqref{eq:weighted_pred_set} is stated next.
\begin{theorem}\label{thm:oracle_imp_weights_ps}
For any $\alpha\in (0,1)$, if the true likelihood ratios $w(y)=q(y)/p(y)$ are known for all $y\in\calY$, it holds that
\begin{equation*}
    \Prob (Y_{n+1}\in \calF^{(w)}_{\tau^\star} \roundbrack{X_{n+1}, U_{n+1};\widehat{\pi}} | \curlybrack{(X_i,Y_i)}_{i\in\calI_1})\geq 1-\alpha.
\end{equation*}
\end{theorem}
The proof is given in Appendix~\ref{appsubsec:conformal_proofs}. It relies on the concept of \emph{weighted exchangeability} introduced by \citet{tibs2019conf} to handle covariate shift in regression, and we adapt those ideas here to correct for label shift in classification. Returning to the example considered in the beginning of this section, Figure~\ref{subfig:conf_toy_1_coverage_fixed} illustrates that calibrating the threshold $\tau$ as in~\eqref{eq:weighted_pred_set} with either oracle or estimated importance weights allows to achieve the target marginal coverage. Here we use BBSE~\citep{lipton2018label} to estimate the importance weights; more details are provided in Appendix~\ref{appsec:imp_weight_estimation}. 

Next, we perform a similar experiment with the \texttt{wine quality} dataset~\citep{data_wine}. We refer the reader to Appendix~\ref{appsubsec:conformal_real_data} for details regarding data pre-processing and modeling steps. The source and target class proportions are taken to be $p = (0.1,0.4,0.5)$ and $q=(0.4,0.5,0.1)$ and the data are resampled accordingly. Using a shallow multilayer perceptron as an underlying predictor and BBSE for importance weights estimation, at each iteration we repeat the routine for random splits of the original dataset and compare empirical coverage for different conformal prediction sets. Marginal coverage results given in Figure~\ref{subfig:conf_toy_1_coverage_real} support the idea that both shift-corrected conformal prediction sets demonstrate superior coverage performance compared with uncorrected ones. While conformal sets with oracle importance weights closely match the nominal coverage level, sets that proceed with estimated ones have a slightly downgraded performance. Arising basically due to an imperfect classification model and an imperfect importance weight estimation procedure, it highlights an important issue we discuss next.

While (weighted) exchangeability arguments yield a coverage guarantee in case of known importance weights, in practice one only has access to a corresponding estimator. Dominant methods, which we briefly touch upon in Appendix~\ref{appsec:imp_weight_estimation}, estimate importance weights using a separate labeled dataset from the source distribution and unlabeled dataset from the target. Under reasonable assumptions, such as identifiability and boundedness of the true importance weights, these estimators are known to be consistent as the size of both samples grows. For succinctness, we write $k=\abs{\calD_{\text{est}}}$ to denote the \emph{total} size of the datasets used for constructing an estimator $\widehat{w}_k$ of the importance weights $w$.

\begin{corollary}\label{cor:asympt_est_imp_weights}
Fix $\alpha\in (0,1)$. Assume that $\widehat{w}_k$ is a consistent estimator of $w$. Further, assume that for the true $w$ and all $y\in\calY$, the discrete distribution in~\eqref{eq:weighted_pred_set} does not have a jump at level $1-\alpha$. Then:
\begin{equation*}
    \lim_{k\rightarrow\infty} \Prob\roundbrack{Y_{n+1}\in \calF^{(\widehat{w}_k)}_{\tau^\star} \roundbrack{X_{n+1}, U_{n+1};\widehat{\pi}} } \geq 1-\alpha.
\end{equation*}
\end{corollary}
The proof is given in Appendix~\ref{appsubsec:conformal_proofs}. To demonstrate why presence of a jump might cause problems, consider a simplified example. Let $Z\sim\mathrm{Ber}(p)$ for which the quantile corresponding to any given level $\alpha$ is given by
\begin{equation*}
\quant_\alpha \roundbrack{(1-p) \cdot \delta_0+p\cdot \delta_1}=\indicator{p>1-\alpha},
\end{equation*}
Assume that we are given a sample of coin tosses $Z_1$, $\dots$, $Z_n$ with the same bias parameter $p$. Even though the sample average $\overline{Z}_n$ is a consistent estimator of $p$, it nonetheless does not imply that the corresponding plug-in quantile estimator is consistent as the continuous mapping theorem cannot be invoked due to a discontinuity at $p=1-\alpha$. Indeed, let
\begin{equation*}
\begin{aligned}
\widehat{q}_n := \quant_\alpha \roundbrack{\roundbrack{1-\overline{Z}_n} \cdot \delta_0+\overline{Z}_n\cdot \delta_1}  = \indicator{\overline{Z}_n>1-\alpha},
\end{aligned}
\end{equation*}
and observe that $\widehat{q}_n\sim \mathrm{Ber}\roundbrack{\Prob\roundbrack{\overline{Z}_n>1-\alpha}}$. Then by the normal approximation it follows that:
\begin{equation*}
\Prob\roundbrack{\overline{Z}_n>1-\alpha}\approx 1-\Phi\roundbrack{\sqrt{n}\frac{(1-\alpha)-p}{\sqrt{p(1-p)}}}.
\end{equation*}
If $p>1-\alpha$, we can conclude that $\widehat{q}_n$ converges in probability to 1, and thus the estimator is consistent (similarly for $p<1-\alpha$). In case of equality, $\widehat{q}_n$ converges to $\mathrm{Ber}(1/2)$, and thus the estimator will not be consistent. Still, for a more general setting of the distribution defined in~\eqref{eq:weighted_pred_set} it is reasonable to expect the assumption regarding absence of jumps to be satisfied as also confirmed by our conducted empirical study.

\paragraph{Label-conditional conformal prediction.} Observing multiple points sharing the same label in a dataset makes it possible to apply the split-conformal framework in a way that makes the resulting prediction sets inherently robust to label shift~\citep{vovk2005algorithmic,vovk2016criteria,sadinle2019label,guan2019prediction}. Assume that a set of significance levels for each class $\{\alpha_y\}_{y\in\calY}$ has been chosen (e.g., $\alpha_y=\alpha$ for all $y$). By further splitting the calibration set $\calI_2$ into $\abs{\calY}=K$ groups depending on the corresponding labels, $\calI_{2,y}:=\curlybrack{i\in\calI_2: Y_i=y}$, one can consider prediction sets of the following form:
\begin{align}
    \calF_{\tau^\star_c}^{c}\roundbrack{x,u;\widehat{\pi}} & = \curlybrack{y\in\calY: \rho_y(x;\widehat{\pi})+u\cdot \widehat{\pi}_y(x)\leq \tau_{c}^\star(y)}, \nonumber \\
    \tau_{c}^\star(y) & = Q_{1-\alpha_y}\roundbrack{\curlybrack{r_i}_{i\in\calI_{2,y}}\cup \curlybrack{1}}.\label{eq:pred_sets_label_cond}
\end{align}
In other words, we separately apply split-conformal prediction framework for each label; this is like performing a separate hypothesis test for each label to determine whether there is sufficient evidence to exclude the label from the prediction set.
To elaborate, the label shift assumption states that conditional distribution of $X$ given $Y=y$ for all $y\in\calY$ does not change between source and target distributions. Thus for a test point $(X_{n+1},Y_{n+1})$ the corresponding non-conformity score $r(X_{n+1},Y_{n+1},U_{n+1};\widehat{\pi})$ together with $\curlybrack{r_i}_{i\in\calI_{2,Y_{n+1}}}$ forms a collection of exchangeable random variables, which implies label-conditional validity, that is:
\begin{equation*}
    \Prob\roundbrack{Y_{n+1} \notin \calF_{\tau_{c}^\star}^{c}\roundbrack{X_{n+1},U_{n+1};\widehat{\pi}} \mid Y_{n+1}=y} \leq \alpha_y,
\end{equation*}
for all $y\in\calY$. When $\alpha_y=\alpha$ for all $y$, one can marginalize over $y$ using \emph{any} distribution (shifted or not), to yield $\Prob\roundbrack{Y_{n+1} \notin\calF_{\tau_{c}^\star}^{c}\roundbrack{X_{n+1},U_{n+1};\widehat{\pi}}} \leq \alpha$. Thus, the label-conditional conformal framework yields a stronger guarantee than the standard (marginal) conformal and, it is automatically robust to changes in class proportions, retaining validity under label shift. The price to pay for the stronger conditional guarantee is larger prediction sets: for example, when the classes are not well-separated, label-conditional conformal can be expected to yield larger prediction sets; see Appendix~\ref{appsubsec:lcc_marg_conf} for a careful empirical study. It should also be noted that the label-conditional conformal framework requires splitting available calibration data into $K$ parts that could result in large losses of statistical efficiency when the number of classes $K$ is large. On the other hand, such construction allows to tackle label shift in a way that does not require importance weights estimation, and thus get exact finite-sample guarantee instead of asymptotic one established in Corollary~\ref{cor:asympt_est_imp_weights}. Thus, we view the label-conditional conformal framework as a complementary approach, perhaps worth utilizing when the amount of calibration data is larger relative to the number of labels.

\section{Calibration}\label{sec:calibration}
While prediction sets describe a construction on top of the output of a predictor, calibration quantifies whether the output itself admits a rigorous frequentist interpretation. In contrast to the binary setting where there is usually no confusion about a definition of a calibrated predictor, there is one in the multiclass setting. First, we state a definition of a canonically calibrated predictor.
\begin{definition}[Calibration]\label{def:calibration}
A probabilistic predictor $f:\calX\to \Delta_K$ is said to be calibrated if
\begin{equation*}
    \Prob\roundbrack{Y= y \mid f(X)} = f_y(X), \quad y\in\calY,
\end{equation*}
where $f_y(x)$ denotes the $y$-th coordinate of $f(x)$.
\end{definition}
Observe that canonical calibration requires the whole output vector to reflect the true conditional probabilities. Two extreme examples of canonically calibrated predictors include: (a) $f^{\text{Marg}}$: $f^{\text{Marg}}_y(x)=p(y)$, (b) $f^{\text{Bayes}}$: $f^{\text{Bayes}}_y(x)=\pi_y(x)$. In words, the former predictor outputs marginal probabilities of classes and the latter outputs the true class-posterior probabilities. In terms of classification efficiency, however, the first one is useless, while the second minimizes the classification risk, or the probability of incorrectly classifying a new point. Minimizing classification risk with respect to zero-one loss is computationally infeasible, and thus one refers instead to minimizing so-called \emph{surrogate} losses, e.g., cross-entropy loss, with possibly added regularization terms. As a result, one obtains prediction models that are not calibrated out-of-the-box without making strong distributional and modeling assumptions, and thus aims to achieve it by performing post-processing using held-out data. While this topic has attracted a lot attention from practitioners recently, less results have been established on the theoretical side providing formal guarantees for common procedures that target improving model's calibration. Recognized approaches include Platt scaling~\citep{platt99probabilisticoutputs}, temperature scaling~\citep{guo2017nn_calibration}, histogram binning~\citep{zadrozny2001obtaining}, isotonic regression~\citep{zadrozny2002transforming} and others. 

Model miscalibration is usually assessed using either reliability curves or related one-dimensional summary statistics. It is known that popular metrics, such as Expected Calibration Error (ECE), are not reliable since plug-in estimates can be biased if binning, or discretization, of the output of the resulting model is not performed~\citep{kumar2019calibration, vaicenavicius2019calibration}. \citet{gupta2020df_calib} establish the necessity of binning for obtaining distribution-free calibration guarantees in a binary classification setup. Binning represents coarsening of the sample space and is defined as the partitioning of the probability simplex into non-overlapping bins: $\Delta_K = B_1 \cup \cdots \cup B_M$,  $B_i \cap B_j = \emptyset$, $i\neq j$. Then a predictor $f$ induces a partition of the sample space:
\begin{equation*}
    \calX_m := \curlybrack{x\in\calX: f(x)\in B_m},\quad m\in \calM := \curlybrack{1,\dots, M}.
\end{equation*}
Since provable guarantees for canonical calibration require binning of the probability simplex, it is clear that the task becomes prohibitive with growing number of classes as each bin has to be supplied with sufficiently many data points during the calibration step for the resulting guarantees to be meaningful. One solution is given by either referring to other notions of UQ, such as the aforementioned prediction sets, or by relaxing the notion of calibration in the multiclass setting. One of well-known relaxations is class-wise, or marginal, calibration~\citep{zadrozny2002transforming,vaicenavicius2019calibration,kull2019beyond}.
\begin{definition}[Class-wise calibration]\label{def:marg_calib}
A probabilistic predictor $f:\calX\to \Delta_K$ is said to be class-wise calibrated if
\begin{equation}
\label{eq:def_marg_calibration}
    \Prob\roundbrack{Y= y \mid f_y(X)} = f_y(X), \quad y\in\calY.
\end{equation}
\end{definition}
\citet{vaicenavicius2019calibration} illustrate the difference with the canonical calibration through useful examples. In the binary setting, the two notions are equivalent with class-wise calibration being a weaker requirement for larger number of classes. It is achieved by reducing the original multiclass problem to $K$ one-vs-all binary problems with the standard post-processing routine applied consequently to each one. We focus on canonical calibration for multiclass problems as per Definition~\ref{def:calibration} and explicitly mention important implications for the binary setting, and thus marginal calibration. 
\subsection{Calibration for i.i.d. data}\label{subsec:calib_iid}
First, we assume that the binning scheme has been chosen and use $g:\calX \to \calM$ to denote the bin-mapping function: $g(x)=m$ if and only if $f(x)\in B_m$. The calibration set $\calD_{\text{cal}}=\curlybrack{(X_i,Y_i)}_{i=1}^n$ is used for estimating
\begin{equation}\label{eq:bin_cond_prob_source}
    \pi^P_{y,m} := \Prob\roundbrack{Y=y\mid f(X)\in B_m}, \quad y\in\calY,
\end{equation}
for all bins $m\in\calM$. The superscript here highlights that probabilities correspond to the source distribution $P$ and the notation will become convenient when we talk about label shift setting. With finite data one can only estimate~\eqref{eq:bin_cond_prob_source} with quantifiable measures of error, and thus provably satisfy the calibration requirement only approximately:
\begin{equation}
\label{eq:approx_calib_source}
    \Prob\roundbrack{Y=y \mid \widehat{\pi}_{y,g(X)}^{P}} \approx \widehat{\pi}_{y,g(X)}^{P}.
\end{equation}
Let $N_{m} = \abs{ \curlybrack{(X_i,Y_i)\in \calD_{\text{cal}}: f(X_i)\in B_m}}$ denote the number of calibration points that fall into bin $m\in\calM$. Note that $\curlybrack{N_m}_{m\in\calM}$ are random and satisfy $\sum_{m=1}^M N_m=n$. Empirical frequencies of class labels $y\in\calY$ in each bin $m\in\calM$:
\begin{equation}
    \widehat{\pi}_{y,m}^{P} := \frac{1}{N_{m}}\sum_{i=1}^{n} \indicator{Y_i=y,  f(X_i)\in B_m},
\end{equation}
are natural candidates to satisfy the approximate calibration condition~\eqref{eq:approx_calib_source}. For convenience, let $\pi_{m}^{P} := (\pi_{1,m}^{P}, \dots, \pi_{K,m}^{P})^\top$ denote a vector with coordinates representing bin-conditional class probabilities and let $h:\calX\to \Delta_K$ denote the \emph{recalibrated} predictor, i.e., the function that maps any feature vector to the corresponding vector of \emph{calibrated} probability estimates: $h(x)=\widehat{\pi}_{g(x)}$.
\begin{theorem}\label{thm:calib_guar_source}
Fix $\alpha\in (0,1)$. With probability at least $1-\alpha$, $\norm{1}{\widehat{\pi}^{P}_{m}-\pi^{P}_{m}} \leq \varepsilon_m$, simultaneously for all $m\in\calM$, where
\begin{equation*}
     \varepsilon_m:= \frac{2}{\sqrt{N_m}} \sqrt{\frac{1}{2}\ln \roundbrack{\frac{M2^K}{\alpha}}}.
\end{equation*}
As a consequence, with probability at least $1-\alpha$,
\begin{equation*}
     \sum_{y=1}^K\abs{\Prob\roundbrack{Y=y\mid h(X)=z} - z_{y}} \leq \max_{m\in\calM} \varepsilon_m,
\end{equation*}
simultaneously for all $z$ in the range of $h$.
\end{theorem}
The proof is given in Appendix~\ref{appsubsec:calibration_proofs}. In words, Theorem~\ref{thm:calib_guar_source} states that as long as the least-populated bin contains sufficiently many points, the output of the recalibrated predictor will approximately satisfy condition~\eqref{eq:approx_calib_source}. The first part of Theorem~\ref{thm:calib_guar_source} justifies use of empirical frequencies in place of unknown population quantities using the language of the confidence intervals. In the binary setting, the fact that it yields the desired calibration guarantee, has been formally established by~\citet{gupta2020df_calib}, and the second part of the theorem states a corresponding result for canonical calibration in the multiclass setting. 

A natural question is whether one can guarantee that each bin is supplied with a sufficient number of calibration data points in order to obtain meaningful bounds. We note that in the binary setting, one way to provably spread the calibration data evenly across bins is uniform-mass, or equal frequency, binning~\citep{kumar2019calibration, gupta2020df_calib,gupta2021binning}. 

\subsection{Label-shifted calibration}\label{subsec:calib_label_shift}
For illustrating the necessity of accounting for label shift we consider the following binary classification problem: $\calY=\curlybrack{0,1}$ with class probabilities given as $p(0)=p(1) =1/2$ and $q(0) =0.2$, $q(1)=0.8$, i.e., while on the source domain classes are equally balanced, on the target class 1 becomes dominant. For each data point, conditionally on the corresponding label, the covariates are sampled according to $X\mid Y=y\sim\mathcal{N}(\mu_y,\Sigma)$, where
\begin{equation*}
    \mu_0 = \begin{pmatrix}
    -1 \\ 0
    \end{pmatrix}, \quad \mu_1 = \begin{pmatrix}
    1 \\ 0
    \end{pmatrix}, \quad \Sigma = \begin{pmatrix}
    0.75 & 0.25 \\
    0.25& 0.75
    \end{pmatrix}.
\end{equation*}
Similarly to the toy example from Section~\ref{subsec:conf_label_shift}, here the class-posterior probabilities, and thus the Bayes-optimal rules have a closed form for both source and target domains. Not only do they minimize the probability of misclassifying a new point from the corresponding domain but also they are calibrated\footnote{Recall that in the binary setting, canonical and class-wise calibration are equivalent.}. For the source distribution a perfect probabilistic predictor is given by
\begin{equation}\label{eq:cal_bayes_opt}
    \pi^P_1(x) = \frac{p(1)\cdot \varphi(x;\mu_1,\Sigma)}{p(0)\cdot \varphi(x;\mu_0,\Sigma)+p(1)\cdot \varphi(x;\mu_1,\Sigma)},
\end{equation}
where $\varphi(x;\mu_i,\Sigma)$, $i=0,1$ denotes the PDF of a Gaussian random vector with the corresponding parameters. As illustrated on Figure~\ref{subfig:data_bayes_source}, even though the Bayes-optimal rule is calibrated on the source, a correction is required to obtain a calibrated classifier under label shift. We sample points from the target distribution and highlight those that fall inside the area $S=\curlybrack{x\in\Real^2: \pi_1^P(x)\in [0.4;0.6]}$ with boundary given by the black dashed lines. When the shift is present, predictor~\eqref{eq:cal_bayes_opt} is no longer calibrated, since otherwise one should expect roughly half of the test data points inside $S$ to be labeled as class $1$ (red squares) and half as class 0 (blue circles),
which clearly does not happen.
\begin{figure*}
    \centering
    \begin{subfigure}{0.45\textwidth}
        \centering
        \includegraphics[width=\textwidth]{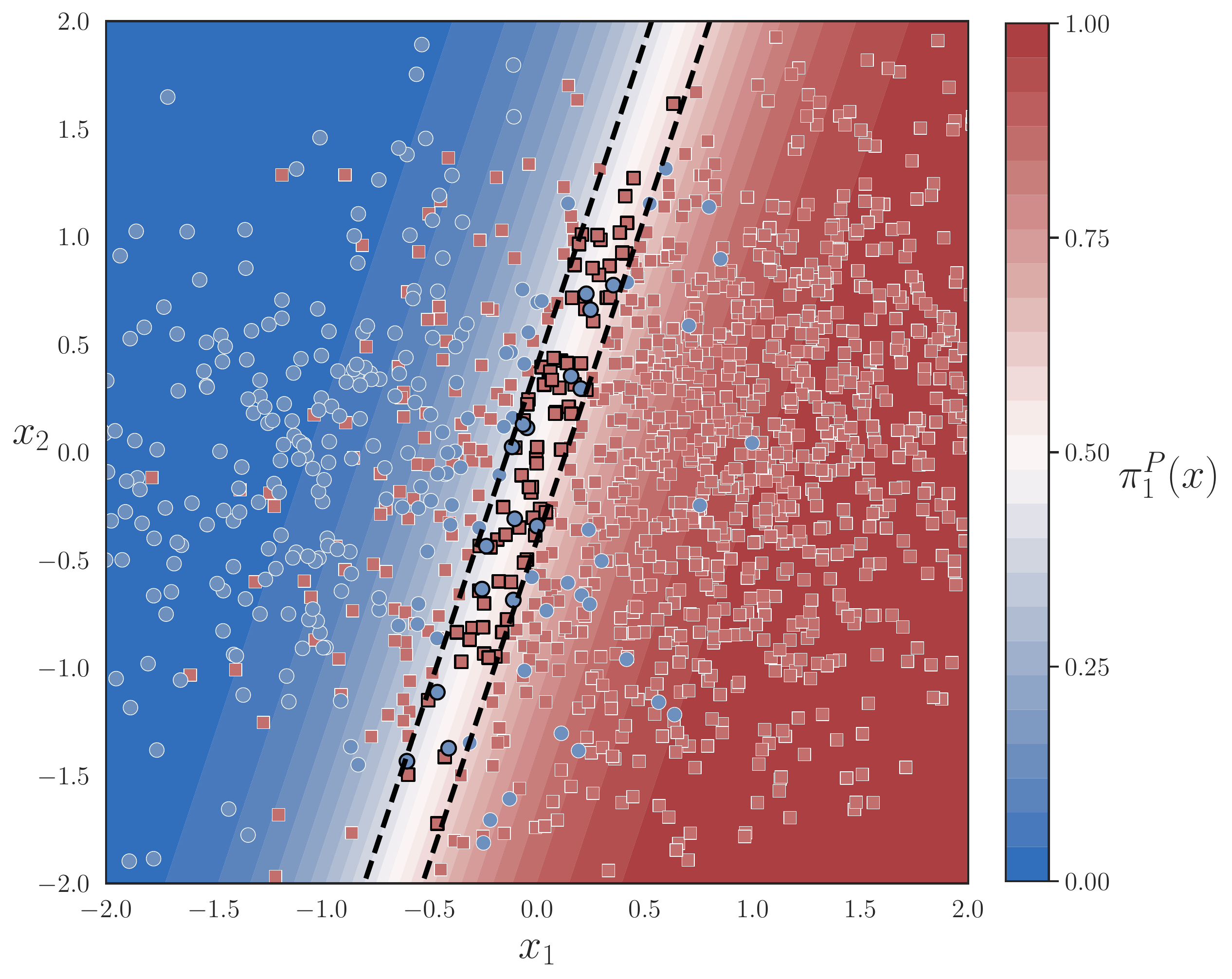}
    \caption{}
    \label{subfig:data_bayes_source}
    \end{subfigure}%
    ~ 
     \begin{subfigure}{0.45\textwidth}
         \centering
         \includegraphics[width=0.825\textwidth]{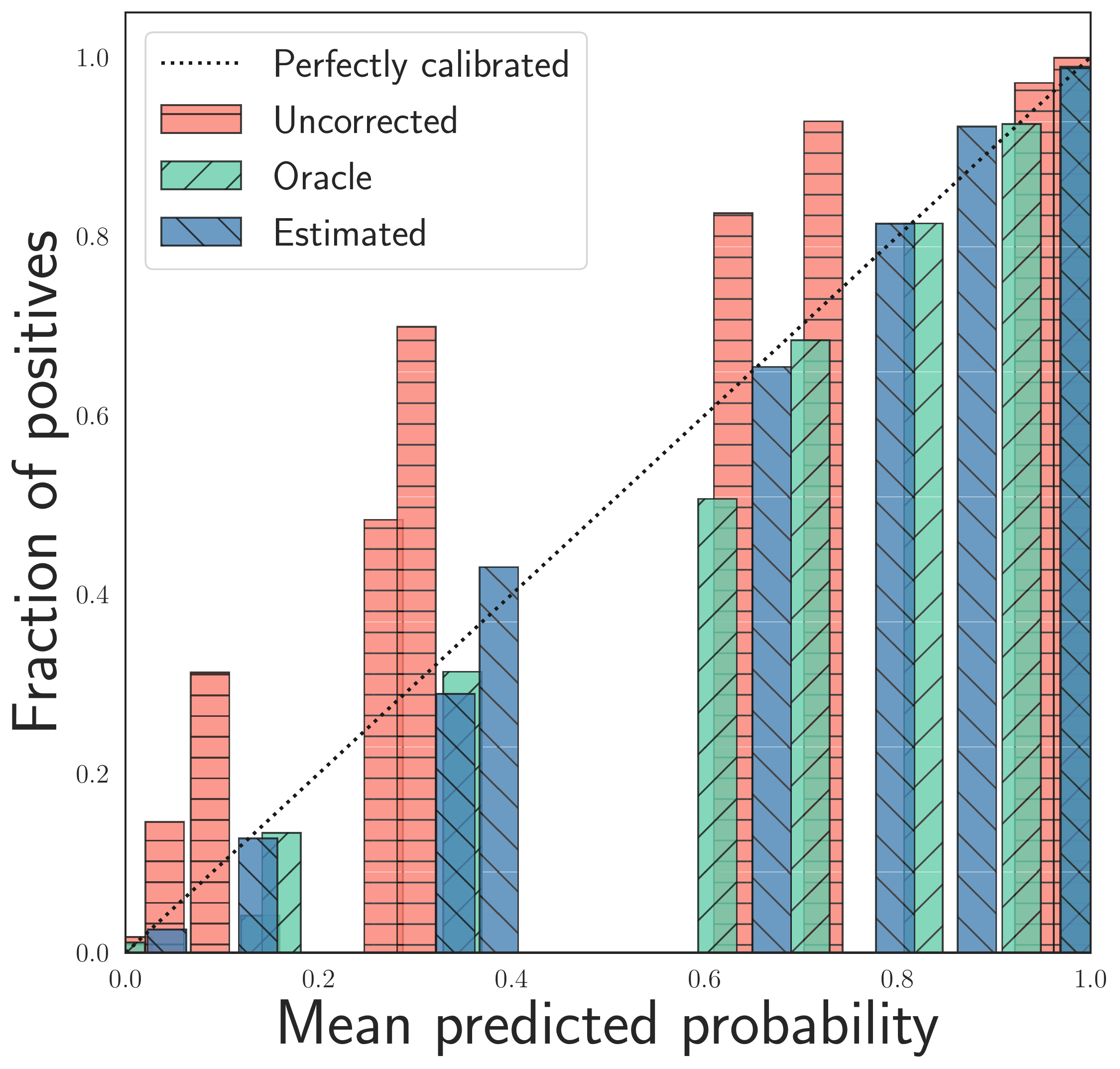}
     \caption{}
     \label{subfig:cal_toy_example_rel_curves}
     \end{subfigure}
    \caption{(\subref{subfig:data_bayes_source})
    Sampled points from the target distribution plotted against the true source class-posterior probabilities. (\subref{subfig:cal_toy_example_rel_curves}) Reliability curves for Fisher's LDA calibrated via binning with/without taking label shift into account. The deviation of uncorrected  probabilities from the diagonal line (perfect calibration) reflects the need to correct for label shift; recalibration based on estimated weights is almost identical to using oracle weights, both of which result in near-perfect calibration.}
    \label{fig:calibration_toy_example}
\end{figure*}

If both the true class-posterior distribution $\pi^P_y(x)$ and the true label likelihood ratios $w$ are known, then the form of the adjustment of the  probabilistic classifier under label shift is a simple implication of the Bayes rule~\citep{saerens2002adjusting}:
\begin{equation}\label{eq:oracle_pred_target}
    \pi_{y}^{Q}(x) = \frac{ w(y)\cdot\pi_{y}^{P}(x)}{\sum_{k=1}^K w(k)\cdot \pi_{k}^{P}(x)}.
\end{equation}
While in the oracle setting predictor~\eqref{eq:oracle_pred_target} is indeed calibrated on the target, in practice neither $\pi_y^P(x)$ nor $w$ are known. Using corresponding plug-in estimators in~\eqref{eq:oracle_pred_target} would guarantee calibration of the resulting predictor only asymptotically and under restricting modeling assumptions, and thus to obtain the distribution-free guarantees the output of the original predictor has to be discretized, or binned as in the i.i.d. setting. Relationship~\eqref{eq:oracle_pred_target} does clearly continue to hold as formally stated next.
\begin{proposition}\label{prop:calibration_label_shift}
Under label shift, for any class label $y\in\calY$ and any bin $B_m$, $m\in \calM$ it holds that:
\begin{equation*}
    \pi_{y,m}^{Q} = \frac{ w(y)\cdot\pi_{y,m}^{P}}{\sum_{k=1}^K w(k)\cdot \pi_{k,m}^{P}}.
\end{equation*}
\end{proposition}
In Section~\ref{subsec:calib_iid} we justified the use of empirical frequencies of class labels $\curlybrack{\widehat{\pi}^P_m}_{m\in\calM}$ for achieving canonical calibration of a predictor on the source domain and, as it has been noted in Section~\ref{subsec:conf_label_shift}, there are estimators of the importance weights which are known to be provably consistent under reasonable assumptions. Thus, with an estimator $\widehat{w}$ at hand, Proposition~\ref{prop:calibration_label_shift} suggests an appropriate correction to provably obtain asymptotically calibrated predictors on the target:
\begin{equation}\label{eq:weighted_emp_freqs}
    \widehat{\pi}_{y,m}^{(\widehat{w})} = \frac{\widehat{w}(y)\cdot \widehat{\pi}^{P}_{y,m}}{\sum_{k=1}^K \widehat{w}(k)\cdot \widehat{\pi}^P_{k,m}}, \quad  y\in\calY,
\end{equation}
for all bins $m\in\calM$. Theorem~\ref{thm:calib_guar_source} quantifies the error when the empirical label frequencies are used as estimators for the true unknown bin-conditional class probabilities on the source domain. However, different bounds on $\varepsilon_m$ could be available depending on chosen binning scheme, and thus we instead quantify how this estimation error on the source domain translates into the estimation error on the target for the cases when the importance weights are known and when they are rather estimated. As we shall see, the performance depends on the ratio of the largest to the smallest nonzero importance weight. Define the \emph{condition number}:
\begin{equation*}
    \kappa := \frac{\sup_k w(k)}{\inf_{k:w(k)\neq 0}w(k)},
\end{equation*}
with $\kappa=1$ corresponding to label shift not being present. Next, we quantify the miscalibration of the predictor~\eqref{eq:weighted_emp_freqs}.
\begin{theorem}\label{thm:calib_est_imp_weights}
Let $\widehat{w}$ be an estimator of $w$ and let $\widehat{\pi}^{(\widehat{w})}_{y,m}$ denote the reweighted empirical frequencies~\eqref{eq:weighted_emp_freqs} for all labels $y\in\calY$ and bins $m\in\calM$. For any bin $m\in\calM$, it holds that:
\begin{equation}\label{eq:label_shift_error}
    \norm{1}{\widehat{\pi}^{(\widehat{w})}_{m}-\pi^{Q}_{m}}\leq \underbrace{2\kappa\cdot \norm{1}{\widehat{\pi}^{P}_{m}-\pi^{P}_{m}}}_{(a)}+\underbrace{\frac{2 \norm{\infty}{\widehat{w}-w}}{\inf_{l:w(l)\neq 0}w(l)}}_{(b)}.
\end{equation}
\end{theorem}
The proof is given in Appendix~\ref{appsubsec:calibration_proofs}. In words, the calibration error on the target decomposes into two terms where (a) is controlled by the calibration error on the source and (b) is controlled by the importance weights estimation error. Further, under reasonable assumptions common procedures, such as BBSE and RLLS, construct estimators of the importance weights which are not only known to be consistent but also have quantifiable error~\citep{lipton2018label,azizzadenesheli2018label}. Similarly, any proper binning scheme that provably controls number of calibration points in each bin, e.g., uniform-mass binning in the binary setting~\citep{kumar2019calibration}, yields finite-sample guarantees for the calibration error on the source~\citep{gupta2020df_calib}. Thus,  finite-sample guarantees for the miscalibration of the resulting predictor on the target domain trivially follow by virtue of Theorem~\ref{thm:calib_est_imp_weights} via invoking simple probabilistic arguments. 

Within the same binary classification setup from the beginning of Section~\ref{subsec:calib_label_shift}, we also compare calibration via uniform-mass binning with and without accounting for label shift but this time we use Fisher's LDA as an underlying classifier, which differs from the Bayes-optimal rule by using estimators of the corresponding means and covariance matrices in~\eqref{eq:cal_bayes_opt}. Results illustrated on Figure~\ref{subfig:cal_toy_example_rel_curves} via the reliability curves indicate that shift-corrected binning with either true or estimated importance weights yields a calibrated predictor on the target domain while uncorrected fails to do so as expected. To complete the empirical study, Appendix~\ref{appsubsec:calib_real_data} further examines calibration with and without accounting for label shift on the \texttt{wine~quality} dataset from Section~\ref{subsec:conf_label_shift}.
\section{Discussion}
For safety-critical applications model's prediction must be supported with a proper measure of uncertainty. As various ad-hoc procedures provide valid inference only under assumptions that are either unrealistic or unverifiable, it is essential to understand whether non-trivial guarantees can be obtained in an assumption-lean manner. Guided by this principle, we analyzed distribution-free uncertainty quantification for classification via two complementary notions: prediction sets and calibration.

We focused on a less studied --- but still highly relevant to real-world scenarios --- setting of label shift. While it is evident that label shift does hurt model's calibration, the corresponding impact on prediction sets is less obvious. In the extreme example of almost perfectly separable data, prediction sets are usually expected to contain the most likely label only, and thus coverage is not expected to suffer much no matter how the class proportions change for the test data. Still, as we illustrated, in less idealized settings, a correction for label shift is necessary. By adapting conformal prediction sets and calibration via binning to label shift, we close an existing gap for distribution-free uncertainty quantification under two standard ways of generalizing beyond the classic i.i.d. setting. Importantly, those adaptations do not require labeled data from the target domain which can be useful in applications where the labeling process is expensive. We note that handling label shift should be expected to be an easier task rather than handling another common setting --- covariate shift --- as the latter typically involves estimating a high-dimensional, and usually continuous, likelihood ratio. 

With theoretical results available for calibration in the binary setting, and thus class-wise (coordinatewise) calibration in a more general multiclass setting, establishing meaningful guarantees for ``full'' canonical calibration in the latter setting remains an intriguing future research direction. One particular example is related to the question of the importance weights estimation under label shift. While approaches based on confusion matrices, e.g., BBSE and RLLS, provably yield consistent estimators under relatively mild assumptions, alternative approaches, such as MLLS with preceding ad-hoc calibration on the source domain, tend to perform better empirically~\citep{alexandari2020mle}. Theoretical foundations for MLLS developed recently by~\citet{garg2020unified} require the underlying predictor to be canonically calibrated which is itself, unfortunately, hard to guarantee provably which creates a (somewhat circular) gap between theory and practice.

\paragraph{Acknowledgements} The authors would like to thanks Chirag Gupta and the anonymous UAI 2021 reviewers for comments on an initial version of this paper.

\bibliographystyle{plainnat}
\addcontentsline{toc}{section}{References}
\bibliography{refs}

\begin{thebibliography}{30}
\providecommand{\natexlab}[1]{#1}
\providecommand{\url}[1]{\texttt{#1}}
\expandafter\ifx\csname urlstyle\endcsname\relax
  \providecommand{\doi}[1]{doi: #1}\else
  \providecommand{\doi}{doi: \begingroup \urlstyle{rm}\Url}\fi

\bibitem[Alexandari et~al.(2020)Alexandari, Kundaje, and
  Shrikumar]{alexandari2020mle}
Amr Alexandari, Anshul Kundaje, and Avanti Shrikumar.
\newblock Maximum likelihood with bias-corrected calibration is hard-to-beat at
  label shift adaptation.
\newblock In \emph{International Conference on Machine Learning}, 2020.

\bibitem[Angelopoulos et~al.(2021)Angelopoulos, Bates, Malik, and
  Jordan]{angelopoulos2021classification}
Anastasios Angelopoulos, Stephen Bates, Jitendra Malik, and Michael~I. Jordan.
\newblock Uncertainty sets for image classifiers using conformal prediction.
\newblock In \emph{International Conference on Learning Representations}, 2021.

\bibitem[Azizzadenesheli et~al.(2019)Azizzadenesheli, Liu, Yang, and
  Anandkumar]{azizzadenesheli2018label}
Kamyar Azizzadenesheli, Anqi Liu, Fanny Yang, and Animashree Anandkumar.
\newblock Regularized learning for domain adaptation under label shifts.
\newblock In \emph{International Conference on Learning Representations}, 2019.

\bibitem[Barber et~al.(2021)Barber, Candes, Ramdas, and
  Tibshirani]{barber2021predictive}
Foygel~Rina Barber, J.~Emmanuel Candes, Aaditya Ramdas, and J.~Ryan Tibshirani.
\newblock Predictive inference with the jackknife+.
\newblock \emph{The Annals of Statistics}, 2021.

\bibitem[Cauchois et~al.(2020)Cauchois, Gupta, and Duchi]{cauchois2020knowing}
Maxime Cauchois, Suyash Gupta, and John~C. Duchi.
\newblock Knowing what you know: valid confidence sets in multiclass and
  multilabel prediction.
\newblock \emph{arXiv preprint: 2004.10181}, 2020.

\bibitem[Cortez et~al.(2009)Cortez, Cerdeira, Almeida, Matos, and
  Reis]{data_wine}
Paulo Cortez, Ant\'{o}nio Cerdeira, Fernando Almeida, Telmo Matos, and Jos\'{e}
  Reis.
\newblock Modeling wine preferences by data mining from physicochemical
  properties.
\newblock \emph{Decision Support Systems}, 2009.

\bibitem[Garg et~al.(2020)Garg, Wu, Balakrishnan, and Lipton]{garg2020unified}
Saurabh Garg, Yifan Wu, Sivaraman Balakrishnan, and Zachary Lipton.
\newblock A unified view of label shift estimation.
\newblock In \emph{Advances in Neural Information Processing Systems}, 2020.

\bibitem[Guan and Tibshirani(2019)]{guan2019prediction}
Leying Guan and Rob Tibshirani.
\newblock Prediction and outlier detection in classification problems.
\newblock \emph{arXiv preprint: 1905.04396}, 2019.

\bibitem[Guo et~al.(2017)Guo, Pleiss, Sun, and
  Weinberger]{guo2017nn_calibration}
Chuan Guo, Geoff Pleiss, Yu~Sun, and Kilian~Q. Weinberger.
\newblock On calibration of modern neural networks.
\newblock In \emph{International Conference on Machine Learning}, 2017.

\bibitem[Gupta and Ramdas(2021)]{gupta2021binning}
Chirag Gupta and Aaditya Ramdas.
\newblock Distribution-free calibration guarantees for histogram binning
  without sample splitting.
\newblock In \emph{International Conference on Machine Learning}, 2021.

\bibitem[Gupta et~al.(2019)Gupta, Kuchibhotla, and Ramdas]{gupta2019nested}
Chirag Gupta, Arun~K. Kuchibhotla, and Aaditya~K. Ramdas.
\newblock Nested conformal prediction and quantile out-of-bag ensemble methods.
\newblock \emph{arXiv preprint: 1910.10562}, 2019.

\bibitem[Gupta et~al.(2020)Gupta, Podkopaev, and Ramdas]{gupta2020df_calib}
Chirag Gupta, Aleksandr Podkopaev, and Aaditya Ramdas.
\newblock Distribution-free binary classification: prediction sets, confidence
  intervals and calibration.
\newblock In \emph{Advances in Neural Information Processing Systems}, 2020.

\bibitem[Kull et~al.(2019)Kull, Perello~Nieto, K\"{a}ngsepp, Silva~Filho, Song,
  and Flach]{kull2019beyond}
Meelis Kull, Miquel Perello~Nieto, Markus K\"{a}ngsepp, Telmo Silva~Filho, Hao
  Song, and Peter Flach.
\newblock Beyond temperature scaling: Obtaining well-calibrated multi-class
  probabilities with dirichlet calibration.
\newblock In \emph{Advances in Neural Information Processing Systems}, 2019.

\bibitem[Kumar et~al.(2019)Kumar, Liang, and Ma]{kumar2019calibration}
Ananya Kumar, Percy~S Liang, and Tengyu Ma.
\newblock Verified uncertainty calibration.
\newblock In \emph{Advances in Neural Information Processing Systems}, 2019.

\bibitem[Lei et~al.(2013)Lei, Robins, and Wasserman]{lei2013distribution}
Jing Lei, James Robins, and Larry Wasserman.
\newblock Distribution-free prediction sets.
\newblock \emph{Journal of the American Statistical Association}, 2013.

\bibitem[Lei et~al.(2018)Lei, G’Sell, Rinaldo, Tibshirani, and
  Wasserman]{lei2018df_pred_inference}
Jing Lei, Max G’Sell, Alessandro Rinaldo, Ryan~J. Tibshirani, and Larry
  Wasserman.
\newblock Distribution-free predictive inference for regression.
\newblock \emph{Journal of the American Statistical Association}, 2018.

\bibitem[Lipton et~al.(2018)Lipton, Wang, and Smola]{lipton2018label}
Zachary~C. Lipton, Yu{-}Xiang Wang, and Alexander~J. Smola.
\newblock Detecting and correcting for label shift with black box predictors.
\newblock In \emph{International Conference on Machine Learning}, 2018.

\bibitem[Platt(1999)]{platt99probabilisticoutputs}
John~C. Platt.
\newblock Probabilistic outputs for support vector machines and comparisons to
  regularized likelihood methods.
\newblock In \emph{Advances in Large Margin Classifiers}, 1999.

\bibitem[Romano et~al.(2019)Romano, Patterson, and Candes]{romano2019cqr}
Yaniv Romano, Evan Patterson, and Emmanuel Candes.
\newblock Conformalized quantile regression.
\newblock In \emph{Advances in Neural Information Processing Systems}, 2019.

\bibitem[Romano et~al.(2020)Romano, Sesia, and
  Candès]{romano2020classification}
Yaniv Romano, Matteo Sesia, and Emmanuel~J. Candès.
\newblock Classification with valid and adaptive coverage.
\newblock In \emph{Advances in Neural Information Processing Systems}, 2020.

\bibitem[Sadinle et~al.(2019)Sadinle, Lei, and Wasserman]{sadinle2019label}
Mauricio Sadinle, Jing Lei, and Larry Wasserman.
\newblock Least ambiguous set-valued classifiers with bounded error levels.
\newblock \emph{Journal of the American Statistical Association}, 114\penalty0
  (525):\penalty0 223--234, 2019.

\bibitem[Saerens et~al.(2002)Saerens, Latinne, and
  Decaestecker]{saerens2002adjusting}
Marco Saerens, Patrice Latinne, and Christine Decaestecker.
\newblock Adjusting the outputs of a classifier to new a priori probabilities:
  A simple procedure.
\newblock \emph{Neural Computation}, 2002.

\bibitem[Shimodaira(2000)]{shimodaira2000improving}
Hidetoshi Shimodaira.
\newblock Improving predictive inference under covariate shift by weighting the
  log-likelihood function.
\newblock \emph{Journal of Statistical Planning and Inference}, 2000.

\bibitem[Tibshirani et~al.(2019)Tibshirani, Foygel~Barber, Candes, and
  Ramdas]{tibs2019conf}
Ryan~J Tibshirani, Rina Foygel~Barber, Emmanuel Candes, and Aaditya Ramdas.
\newblock Conformal prediction under covariate shift.
\newblock In \emph{Advances in Neural Information Processing Systems}, 2019.

\bibitem[Vaicenavicius et~al.(2019)Vaicenavicius, Widmann, Andersson, Lindsten,
  Roll, and Sch{\"o}n]{vaicenavicius2019calibration}
Juozas Vaicenavicius, David Widmann, Carl Andersson, Fredrik Lindsten, Jacob
  Roll, and Thomas~B Sch{\"o}n.
\newblock Evaluating model calibration in classification.
\newblock In \emph{International Conference on Artificial Intelligence and
  Statistics}, 2019.

\bibitem[van~der Vaart and Wellner(1996)]{vandervaart1996weak}
Aad~W. van~der Vaart and Jon~A. Wellner.
\newblock \emph{Weak Convergence}.
\newblock Springer, 1996.

\bibitem[Vovk et~al.(2005)Vovk, Gammerman, and Shafer]{vovk2005algorithmic}
Vladimir Vovk, Alex Gammerman, and Glenn Shafer.
\newblock \emph{Algorithmic learning in a random world}.
\newblock Springer, 2005.

\bibitem[Vovk et~al.(2016)Vovk, Fedorova, Nouretdinov, and
  Gammerman]{vovk2016criteria}
Vladimir Vovk, Valentina Fedorova, Ilia Nouretdinov, and Alex Gammerman.
\newblock Criteria of efficiency for conformal prediction.
\newblock In \emph{Symposium on Conformal and Probabilistic Prediction with
  Applications}, 2016.

\bibitem[Zadrozny and Elkan(2001)]{zadrozny2001obtaining}
Bianca Zadrozny and Charles Elkan.
\newblock Obtaining calibrated probability estimates from decision trees and
  naive {B}ayesian classifiers.
\newblock In \emph{International Conference on Machine Learning}, 2001.

\bibitem[Zadrozny and Elkan(2002)]{zadrozny2002transforming}
Bianca Zadrozny and Charles Elkan.
\newblock Transforming classifier scores into accurate multiclass probability
  estimates.
\newblock In \emph{International Conference on Knowledge Discovery and Data
  Mining}, 2002.

\end{thebibliography}

\newpage

\newpage
\appendix

\section{Importance weights estimation under label shift}\label{appsec:imp_weight_estimation}

Below we provide details about importance weights estimation procedures which are relevant mainly to Sections~\ref{subsec:conf_label_shift} and \ref{subsec:calib_label_shift} of the paper. Estimation of the importance weights is performed using a held-out labeled set from the source distribution and an unlabeled set from the target distribution. Procedures, such as BBSE~\citep{lipton2018label} or RLLS~\citep{azizzadenesheli2018label}, are based on estimation of the confusion matrix and yield consistent importance weights estimators with quantifiable estimation error under relatively mild assumptions. First, given a black-box predictor $f: \calX \to \Delta_K$, define the corresponding expected confusion matrix $C_P(f)\in \Real^{\abs{\calY}\times\abs{\calY}}$:
\begin{equation*}
    \squarebrack{C_P(f)}_{ij}:= \Exp{P}{\indicator{\argmax_k f_k(X)=i}\cdot \indicator{Y=j}}.
\end{equation*}
We assume that
\begin{enumerate}[itemsep=0cm]    
    \item[] (A1) for every label $y\in \calY$, it holds that $q(y)>0\Longrightarrow p(y)>0$,
	\item[] (A2) expected confusion matrix $C_P(f)$ is full-rank.
\end{enumerate}
Assumption (A1) states that target label distribution is absolutely continuous with respect to the source. Indeed, reasoning properly about a class in the target domain which is not represented in the source domain is not possible. Assumption (A2) simply represents an identifiability condition. \citet{lipton2018label} show that under label shift assumption: $\Prob_Q\roundbrack{f(X)=i} = \sum_{j\in\calY} \squarebrack{C_P(f)}_{ij} w(j)$, or in matrix-vector notation:
\begin{equation*}
\label{eq:label_shift_linear_system}
    \mu = C_P(f) w.
\end{equation*}
where $\mu\in\Real^{\abs{\calY}}: \mu_i =\Prob_Q\roundbrack{f(X)=i}$. BBSE is a simple plug-in procedure, which yields the following estimator of the importance weights: 
\begin{align*}
\widehat{w} &=~ \widehat{C}^{-1}\ \widehat{\mu}, \\
    \text{where } \quad \widehat{C}_{ij} &=~ \frac{1}{m}\sum_{p=1}^m \indicator{f(X_p^{s})=i \text{ and } Y_p^s = j}, \\
    \widehat{\mu}_{i} &=~
    \frac{1}{l} \sum_{p=1}^l \indicator{f(X_p^t)=i},
\end{align*}
where $\curlybrack{(X_i^s,Y_i^s)}_{i=1}^m$ is a labeled dataset from the source distribution and $\curlybrack{(X_i^t)}_{i=1}^l$ is unlabeled data from the target distribution. BBSE-hard described above can be trivially modified to the whole probability distribution output of $f$ which is referred to as BBSE-soft procedure. Under aforementioned assumptions, \citet{lipton2018label} establish results with respect to consistency of BBSE and corresponding convergence rates.

A well-known alternative approach to directly estimate the importance weights which performs well in practice is MLLS~\citep{saerens2002adjusting} and its recent variations that combine it with preceding calibration on the source domain~\citep{alexandari2020mle}. We refer the reader to~\citet{garg2020unified} for the theoretical analysis of MLLS and a detailed overview of the results for the importance weights estimation under label shift. For all simulations in this work we use BBSE-soft procedure motivated simply by its satisfactory empirical performance throughout all of the simulations we performed. Our modular approach to UQ allows to replace BBSE with any alternative choice.

\section{Conformal classification}

Below, Section~\ref{appsubsec:tie_break} includes details about the tie-breaking rules for the oracle prediction sets, Section~\ref{appsubsec:randomization} includes a discussion regarding the role of randomization for conformal classification, Section~\ref{appsubsec:conformal_proofs} includes all necessary proofs for Sections~\ref{subsec:exch_conformal} and \ref{subsec:conf_label_shift} and Section~\ref{appsubsec:conformal_real_data} includes details about the simulation on a real dataset mentioned in Section~\ref{subsec:conf_label_shift}.

\subsection{Tie-breaking rules for the oracle prediction set}\label{appsubsec:tie_break}

In practice, when an estimator $\widehat{\pi}_y(x)$ is used in place of $\pi_y(x)$, one does not expect ties to be present but for completeness it is important to consider such scenario in the oracle setting. First, note that for any $\alpha\in(0,1)$, the oracle prediction set clearly never include labels $y\in\calY:\pi_y(x)=0$. Now, presence of ties can lead to a conservative prediction set for some $x\in\calX$ if there is a subset of class labels $S(x)\subseteq \calY$ of size $L=\abs{S(x)}>1$, such that $\forall y, y'\in S(x): \pi_y(x) = \pi_{y'}(x)>0$ and
\begin{equation*}
\begin{cases}
      \Prob \roundbrack{Y\in C^{\mathrm{oracle}}_\alpha(X)\backslash S(X) \mid X=x} < 1-\alpha,\\
 \Prob \roundbrack{Y\in C^{\mathrm{oracle}}_\alpha(X)\mid X=x} \geq (1-\alpha).
\end{cases}
\end{equation*}
In the oracle case ties can be broken arbitrarily in order to preserve the conditional coverage. One option is to break ties randomly, i.e. one can fix a random permutation of labels in $S(x)$: $\widetilde{y}_{i_1},\dots,\widetilde{y}_{i_l}$, and output a smaller oracle prediction set:
\begin{equation*}
    C^{\mathrm{oracle, new}}_\alpha(X):=\roundbrack{C^{\mathrm{oracle}}_\alpha(X)\backslash S(X)} \cup \curlybrack{\widetilde{y}_{i_1},\dots,\widetilde{y}_{i_{l^\star}}},
\end{equation*}
where $l^\star$ is the smallest index in $\curlybrack{1,\dots,l}$ such that
\begin{equation*}
\begin{aligned}
\Prob \roundbrack{Y\in C^{\mathrm{oracle}}_\alpha(X)\backslash S(X) \mid X=x} + \sum_{k=1}^{l^\star}\pi_{i_k}(x)\geq 1-\alpha.
\end{aligned}
\end{equation*}
\subsection{Note on randomization and conditional coverage}\label{appsubsec:randomization}
As the number of works on conformal classification has seen a recent spurt, it is important to understand what exactly might be the benefits of using one nested sequence over another. For example, \citet{angelopoulos2021classification} state in their Appendix B that ``randomization is of little practical importance, since... output by the randomized procedure will differ from that of the non-randomized procedure by at most one element''. However, we do not quite agree with their sentiment about it being of little practical importance for the following reason. While their observation is indeed accurate in the oracle setting, there is a noticeable difference in the empirical conditional coverage when the nested sequences are conformalized in practice (non-oracle setting). Roughly speaking, randomized scores better handle the heterogeneity of the conditional distribution of the response variable across the sample space. Note that this type of randomization has a different role from that of a randomized conformal p-value~\cite{vovk2005algorithmic} which aims to improve possibly conservative marginal coverage. We believe that the reasoning below complements the one given in~\citet{romano2020classification} and, in particular, might help an unfamiliar reader to gain some useful insights (as well as arguably having simpler notation). For completeness, we start with an example of randomization in action. Consider a binary classification problem: $\calY=\curlybrack{0,1}$, and fix target miscoverage level $\alpha=0.05$. Now, assume that for some $x\in\calX$:
\begin{itemize}
    \item $\pi_0(x)=0.99$, $\pi_1(x)=0.01$. Then with probability $95/99$, we have $\tilde{C}^{\mathrm{oracle}}_\alpha(x,u)=\curlybrack{0}$ and $\tilde{C}^{\mathrm{oracle}}_\alpha(x,u)=\curlybrack{\emptyset}$ otherwise. 
    \item $\pi_0(x)=0.9$, $\pi_1(x)=0.1$. Then with probability $1/2$, $\tilde{C}^{\mathrm{oracle}}(x,u)=\curlybrack{0,1}$ and $\tilde{C}^{\mathrm{oracle}}(x,u)=\curlybrack{0}$ otherwise.
\end{itemize}

First, consider the marginal coverage of conformal prediction sets in the ``null'' case when $\widehat{\pi}\equiv \pi$. The marginal coverage guarantee of conformal prediction sets is due to Lemma~\ref{lem:quanile_lemma} which states a classic result for quantiles of exchangeable random variables and is tight when these variables are almost surely distinct. In the non-randomized setting for any point $(X,Y)$, the corresponding non-conformity score are given by $\rho_{Y}(X;\pi)$. Such form might suggest that the marginal coverage could be conservative due to possible ties as whenever the predicted most likely label appears to be the correct one, it holds that $\rho_{Y}(X;\pi)=0$. However, if ties among non-conformity scores are present, they would typically occur only between zero-valued scores, and thus in a reasonable classification setup one should expect the marginal coverage to be tight even for non-randomized nested sequence as the calibrated threshold would typically be nonzero.

Next, before reasoning about conditional coverage of conformal sets, recall that the conditional distribution of the response is discrete in classification setting, and thus even in the null case it is hard to reason meaningfully about the distribution of non-conformity scores $\rho_{Y}(X;\pi)$. However, \citet{romano2020classification} noticed that if randomization~\eqref{eq:oracle_rand_ps} is used, then it becomes possible to do at least in the null case. If $\widehat{\pi}\equiv \pi$, it is trivial to see the distribution of corresponding non-conformity scores $\rho_Y(X;\pi)+U\cdot \pi(X)$ is uniform conditional on $X$. Then, as the authors conjecture, it is intuitive that conformal prediction sets would recover the oracle ones under some consistency assumptions for $\widehat{\pi}$.

However, randomization is also performed when the prediction set is a singleton containing the most likely label only, and thus might yield non-interpretable and non-actionable empty prediction sets being purely the consequence of deploying randomization. Thus one might consider abstaining from dropping a label from the prediction set whenever it forms a singleton and perform randomization if and only if the oracle prediction set contains more than one label. While that decision can be embedded into  either prediction step only or calibration step as well, we state explicitly that it should be done at the prediction step only for the aforementioned reasons.

Consider the binary toy example from Section~\ref{subsec:calib_label_shift} with focus on the source distribution only. As the true class-posterior probability $\pi_1^P(x)$ is known, we construct the non-randomized oracle prediction set $C^{\mathrm{oracle}}$ and compare it visually with the randomized version $\widetilde{C}^{\mathrm{oracle}}$ on Figures~\ref{subfig:oracle_non_randomized} and \ref{subfig:oracle_randomized} where randomization demonstrates desired behavior.

Consequently, we consider conformal prediction sets based on non-randomized sequence:
\begin{equation}\label{eq:non_rand_conf_pred_set}
\begin{aligned}
    \calF_{\tau^\star} \roundbrack{x,u;\widehat{\pi}} & = \curlybrack{y\in\calY:  r'(x,y)\leq \tau^\star},\\
    \tau^\star & = \quant_{1-\alpha} \roundbrack{\curlybrack{r'_i}_{i\in\calI_2}\cup \curlybrack{1}},\\
    r'(x,y) & = \rho_{y}(x;\widehat{\pi}),
\end{aligned}
\end{equation}
and two randomized sequences where Scheme 1 performs randomization for all labels and was introduced before for conformal prediction sets~\eqref{eq:exchan_conf_pred_set} and Scheme 2 (added for completeness of comparison) performs randomization for  all labels except the most likely one:
\begin{equation}\label{eq:rand_conf_pred_set_alternative}
\begin{aligned}
    \calF_{\tau^\star} \roundbrack{x,u;\widehat{\pi}} & = \curlybrack{y\in\calY:  r''(x,y)\leq \tau^\star},\\
    \tau^\star & = \quant_{1-\alpha} \roundbrack{\curlybrack{r''_i}_{i\in\calI_2}\cup \curlybrack{1}},
\end{aligned}
\end{equation}
where
\begin{equation*}
    r''(x,y) = \indicator{\rho_y(x;\widehat{\pi})>0}\cdot \roundbrack{\rho_y(x;\widehat{\pi})+u\cdot \widehat{\pi}_y(x)}.
\end{equation*}
We again use the Bayes-optimal classifier $\pi_y(x)$, and thus ignore the results that are due to estimation and focus purely on effects that are due to conformalization. For a single data draw we illustrate the resulting conformal prediction sets on Figures~\ref{subfig:conf_toy_example_nonrand}, \ref{subfig:conf_toy_example_rand_1} and \ref{subfig:conf_toy_example_rand_2}. While at first sight it might seem that non-randomized nested sequences is superior in terms of yielding prediction sets with smaller cardinality, it should be taken with a grain of salt. We repeatedly draw calibration and test data and track marginal characteristics for those sets. As expected, all three resulting prediction sets inherit $1-\alpha$ (marginal) coverage guarantee as confirmed on Figure~\ref{subfig:conf_toy_example_marg_coverage}. Moreover, Figure~\ref{subfig:conf_toy_example_marg_size} indeed confirms that randomization could yield larger prediction sets for not perfectly separable data. But Figure~\ref{subfig:conf_toy_example_tau} confirms that randomization proposed by~\citet{romano2020classification} (Scheme 1) demonstrates superior conditional coverage since for this example the true $\pi_y(x)$ is used, and thus the oracle prediction sets are recovered if $\tau^\star=1-\alpha$. Figure~\ref{subfig:conf_toy_example_tau_change_cal_set_size} confirms that oracle prediction sets are not recovered even when the size of the calibration set is increased.

\begin{figure*}
    \centering
    \begin{subfigure}{0.45\textwidth}
        \centering
        \includegraphics[width=\textwidth]{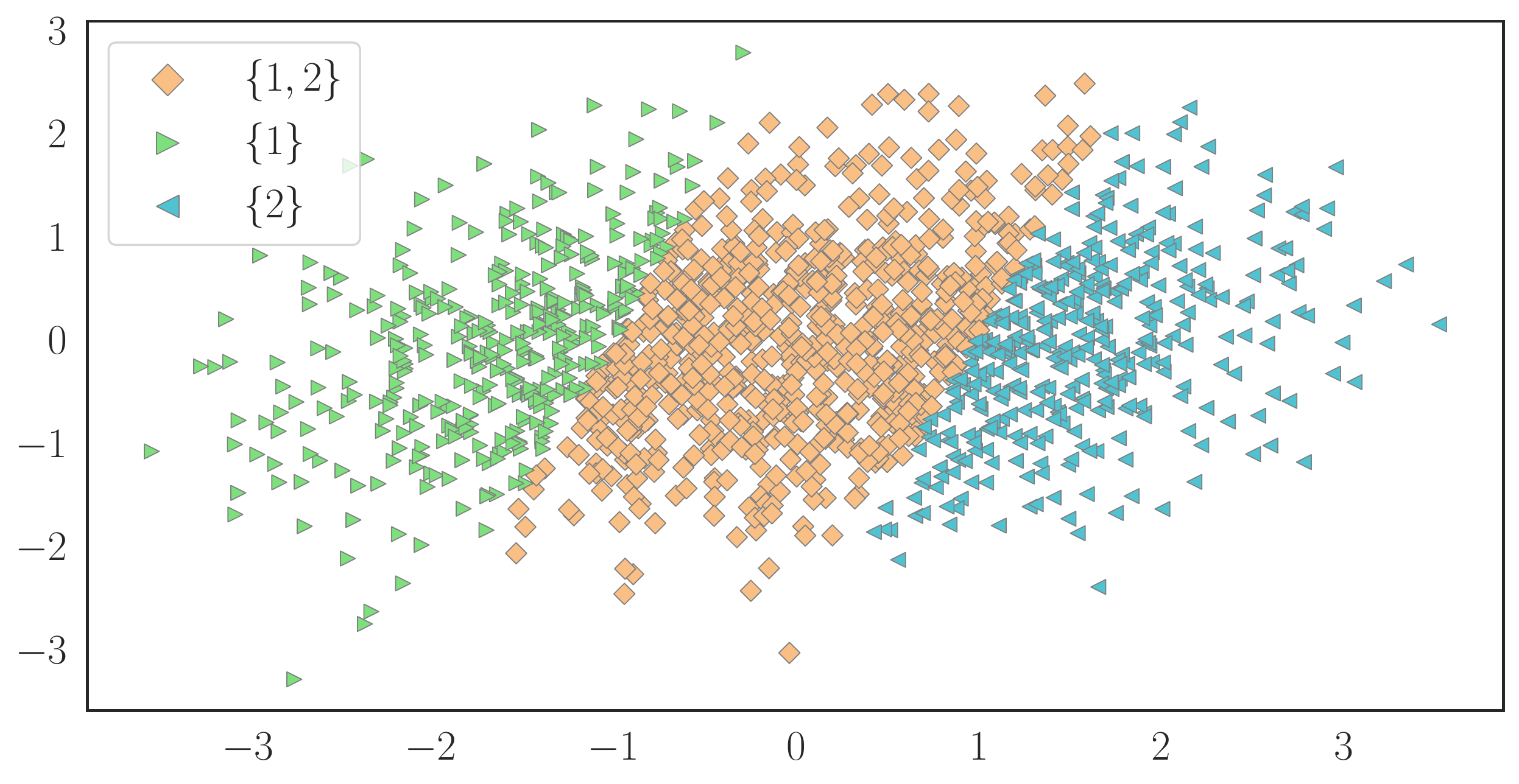}
    \caption{}
    \label{subfig:oracle_non_randomized}
    \end{subfigure}%
    ~ 
    \begin{subfigure}{0.45\textwidth}
        \centering
        \includegraphics[width=\textwidth]{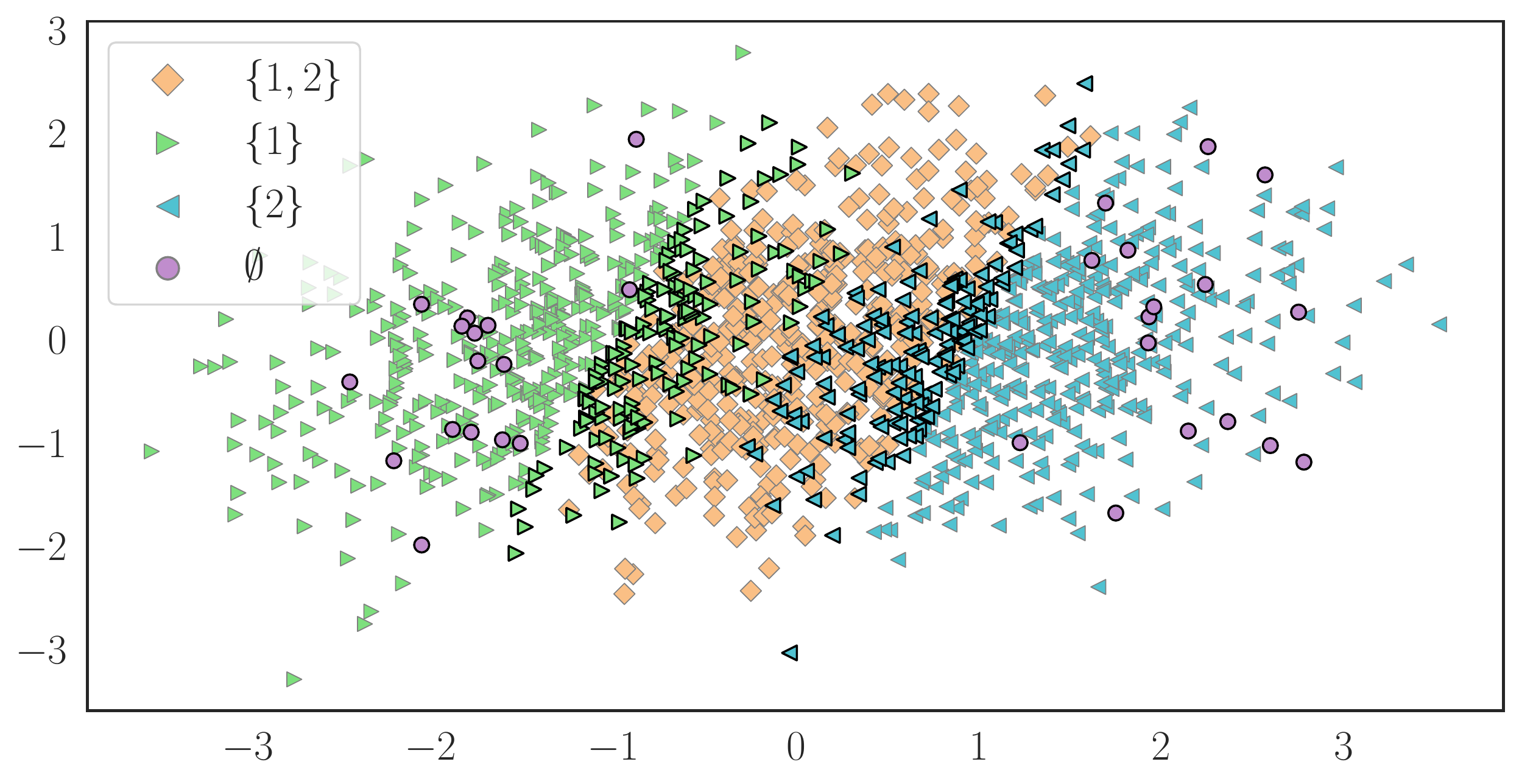}
    \caption{}
    \label{subfig:oracle_randomized}
    \end{subfigure}
    ~
    \begin{subfigure}{0.45\textwidth}
        \centering
        \includegraphics[width=\textwidth]{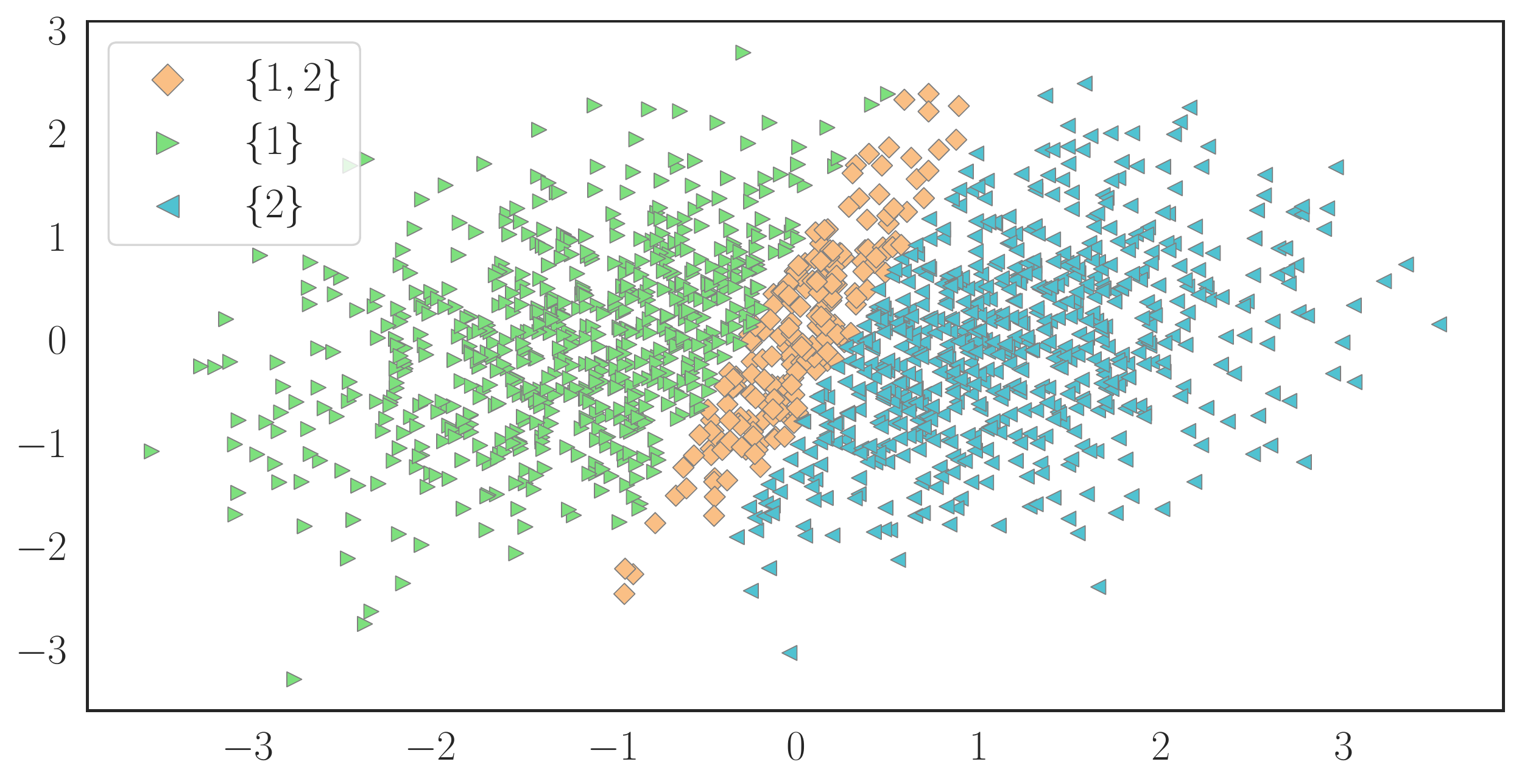}
    \caption{}
    \label{subfig:conf_toy_example_nonrand}
    \end{subfigure}%
    ~ 
    \begin{subfigure}{0.45\textwidth}
        \centering
        \includegraphics[width=\textwidth]{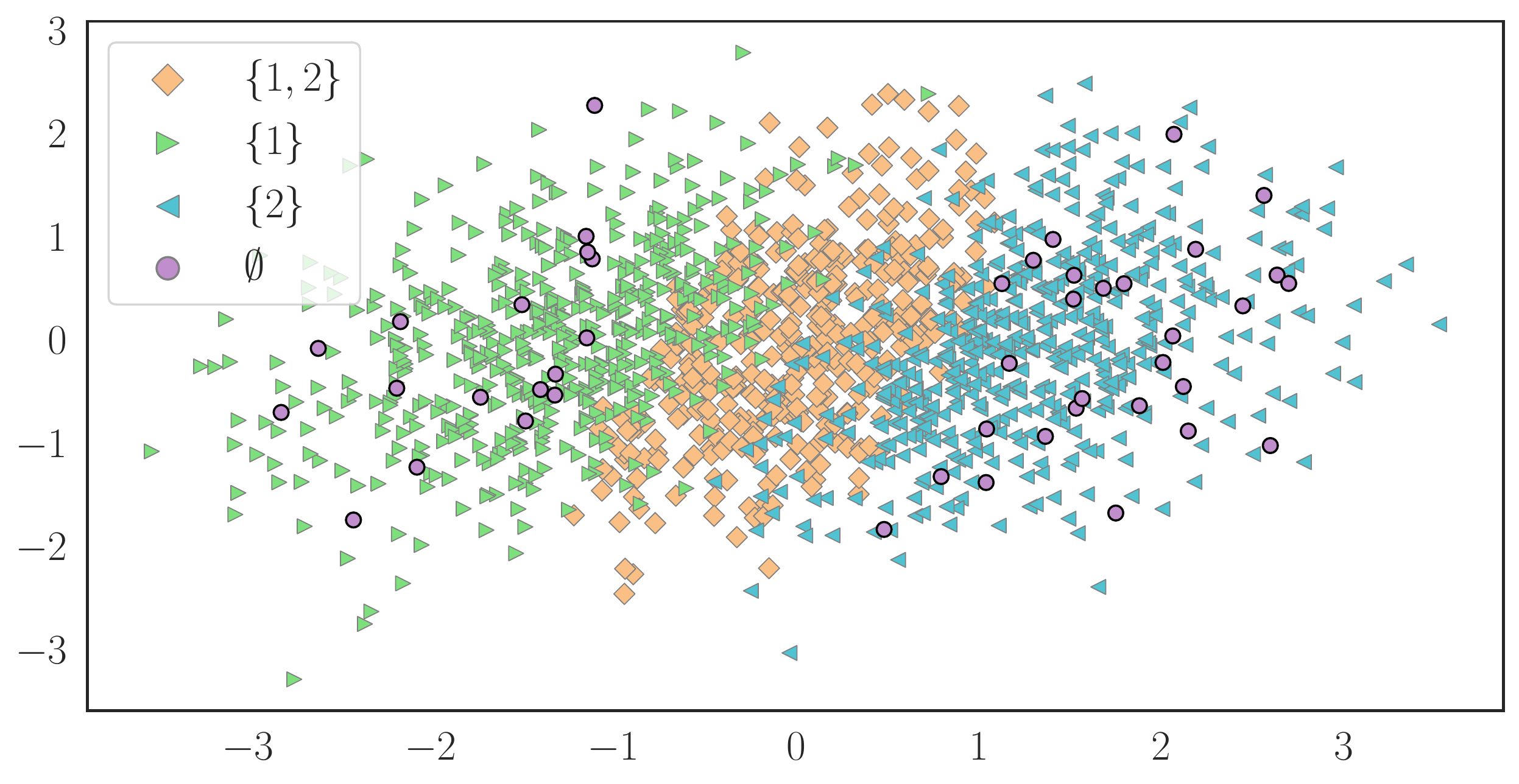}
    \caption{}
    \label{subfig:conf_toy_example_rand_1}
    \end{subfigure}
    ~
    \begin{subfigure}{0.45\textwidth}
        \centering
        \includegraphics[width=\textwidth]{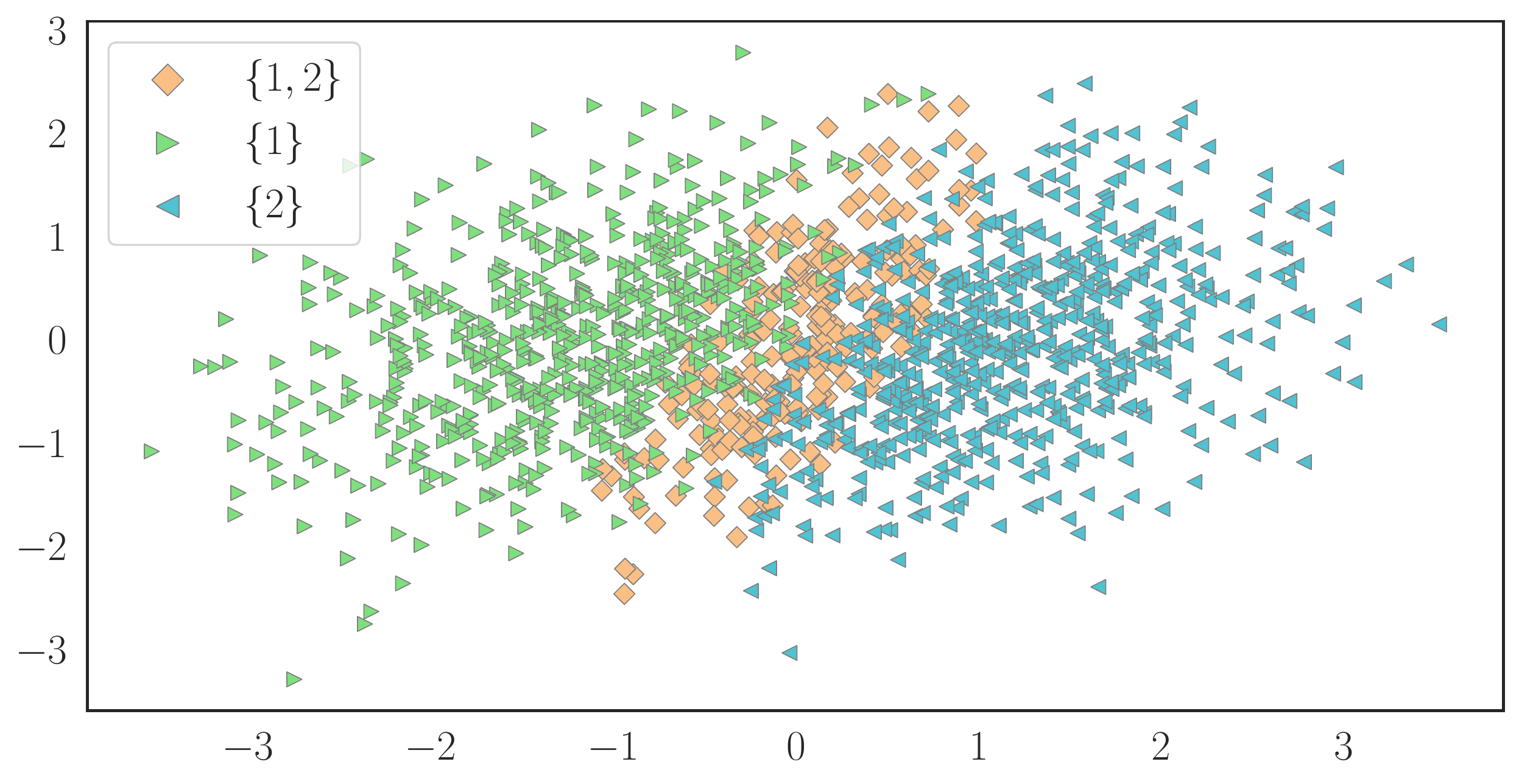}
    \caption{}
    \label{subfig:conf_toy_example_rand_2}
    \end{subfigure}
    \caption{Prediction sets corresponding to (\subref{subfig:oracle_non_randomized}) the non-randomized oracle from~\eqref{eq:oracle_nonrand_pred_set_equiv_def}; (\subref{subfig:oracle_randomized}) the randomized oracle from~\eqref{eq:oracle_rand_ps}; (\subref{subfig:conf_toy_example_nonrand}) the non-randomized conformal method~\eqref{eq:non_rand_conf_pred_set}; (\subref{subfig:conf_toy_example_rand_1}) the randomized conformal (scheme 1) method~\eqref{eq:exchan_conf_pred_set}; (\subref{subfig:conf_toy_example_rand_2}) the randomized conformal (scheme 2) method~\eqref{eq:rand_conf_pred_set_alternative}. Notice that randomization acts differently in the oracle and conformal settings. While for the oracle setting randomization as per scheme 2 corresponds to recoloring the purple points to either green (leftmost color, class $0$) or blue (rightmost color, class $1$) depending on the most likely label, for the conformal setting two schemes yield conceptually different prediction sets. Presented visualizations might be misleading regarding the role of randomization for conformal classification as they suggest the non-randomized conformal method is the optimal one. See Figure~\ref{fig:conf_toy_example_binary_marg_char} and Section~\ref{appsubsec:randomization} for more details.}
    \label{fig:conf_toy_example_binary_diff_rand_schemes}
\end{figure*}

\begin{figure*}
    \centering
    \begin{subfigure}{0.45\textwidth}
        \centering
       \includegraphics[width=\textwidth]{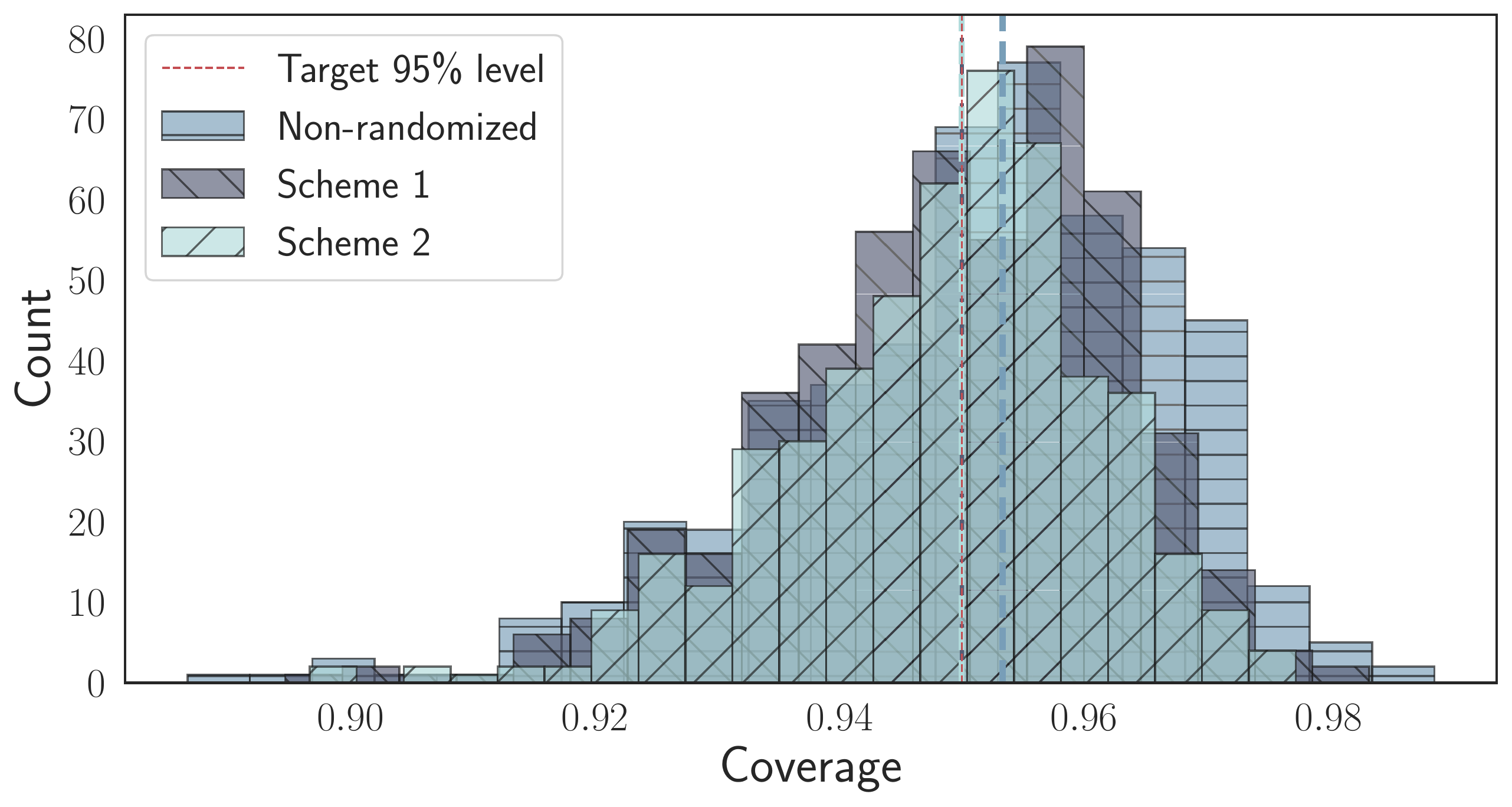}
    \caption{}
    \label{subfig:conf_toy_example_marg_coverage}
    \end{subfigure}%
    ~ 
    \begin{subfigure}{0.45\textwidth}
        \centering
     \includegraphics[width=\textwidth]{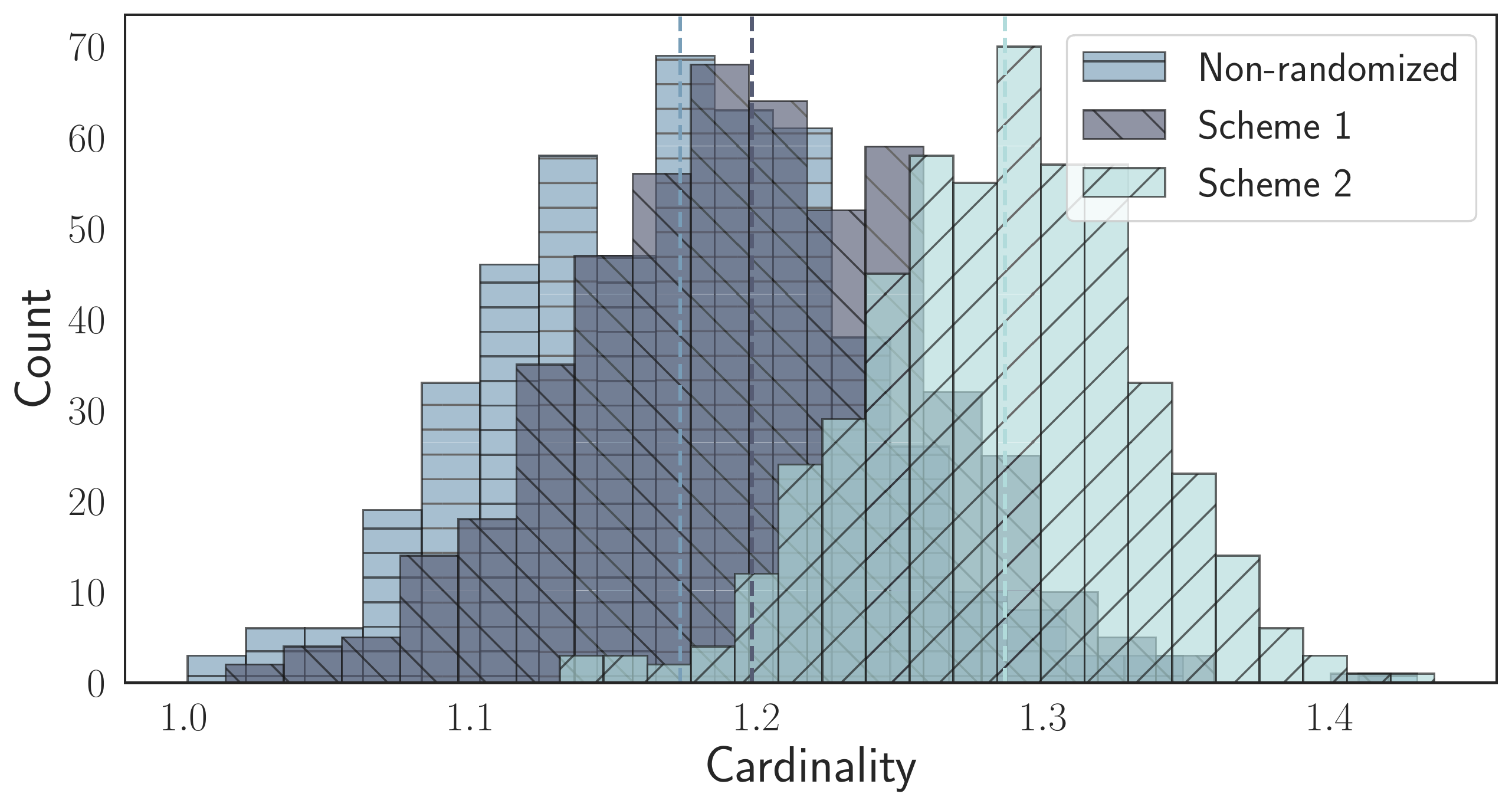}
    \caption{}
    \label{subfig:conf_toy_example_marg_size}
    \end{subfigure}
    ~ 
    \begin{subfigure}{0.45\textwidth}
        \centering
       \includegraphics[width=\textwidth]{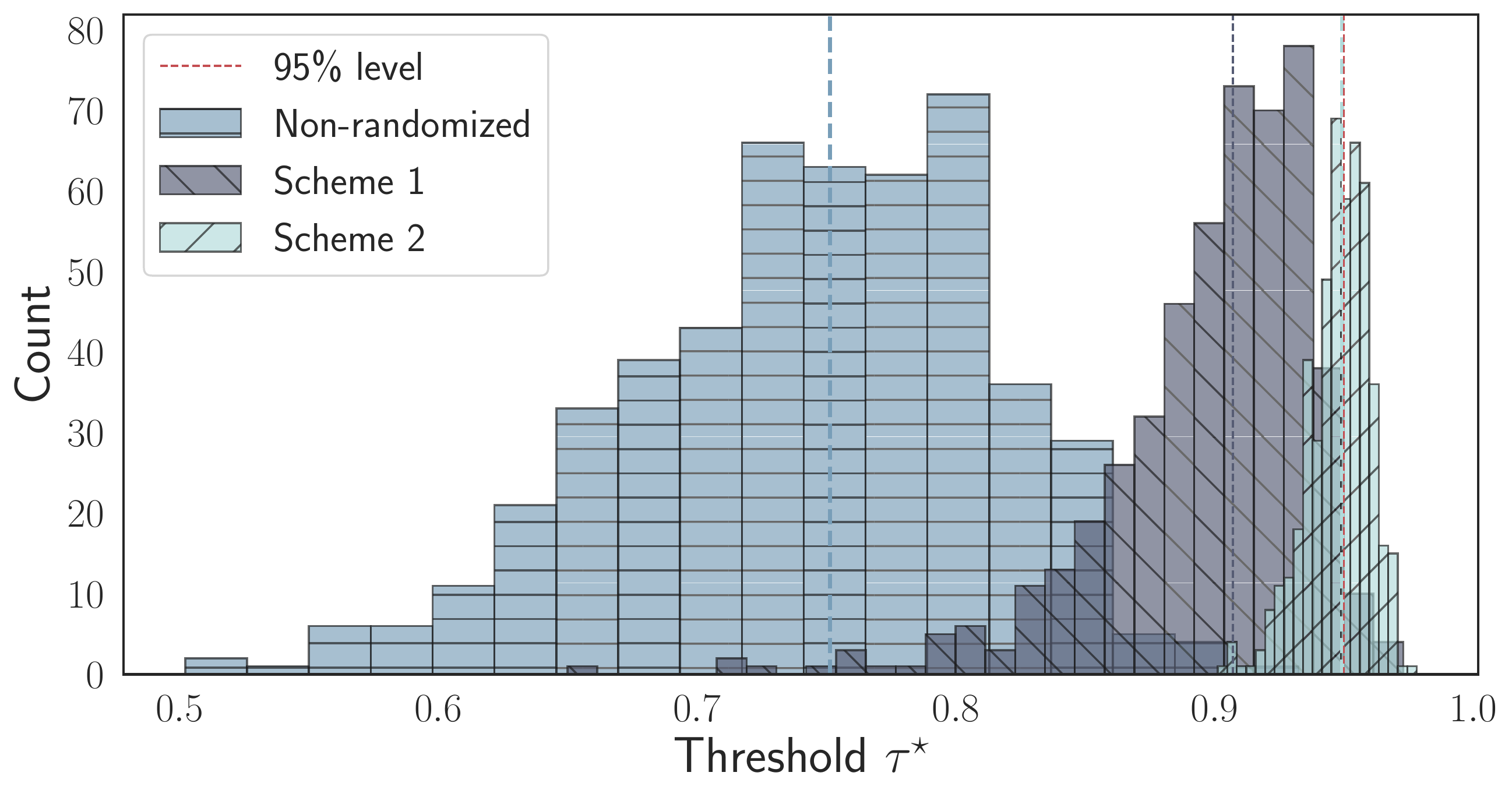}
    \caption{}
    \label{subfig:conf_toy_example_tau}
    \end{subfigure}
    ~ 
    \begin{subfigure}{0.45\textwidth}
        \centering
       \includegraphics[width=\textwidth]{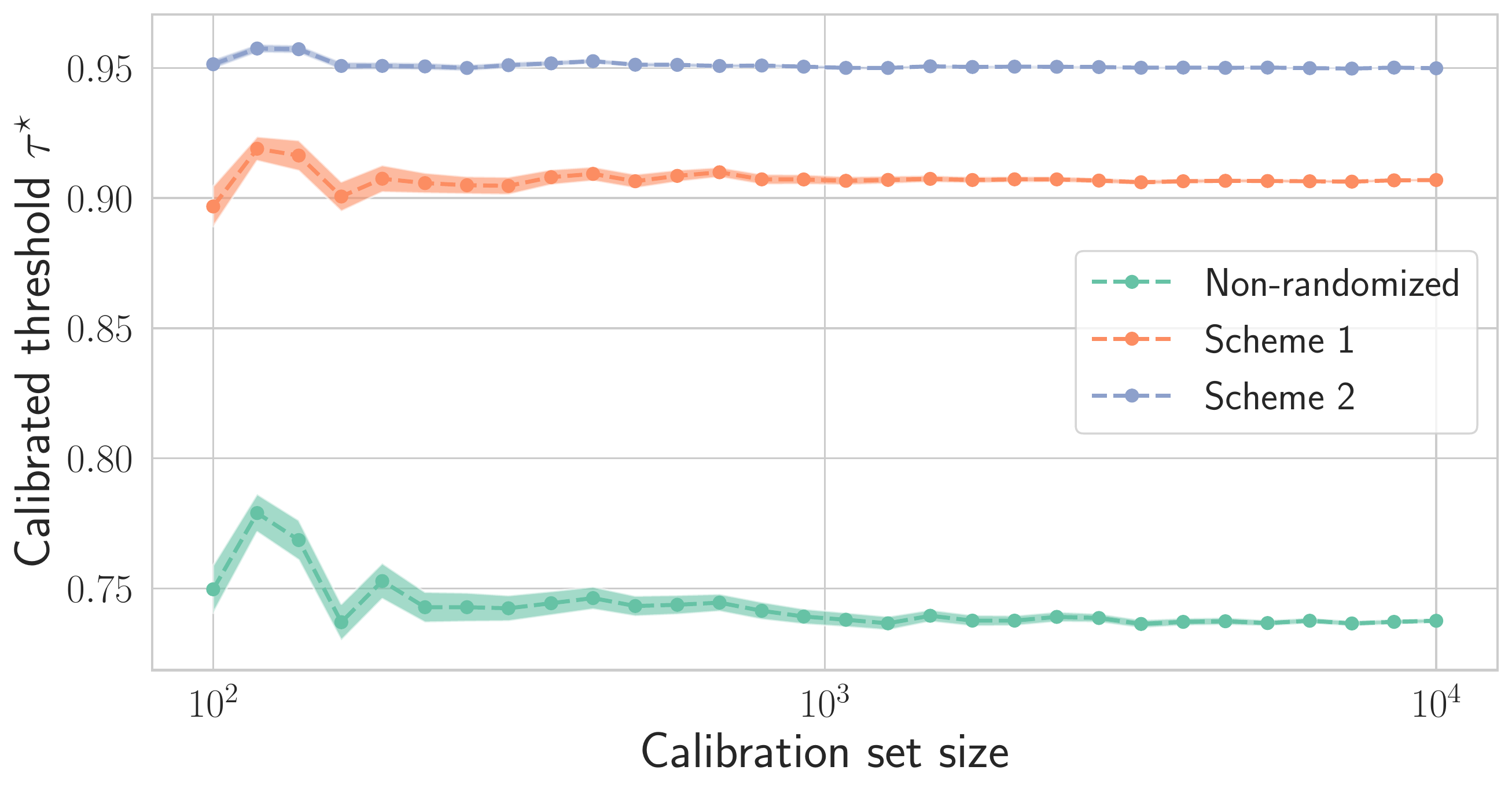}
    \caption{}
    \label{subfig:conf_toy_example_tau_change_cal_set_size}
    \end{subfigure}
\caption{Characteristics of conformal prediction sets for the simulation in Section~\ref{appsubsec:randomization}: (\subref{subfig:conf_toy_example_marg_coverage}) average marginal coverage, (\subref{subfig:conf_toy_example_marg_size}) average cardinality, (\subref{subfig:conf_toy_example_tau}) learned cut-off thresholds in each setting (appending empty prediction sets with the most-likely label does not impact the threshold), (\subref{subfig:conf_toy_example_tau_change_cal_set_size}) learned cut-off thresholds in each setting when increasing the size of the calibration  set. Key takeaways include: (i) marginal coverage requirement is met irrespective of whether conformal method performs randomization or not, (ii) the fact that randomization yields larger prediction sets, and thus is inferior is misleading, (iii) as in considered the example the conformal method recovers the oracle if learned threshold $\tau^\star=0.95$, only randomized (scheme 1) one does it, (iv) the cut-off thresholds do not depend much on the size of the calibration dataset.}
    \label{fig:conf_toy_example_binary_marg_char}
\end{figure*}

\subsection{Proofs}\label{appsubsec:conformal_proofs}

\begin{proof}[Proof of Theorem~\ref{thm:guarantee_exch}]

By the definition of the conformal prediction set, $Y_{n+1}\in \calF_{\tau^\star}(X_{n+1}, U_{n+1};\widehat{\pi})$ if and only if:
\begin{equation*}
    r(X_{n+1},Y_{n+1}, U_{n+1};\widehat{\pi})\leq \quant_{1-\alpha} \roundbrack{\curlybrack{r_i}_{i\in\calI_2}\cup \curlybrack{1}}.
\end{equation*}
As the non-conformity scores $\curlybrack{r_i}_{i=1}^{n+1}$ are exchangeable random variables for any fixed $\widehat{\pi}$, Lemma~\ref{lem:quanile_lemma} implies the desired result conditional on $\curlybrack{(X_i,Y_i)}_{i\in\calI_1}$. Finally, when randomization is performed, the scores are uniformly distributed, and thus Lemma~\ref{lem:quanile_lemma} implies that the marginal coverage is nearly tight.
\end{proof}

\begin{proof}[Proof of Theorem~\ref{thm:oracle_imp_weights_ps}]

First, recall the definition of weighted exchangeability~\citep{tibs2019conf}.
\begin{definition}[Weighted exchangeability]
Random variables $Z_1,\dots,Z_n$ are said to be \emph{weighted exchangeable}, with weight functions $\omega_1,\dots,\omega_n$, if the density $f$ of their joint distribution can be factorized as:
\begin{equation*}
    f(z_1,\dots,z_n) = \prod_{i=1}^n \omega_i(z_i)\cdot g(z_1,\dots,z_n),
\end{equation*}
where $g$ is any function that that invariant to permutations of its arguments, i.e., $g(z_{\sigma(1)},\dots,z_{\sigma(n)})$ for any permutation $\sigma$ of $1,\dots,n.$
\end{definition}
Independent draws are always weighted exchangeable and it is easy to see that under label shift setting $Z_i=(X_i,Y_i,U_i)$, $i = 1,\dots, n+1$ are weighted exchangeable with $\omega_i\equiv 1$, $i=1,\dots,n$ and $\omega_{n+1}((x,y)) = q(y)/p(y)$, for any pair $(x,y)\in\calX\times\calY$. Let $r_{n+1}:=r\roundbrack{X_{n+1},Y_{n+1}, U_{n+1};\widehat{\pi}}$. By construction $Y_{n+1}\in  \calF^{(w)}_{\tau^\star} \roundbrack{X_{n+1}, U_{n+1};\widehat{\pi}}$ if and only if:
\begin{equation*}
 r_{n+1} \leq \quant_{1-\alpha}\roundbrack{\sum_{i=1}^{n} \tilde{p}_i^w(Y_{n+1})\delta_{r_i}+ \tilde{p}_{n+1}^w(Y_{n+1})\delta_1}.
\end{equation*}
Under label shift assumption, weights~\eqref{eq:weight_coef} do simplify as
\begin{equation*}
\begin{aligned}
    p_i^w(Z_1,\dots,Z_{n+1}) &= \frac{\sum_{\sigma:\sigma(n+1)=i}w_{n+1}(Z_{i})}{\sum_{\sigma}w_{n+1}(Z_{\sigma(n+1)})}\\
    &= \frac{w(Y_{i})}{\sum_{j=1}^{n}w(Y_j)+w(Y_{n+1})}\\
    & = \tilde{p}_i^w(Y_{n+1}),
\end{aligned}
\end{equation*}
for $i=1,\dots,n+1$ matching the ones stated in~\eqref{eq:threshold_weighted_case}. The result follows by invoking Lemma~\ref{lem:weight_quant_lemma}. As $\widehat{\pi}$ is fixed at the calibration step being pre-computed on a separate part of the dataset split, the result is conditional on $\curlybrack{(X_i,Y_i)}_{i\in\calI_1}$.
\end{proof}

\begin{proof}[Proof of Corollary~\ref{cor:asympt_est_imp_weights}]
As for the other results, here it is also conditional on the training data, and thus we omit writing $\curlybrack{(X_i,Y_i)}_{i\in\calI_1}$ for succinctness and we use $r_{n+1} =r\roundbrack{X_{n+1},Y_{n+1},U_{n+1};\widehat{\pi}}$ to denote the radius for the test point. Choose an arbitrary $\varepsilon>0$. We have:
\begin{align}
     & \Prob\roundbrack{Y_{n+1}\notin \calF^{(\widehat{w}_k)}_{\tau^\star} \roundbrack{X_{n+1}, U_{n+1};\widehat{\pi}}} \nonumber \\
    = \quad & \Prob\roundbrack{ r_{n+1}>\tau^\star_{\widehat{w}_k}(Y_{n+1})} \label{eq:est_weight_proof}\\
    =\quad &  \Prob\roundbrack{ \curlybrack{r_{n+1}>\tau^\star_{\widehat{w}_k}(Y_{n+1})} \cap \curlybrack{r_{n+1}+\varepsilon>\tau^\star_{w}(Y_{n+1})}} \nonumber\\
    + \quad & \Prob\roundbrack{ \curlybrack{r_{n+1}>\tau^\star_{\widehat{w}_k}(Y_{n+1}}) \cap \curlybrack{r_{n+1}+\varepsilon\leq\tau^\star_{w}(Y_{n+1})}}. \nonumber
\end{align}
We have that:
\begin{equation*}
\begin{aligned}
    & \Prob\roundbrack{ r_{n+1}\geq\tau^\star_{w}(Y_{n+1})}
    = \Prob\roundbrack{ r_{n+1}>\tau^\star_{w}(Y_{n+1})} <\alpha,
\end{aligned}
\end{equation*}
where equality is due to the fact that $r_{n+1}$ in the randomized scheme has a continuous distribution and inequality is due to Theorem~\ref{thm:oracle_imp_weights_ps}. For the first term in~\eqref{eq:est_weight_proof} we have:
\begin{equation*}
\begin{aligned}
     & \Prob\roundbrack{ \curlybrack{r_{n+1}>\tau^\star_{\widehat{w}_k}(Y_{n+1}}) \cap \curlybrack{r_{n+1}+\varepsilon>\tau^\star_{w}(Y_{n+1})}} \\
    = \quad & \Prob\roundbrack{ \curlybrack{r_{n+1}>\tau^\star_{\widehat{w}_k}(Y_{n+1}}) \cap \curlybrack{r_{n+1}>\tau^\star_{w}(Y_{n+1})-\varepsilon}}\\
   \leq \quad & \Prob\roundbrack{r_{n+1}>\tau^\star_{w}(Y_{n+1})-\varepsilon},
\end{aligned}
\end{equation*}
and for the second term we have that:
\begin{equation*}
\begin{aligned}
      \Prob\roundbrack{ \curlybrack{r_{n+1}>\tau^\star_{\widehat{w}_k}(Y_{n+1})} \cap \curlybrack{r_{n+1}\leq\tau^\star_{w}(Y_{n+1})-\varepsilon}} 
   \leq \Prob\roundbrack{\abs{\tau^\star_{\widehat{w}_k}(Y_{n+1})-\tau^\star_{w}(Y_{n+1})}\geq \varepsilon}.
\end{aligned}
\end{equation*}
Note that $\varepsilon$ was chosen arbitrarily, so we can let $\varepsilon\rightarrow 0$. By the continuous mapping theorem, consistency of $\widehat{w}_k$ implies that of $\tau^\star_{\widehat{w}_k}(y)$, $y\in\calY$. Thus,
\begin{equation*}
    \lim_{k\rightarrow\infty} \Prob\roundbrack{Y_{n+1}\in \calF^{(\widehat{w}_k)}_{\tau^\star} \roundbrack{X_{n+1}, U_{n+1};\widehat{\pi}}} \geq 1-\alpha,
\end{equation*}
which concludes the proof of the Corollary.
\end{proof}

\subsection{Simulation on real data}\label{appsubsec:conformal_real_data}

For the simulation in Section~\ref{subsec:conf_label_shift} we use \texttt{wine quality} dataset~\citep{data_wine} to illustrate the performance of the conformal prediction sets when label shift is (not) taken into account. We focus on white wine dataset only, which has 4898 instances with 11 features and construct a 3-class classification problem by keeping classes 5,6,7 only to avoid complications arising due to high imbalance in the dataset (less than $10\%$ of the data points were removed). Other important aspects include
\begin{enumerate}
    \item \textbf{Data Split:} First, the original dataset $\calD$ is split into two disjoint and approximately equal sets $\calD_1$ and $\calD_2$. Then label shift is simulated via resampling according to considered class proportions yielding $\widetilde{\calD}_1$ and $\widetilde{\calD}_2$ of the same size. Finally, the former dataset is split at random into sets for training ($\approx 1000$ instances), calibration ($\approx 100$ instances) and importance weights estimation ($\approx 700$ instances) and the latter is split is split at random into importance weight estimation ($\approx 1000$ instances; recall that only labels from the target are used) and test ($\approx 1600$ instances) sets. 
    \item \textbf{Model:} We use a standard Feed Forward Neural Network with 3 hidden layers with (128,64,32) neurons and $\ell_2$-regularization in each as an underlying model. We use Adam optimizer with default parameters, set the maximum number of training epochs to 500 and deploy Early Stopping with patience for 25 epochs.
    \item \textbf{Estimating label shift:} We use BBSE-soft~\citep{lipton2018label} for estimating importance weights.
\end{enumerate}

\subsection{Marginal (standard) conformal versus label-conditional conformal}\label{appsubsec:lcc_marg_conf}

Various procedures of performing label-conditional conformal prediction have been proposed in a series of works~\citep{vovk2005algorithmic,vovk2016criteria,sadinle2019label,guan2019prediction}. Those are based on a slight modification of the standard conformal p-value used to determine whether there is enough evidence to exclude given label from the prediction set. Roughly speaking, for each candidate label $y$ instead of looking whether a pair $(X_{n+1},y)$ conforms well to the whole collection of points $\widetilde{\calD}=\curlybrack{(X_i,Y_i)}_{i\in\calI}$, one considers only the subcollection that shares the same label $y$. Since the standard exchangeability argument immediately implies validity, the difference then lies in a particular choice for the underlying (non-)conformity score. For example, one could design a score that aims to minimize expected size of the prediction set~\cite{sadinle2019label,guan2019prediction}.

We now apply label-conditional split-conformal framework to the setting discussed in this work and focus on the case of not well-separated data. Consider, for example, the data simulation pipeline from Section~\ref{subsec:conf_label_shift}. First, we fix $\alpha_y=\alpha=0.1$ for all $y\in\calY$ and illustrate the difference between label-conditional conformal~\eqref{eq:pred_sets_label_cond} and standard conformal~\eqref{eq:exchan_conf_pred_set} prediction sets with the same randomized non-conformity scores~\eqref{eq:conf_score_eqiv} for a fair comparison on Figure~\ref{fig:cond_vs_uncond_conformal}. In both cases a shallow MLP (two layers with 100 hidden units in each) is used as an underlying predictor. In this particular example a stronger requirement of conditional validity forces many prediction sets to be larger and to contain the least populated class $1$. 

Then we perform 1000 simulations and compare label-conditional conformal against marginal conformal in two settings (in all cases prediction sets are forced to contain the most likely label for a fair comparison). First, we set the calibration set size to be $\approx 350$ data points and compare two procedures depending on whether class proportions change, and in the former case we perform reweighting of the non-conformity scores as described in Section~\ref{subsec:conf_label_shift}. On Figure~\ref{subfig:larg_cal_source} we observe that when class proportions do not change label-conditional conformal yields larger prediction sets as opposed to standard marginal conformal due to a stronger coverage requirement. However, when class proportions change, after performing the reweighting with the true label likelihood ratios, both procedures output prediction sets of similar size on average as illustrated on Figure~\ref{subfig:larg_cal_target}. Motivated by reasons related to the practical limits of data resources when keeping a sufficiently large held-out set per label could become prohibitive, we also consider a setting when the calibration set contains $\approx 100$ data points (total). Smaller calibration set size results in losses of statistical power when testing whether a given label should be included into the prediction set, and thus, might yield larger prediction sets as observed on Figure~\ref{subfig:small_cal_card}.

To summarize, label-conditional conformal is a complementary (and a powerful) technique to label-shifted conformal that is inherently robust to changes in class proportions. It does not require importance weights, and thus can yield exact finite-sample guarantees. Still, it has certain limitations: (a) it might be potentially a bit conservative in certain areas of the sample space where classes overlap, (b) it requires further splitting of the calibration set that could have negative impact, especially when the number of classes $K$ is large, a common setting for the modern datasets.

\begin{figure*}
    \centering
    \begin{subfigure}{0.45\textwidth}
        \centering
        \includegraphics[width=\textwidth]{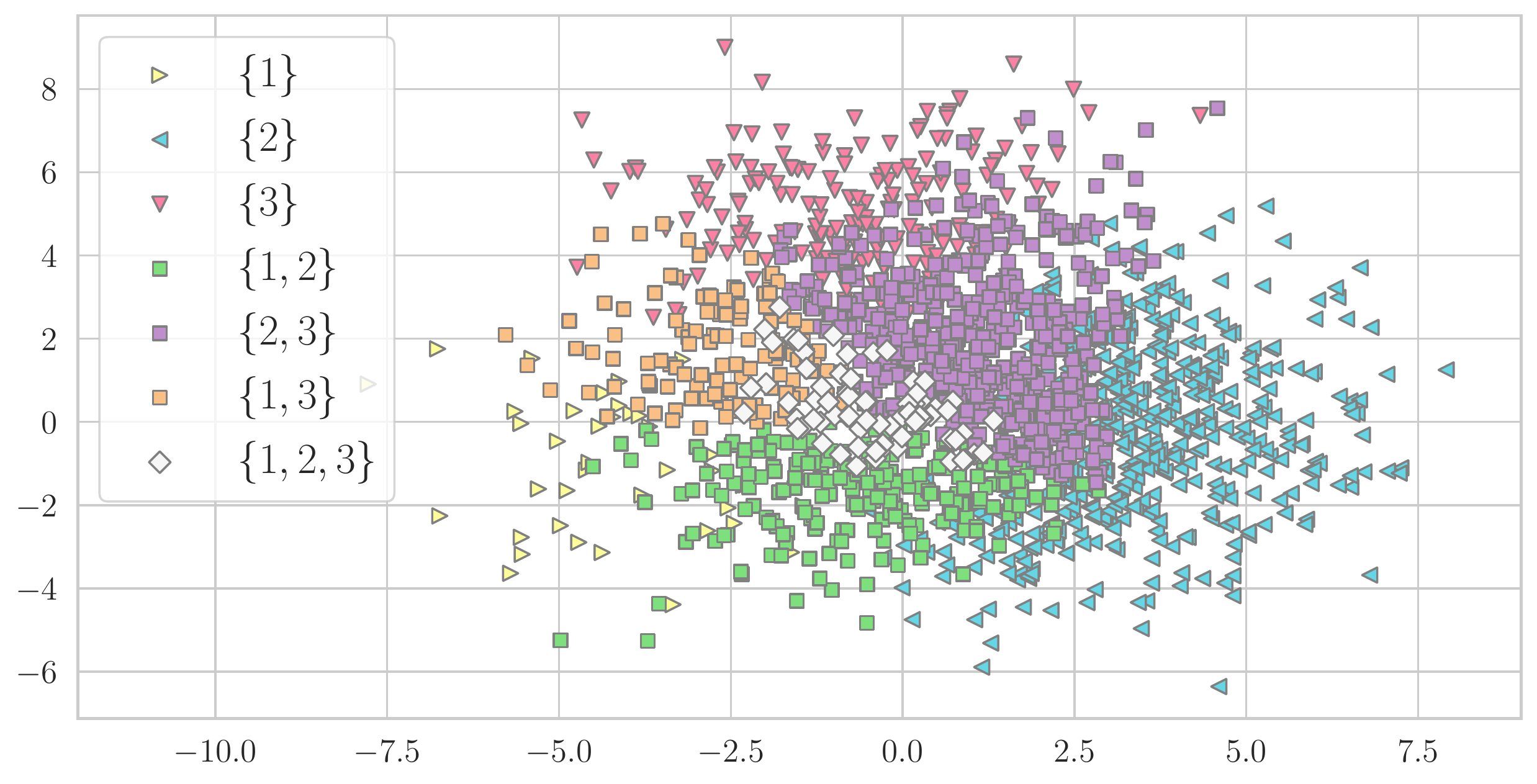}
    \caption{}
    \end{subfigure}%
    ~ 
    \begin{subfigure}{0.45\textwidth}
        \centering
        \includegraphics[width=\textwidth]{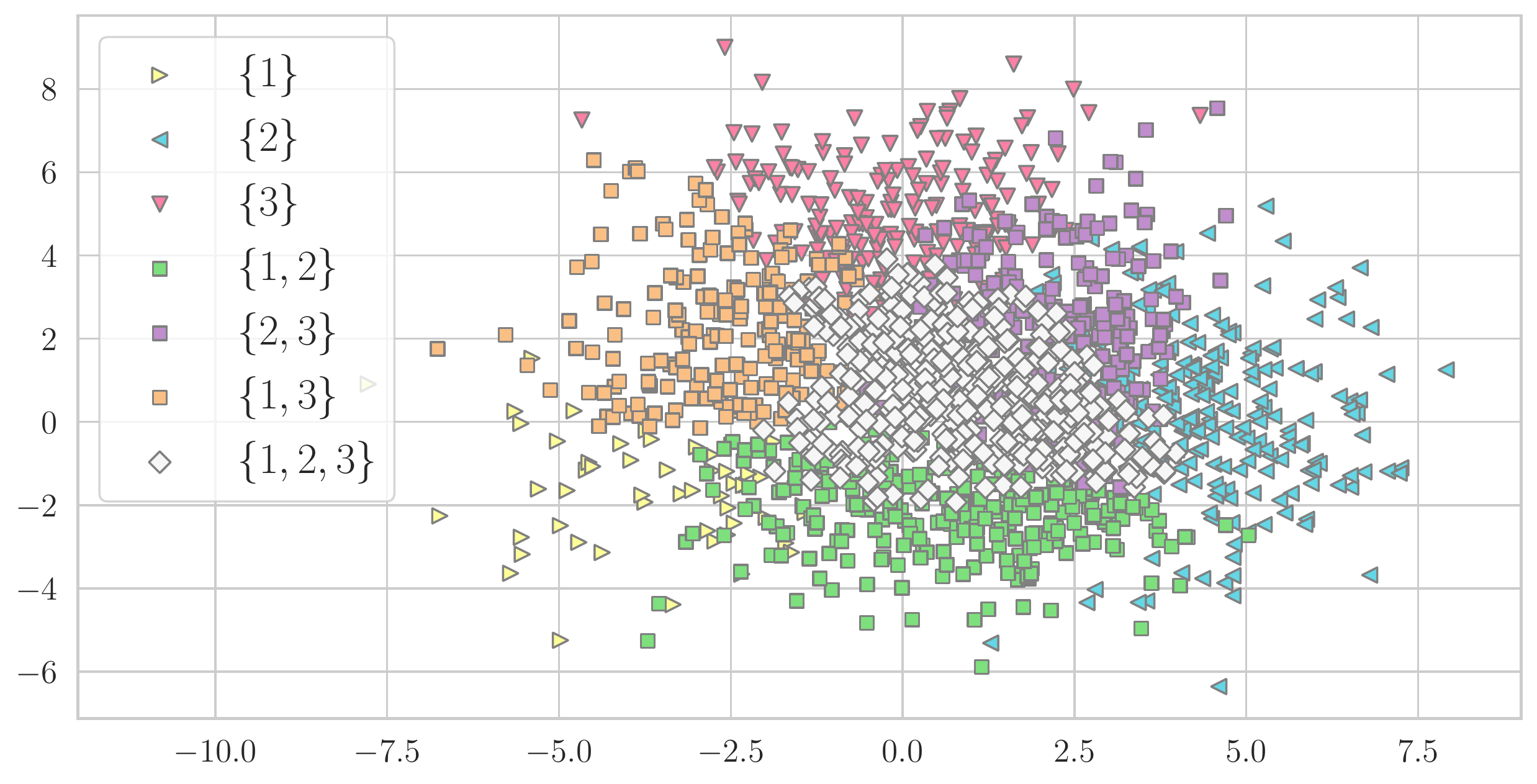}
    \caption{}
    \end{subfigure}
    \caption{(a) Conformal prediction sets with marginal coverage guarantee, (b) Conformal prediction sets with class-specific coverage guarantee. Stronger coverage comes at the price of larger the prediction sets in certain areas.}
    \label{fig:cond_vs_uncond_conformal}
\end{figure*}

\begin{figure*}
    \centering
    \begin{subfigure}{0.45\textwidth}
        \centering
        \includegraphics[width=\textwidth]{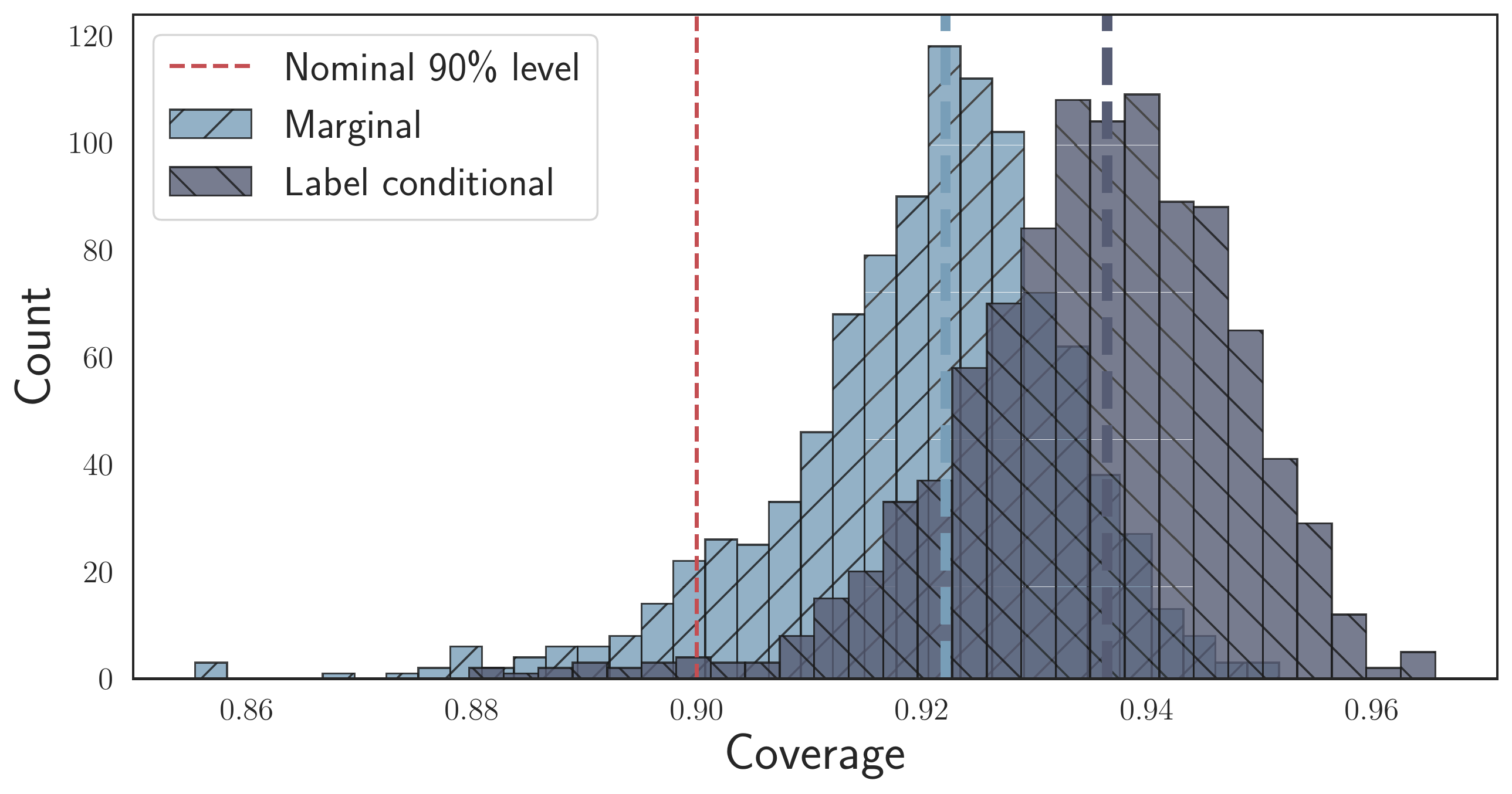}
    \caption{}
    \end{subfigure}%
    ~ 
    \begin{subfigure}{0.45\textwidth}
        \centering
        \includegraphics[width=\textwidth]{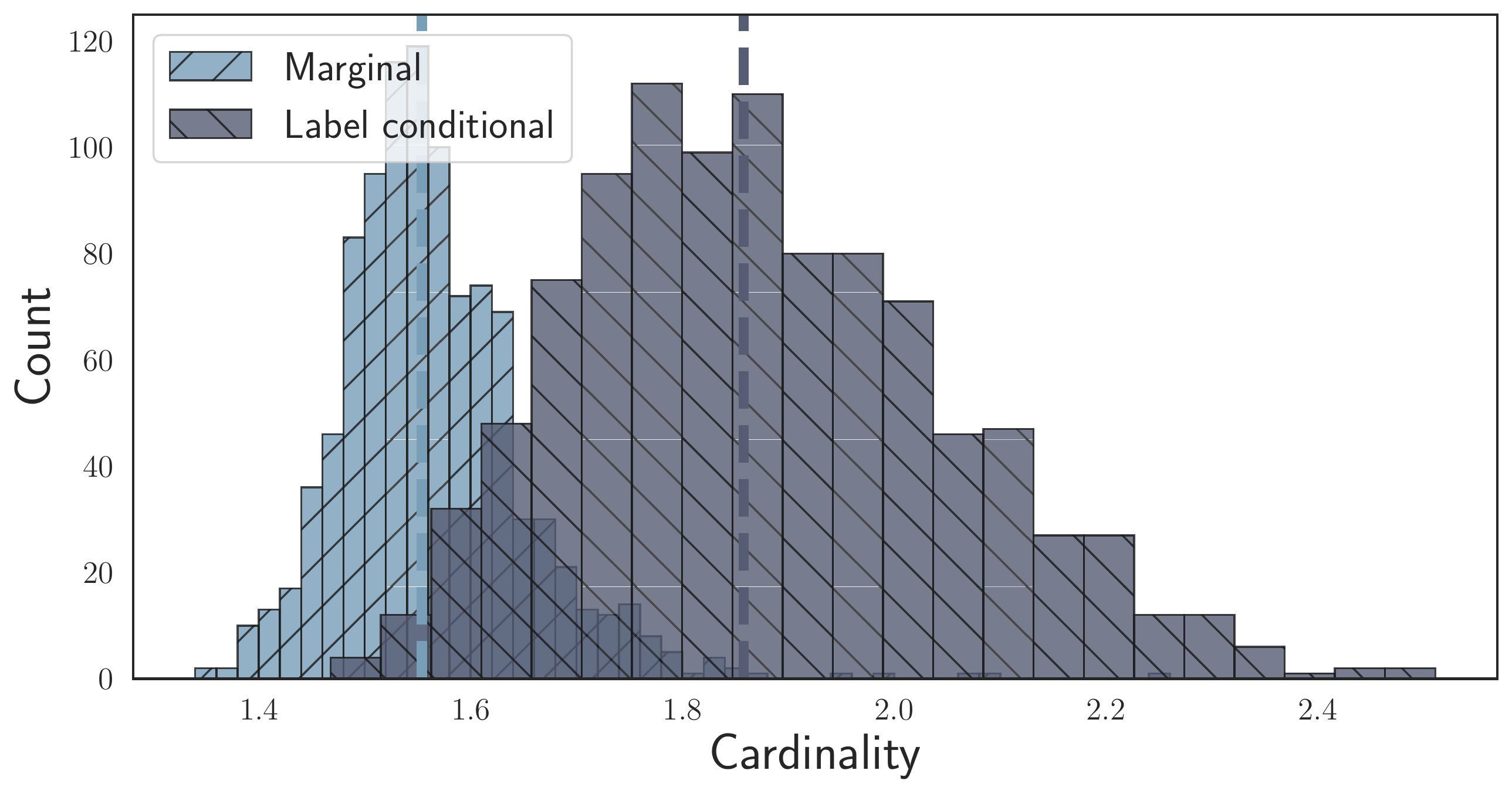}
    \caption{}
    \label{subfig:larg_cal_source}
    \end{subfigure}
    ~ 
    \begin{subfigure}{0.45\textwidth}
        \centering
        \includegraphics[width=\textwidth]{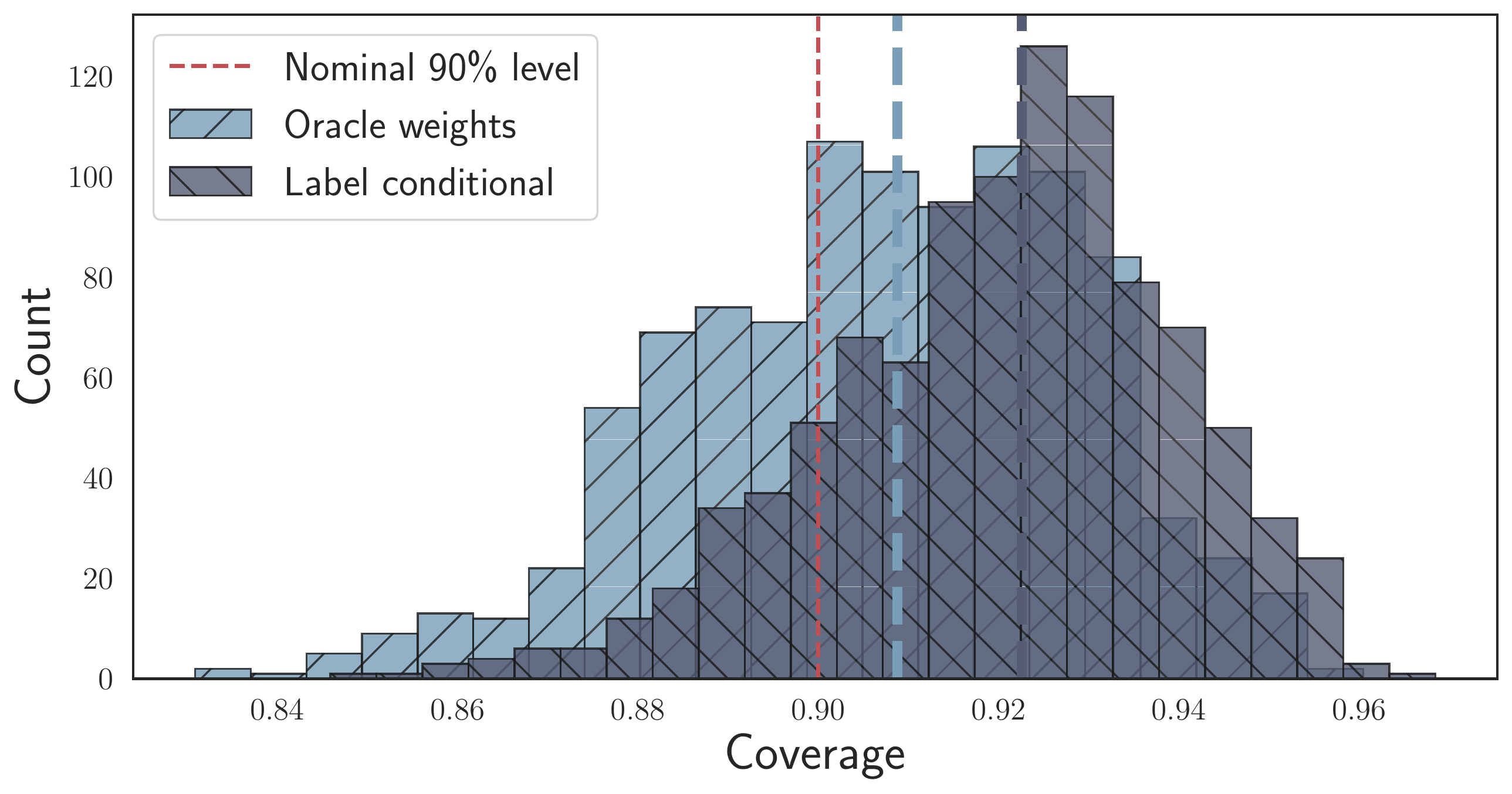}
    \caption{}
    \end{subfigure}%
    ~ 
    \begin{subfigure}{0.45\textwidth}
        \centering
        \includegraphics[width=\textwidth]{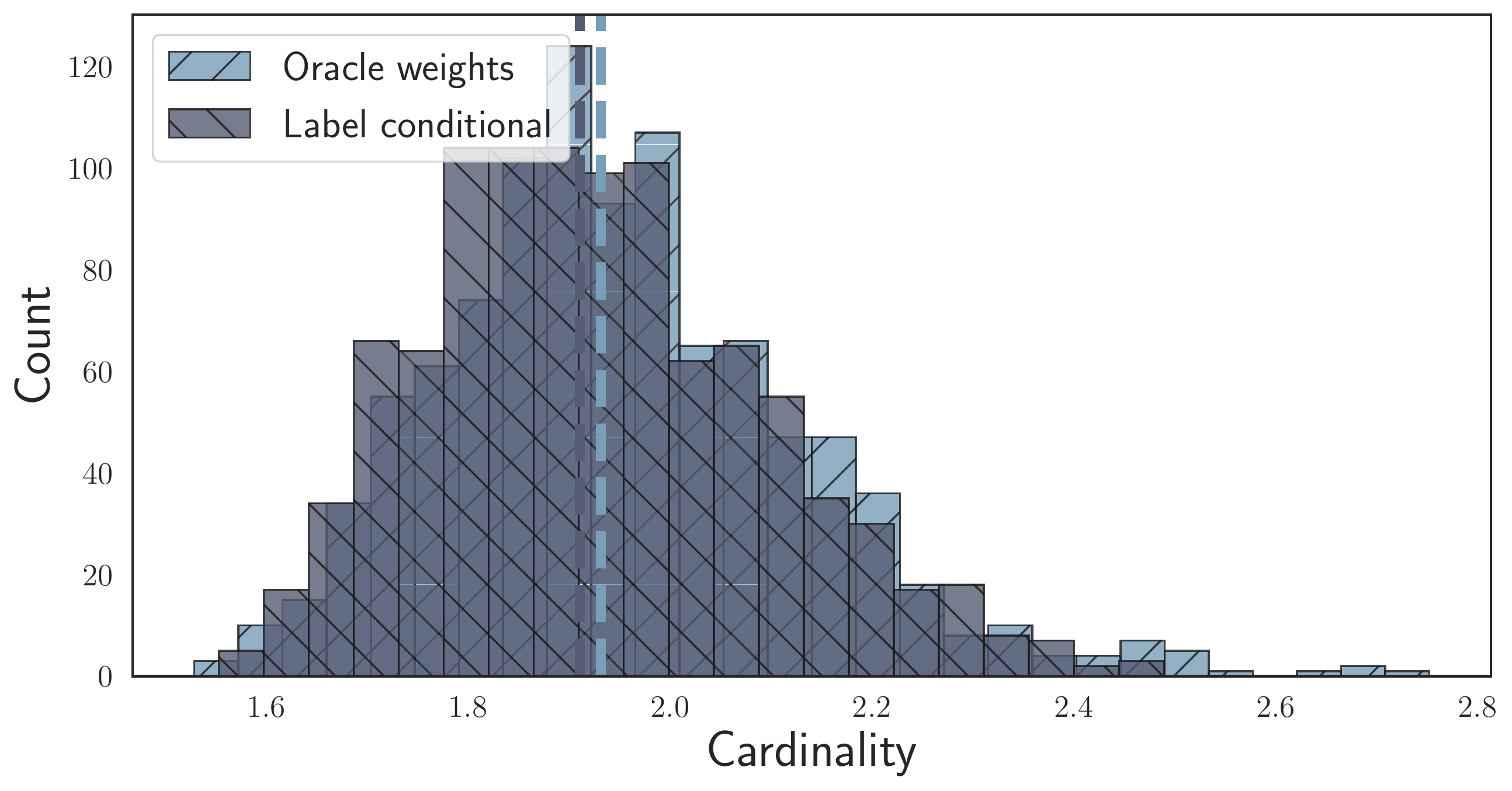}
    \caption{}
    \label{subfig:larg_cal_target}
    \end{subfigure}
    ~ 
    \begin{subfigure}{0.45\textwidth}
        \centering
        \includegraphics[width=\textwidth]{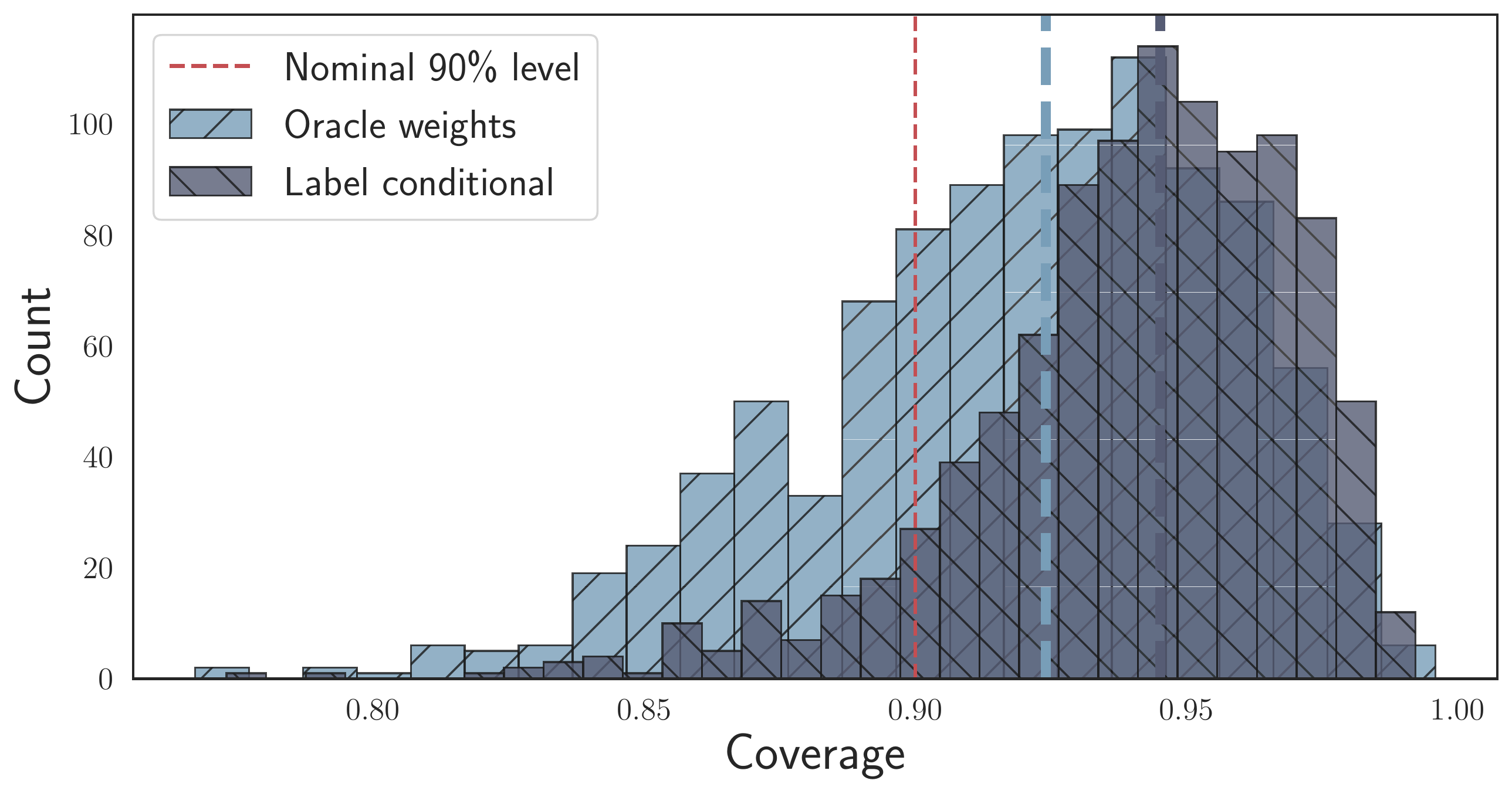}
    \caption{}
    \end{subfigure}%
    ~ 
    \begin{subfigure}{0.45\textwidth}
        \centering
        \includegraphics[width=\textwidth]{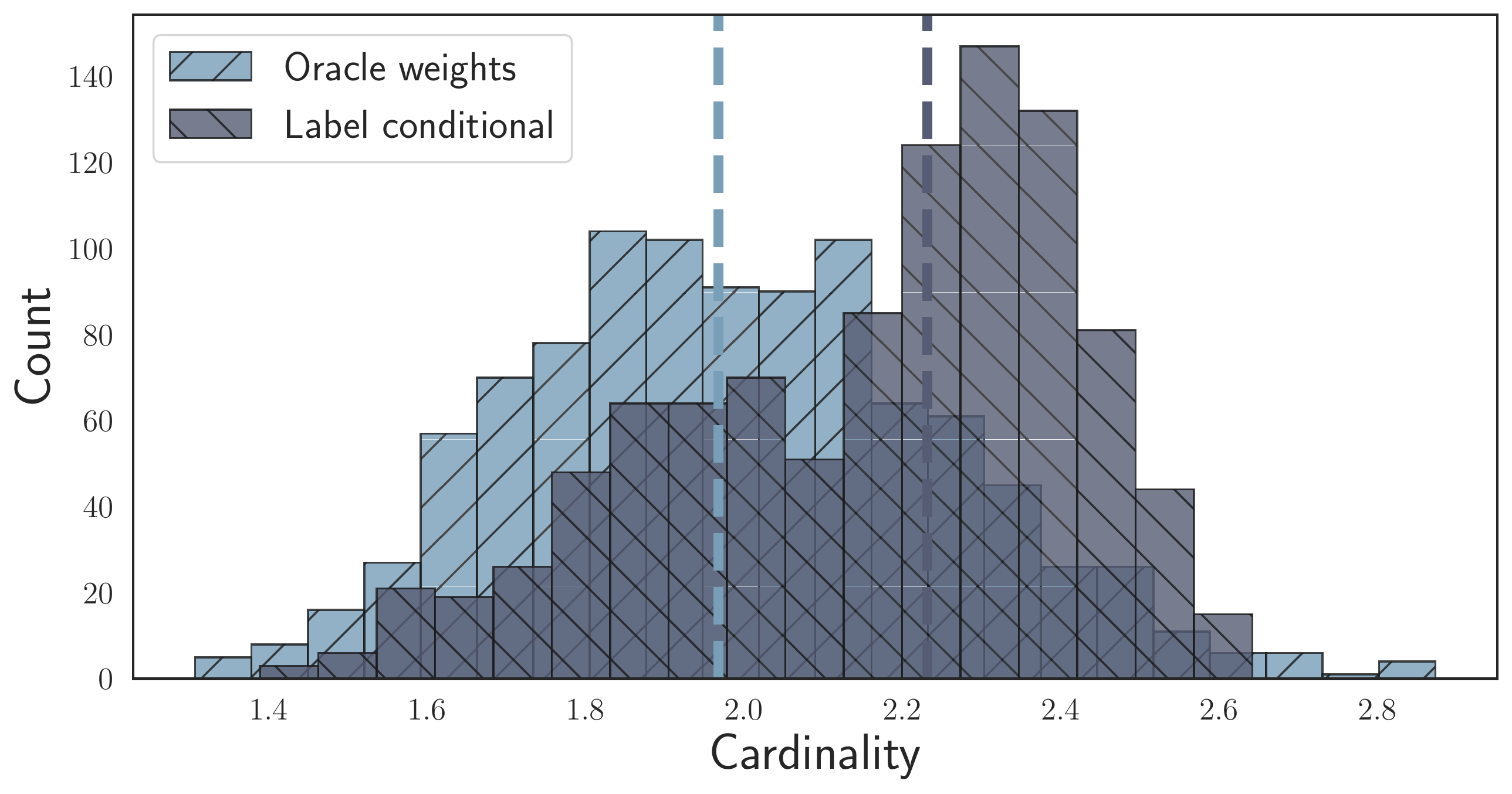}
    \caption{}
    \label{subfig:small_cal_card}
    \end{subfigure}
    \caption{Empirical coverage and average cardinality of conformal prediction sets: (a-b) source distribution and $\approx 350$ calibration data points total, (c-d) target distribution and $\approx 350$ calibration data points total, (e-f) target distribution and $\approx 100$ calibration data points total. Complete comparison of the results is given in Section~\ref{appsubsec:lcc_marg_conf}.}
    \label{fig:cond_vs_uncond_conf_multiple}
\end{figure*}

\section{Calibration}

Section~\ref{appsubsec:calibration_proofs} includes all proofs for Sections~\ref{subsec:calib_iid} and \ref{subsec:calib_label_shift} and Section~\ref{appsubsec:calib_real_data} includes details about the simulation on a real dataset mentioned in Section~\ref{subsec:calib_label_shift}.

\subsection{Proofs}\label{appsubsec:calibration_proofs}

\begin{proof}[Proof of Theorem~\ref{thm:calib_guar_source}]

Recall that $g:\calX\to \calM$ denotes the bin-mapping function. Let $E$ be the event that $\roundbrack{g(X_1),\dots, g(X_n)} = \roundbrack{g(x_1),\dots, g(x_n)}$. On this event, the number of calibration points $N_{m}$ within each bin $B_m$ is known and for each bin labels are i.i.d. with corresponding class probabilities given by $\pi_{y,m}^{P} = \Prob\roundbrack{Y=y\mid f(X)\in B_m}$ for all $y\in \calY$. Thus, a vector corresponding of label frequencies has multinomial distribution with parameters $N_{m}$ and $\curlybrack{\pi_{y,m}^{P}}_{y\in \calY}$. Theorem~\ref{thm:bret_huber_carol} yields that conditional on $E$
\begin{equation*}
    \sum_{y=1}^K\abs{\widehat{\pi}_{y,m}^{P}-\pi_{y,m}^{P}} \geq \frac{2}{\sqrt{N_m}} \sqrt{\frac{1}{2}\ln \roundbrack{\frac{M2^K}{\alpha}}} ,
\end{equation*}
with probability at most $\alpha/M$. Invoking union bound, we get that, conditional on $E$, with probability at least $1-\alpha$,
\begin{equation*}
    \sum_{y=1}^K\abs{\widehat{\pi}_{y,m}^{P}-\pi_{y,m}^{P}}\leq \frac{2}{\sqrt{N_m}} \sqrt{\frac{1}{2}\ln \roundbrack{\frac{M2^K}{\alpha}}},
\end{equation*}
simultaneously for all $m\in\calM$. Since it is true for any $E$, we can marginalize to obtain the first assertion of the Proposition. The second assertion simply represents a consideration of the case when multiple bins happen to have the same calibrated output which is needed to state the desired calibration guarantee. Let
\begin{equation*}
    \varepsilon^\star = \sup_{m\in\calM} \varepsilon_m
\end{equation*}
denote the worst-case bound. Note that $\varepsilon^\star$ is in fact random and to be fully rigorous we, first, perform next steps conditional on $E$ and then marginalize to obtain the assertion. Now, for any $y\in\calY$:
\begin{align}
     & \abs{\Prob\roundbrack{Y=y\mid h(X)} - h_y(X)} \nonumber\\
     = \quad & \abs{\Exp{}{\indicator{Y=y} \mid h(X)}- h_y(X)} \label{eq:calib_prop_proof}\\
     \overset{(a)}{=} \quad &  \abs{\Exp{}{\indicator{Y=y} \mid h(X)}-\Exp{}{h_y(X)\mid h(X)}} \nonumber\\
     \overset{(b)}{=} \quad & \abs{\Exp{}{\Exp{}{\indicator{Y=y} \mid g(X)} \mid h(X)}-\Exp{}{h_y(X)\mid h(X)}} \nonumber \\
     \overset{(c)}{=} \quad & \abs{\Exp{}{\squarebrack{\pi_{y,g(X)}^P- h_y(X)} \mid h(X)}} \nonumber \\
     \overset{(d)}{\leq} \quad & \Exp{}{\abs{\pi_{y,g(X)}^P- \widehat{\pi}_{y,g(X)}} \mid h(X)}, \nonumber
\end{align}
where $(a),(b)$ are due to the tower rule ($h$ is a function of $g$), $(c)$ is due to linearity of conditional expectation and due to definition of $\pi_{y,m}^P$ and, finally, $(d)$ is due to Jensen's inequality. Consider the event:
\begin{equation*}
    E_1: \quad \norm{1}{\widehat{\pi}^{P}_{m}-\pi^{P}_{m}} \leq \varepsilon_m,
\end{equation*}
simultaneously for all $m\in\calM$. Note that the first assertion of the Proposition states event $E_1$ happens with probability at least $1-\alpha$ for chosen $\varepsilon_m$: $\Prob(E_1)\geq 1-\alpha$. Let $E_2$ be the following event:
\begin{equation*}
\begin{aligned}
 E_2: \quad \sum_{y=1}^K\abs{\Prob\roundbrack{Y=y\mid h(X)} - h_{y}(X)} \leq \varepsilon^\star.
\end{aligned}
\end{equation*}
Summing up over labels $y\in\calY$, \eqref{eq:calib_prop_proof} yields that on $E_1$ it holds with probability 1:
\begin{equation*}
\begin{aligned}
 & \sum_{y=1}^K\abs{\Prob\roundbrack{Y=y\mid h(X)} - h_{y}(X)} \\
 \leq\quad & \Exp{}{\norm{1}{\pi_{g(X)}^P- \widehat{\pi}_{g(X)}} \mid h(X)}\\
    \leq \quad & \Exp{}{\varepsilon^\star \mid h(X)}
    = \varepsilon^\star,
\end{aligned}
\end{equation*}
since $\varepsilon^\star$ is a constant. We get that $E_1\subseteq E_2$, and thus $\Prob(E_2)\geq \Prob(E_1)$, and the assertion of the Proposition follows.
\end{proof}
\begin{proof}[Proof of Proposition~\ref{prop:calibration_label_shift}]
The Proposition is a straightforward combination of the Bayes rule and label shift assumption. Given a predictor $f$, for any class label $y\in \calY$ and any bin $B_m$, $m\in\calM=\curlybrack{1,\dots, M}$ one can equivalently represent conditional probabilities with respect to the target distribution as:
\begin{equation*}
    \begin{aligned}
         & \Prob_{Q}\roundbrack{Y=y \mid f(X)\in B_m} \\
        \overset{(a)}{=} \quad &\Prob_{Q}\roundbrack{f(X)\in B_m\mid Y=y}\cdot \frac{\Prob_{Q}\roundbrack{Y=y}}{\Prob_{Q}\roundbrack{f(X)\in B_m}} \\
        \overset{(b)}{=} \quad & \Prob_{P}\roundbrack{f(X)\in B_m\mid Y=y}\cdot \frac{\Prob_{Q}\roundbrack{Y=y}}{\Prob_{Q}\roundbrack{f(X)\in B_m}}\\ 
        \overset{(c)}{=} \quad & \Prob_{P}\roundbrack{Y=y\mid f(X)\in B_m}
        \cdot  \frac{\Prob_{Q}\roundbrack{Y=y}}{\Prob_{P}\roundbrack{Y=y}}\cdot \frac{\Prob_{P}\roundbrack{f(X)\in B_m}}{\Prob_{Q}\roundbrack{f(X)\in B_m}} \\
        =\quad & \Prob_{P}\roundbrack{Y=y\mid X\in B_m}\cdot w(y) \cdot V_m,
    \end{aligned}
\end{equation*}
where $w(y)$ is the importance weight of label $y$ and $V_m$ is the `relative volume' of bin $B_m$. Steps (a), (c) are due to the Bayes rule, (b) is due to label shift assumption. Normalization: $\sum_{k=1}^K\Prob_{Q}\roundbrack{Y=k\mid f(X)\in B_m}=1$, implies that:
\begin{equation*}
\begin{aligned}
V_m & = \frac{1}{\sum_{k=1}^K\pi^P_{k,m}\cdot w(k)}.
\end{aligned}
\end{equation*}
Thus for all bins $m\in\calM$ and labels $y\in\calY$ it holds:
\begin{equation*}
    \pi^Q_{y,m} = \frac{\pi^P_{y,m} \cdot w(y)}{\sum_{k=1}^K\pi^P_{k,m}\cdot w(k)},
\end{equation*}
which concludes the proof of the Proposition. 
\end{proof}


\begin{proof}[Proof of Theorem~\ref{thm:calib_est_imp_weights}]


By triangle inequality, one obtains that for any bin $m\in\calM$:
\begin{align}
    \sum_{y=1}^K\abs{\widehat{\pi}^{(\widehat{w})}_{y,m}-\pi^{Q}_{y,m}} 
    \leq  \sum_{y=1}^K\abs{\widehat{\pi}^{{(w)}}_{y,m}-\pi^{Q}_{y,m}} + \sum_{y=1}^K\abs{\widehat{\pi}^{(\widehat{w})}_{y,m}-\widehat{\pi}^{(w)}_{y,m}}.\label{eq:est_imp_w_triangle_ineq}
\end{align}
Consider the first term in~\eqref{eq:est_imp_w_triangle_ineq}. For any $y\in\calY$:
\begin{equation*}
\begin{aligned}
    & \abs{\widehat{\pi}^{(w)}_{y,m}-\pi^{Q}_{y,m}} \\
    = \quad & \abs{\frac{w(y)\cdot \widehat{\pi}^{P}_{y,m}}{\sum_{k=1}^Kw(k)\cdot \widehat{\pi}^P_{k,m}} - \frac{w(y)\cdot \pi^{P}_{y,m}}{\sum_{l=1}^Kw(l)\cdot \pi^P_{l,m}}} \\
 =\quad &\abs{\frac{\widehat{\pi}^{P}_{y,m}}{\sum_{k=1}^Kw(k)\cdot \widehat{\pi}^P_{k,m}} - \frac{\pi^{P}_{y,m}}{\sum_{l=1}^Kw(l)\cdot \pi^P_{l,m}}} \cdot w(y)\\
     =\quad & \abs{\frac{\widehat{\pi}^{P}_{y,m}}{\sum_{k=1}^Kw(k)\cdot \widehat{\pi}^P_{k,m}} - \frac{\pi^{P}_{y,m}-\widehat{\pi}^{P}_{y,m}+\widehat{\pi}^{P}_{y,m}}{\sum_{l=1}^Kw(l)\cdot \pi^P_{l,m}}} \cdot w(y) \\
         \overset{(a)}{\leq} \quad & \abs{\frac{1}{\sum_{k=1}^Kw(k)\cdot \widehat{\pi}^P_{k,m}} - \frac{1}{\sum_{l=1}^Kw(l)\cdot \pi^P_{l,m}}}\cdot \widehat{\pi}^{P}_{y,m}\cdot w(y)
         +  w(y)\cdot \abs{\frac{\pi^{P}_{y,m}-\widehat{\pi}^{P}_{y,m}}{\sum_{l=1}^Kw(l)\cdot \pi^P_{l,m}}},
\end{aligned}   
\end{equation*}
where $(a)$ is due to triangle inequality. We infer that:
\begin{align*}
    & \sum_{y=1}^K\abs{\widehat{\pi}^{(w)}_{y,m}-\pi^{Q}_{y,m}} \\
    \leq\quad & \abs{1 - \frac{\sum_{k=1}^Kw(k)\cdot \widehat{\pi}^P_{k,m}}{\sum_{l=1}^Kw(l)\cdot \pi^P_{l,m}}}  
     +  \frac{\sum_{y=1}^K w(y) \abs{\pi^{P}_{y,m}-\widehat{\pi}^{P}_{y,m}}}{\sum_{l=1}^Kw(l)\cdot \pi^P_{l,m}} \\
     =\quad & \frac{\abs{ \sum_{k=1}^Kw(k)\cdot \roundbrack{\widehat{\pi}^P_{k,m}-\pi^P_{l,m}}}}{\sum_{l=1}^Kw(l)\cdot \pi^P_{l,m}} 
      +  \frac{\sum_{y=1}^K w(y) \abs{\pi^{P}_{y,m}-\widehat{\pi}^{P}_{y,m}}}{\sum_{l=1}^Kw(l)\cdot \pi^P_{l,m}}\\
     \overset{(a)}{\leq} \quad & 2\cdot \frac{\sum_{y=1}^K w(y) \abs{\pi^{P}_{y,m}-\widehat{\pi}^{P}_{y,m}}}{\sum_{l=1}^Kw(l)\cdot \pi^P_{l,m}} \\
     \overset{(b)}{\leq} \quad & 2\cdot \frac{\roundbrack{\sup_k w(k)} \cdot \sum_{y=1}^K\abs{\pi^{P}_{y,m}-\widehat{\pi}^{P}_{y,m}}}{\sum_{l=1}^Kw(l)\cdot \pi^P_{l,m}},
\end{align*}
where $(a)$ is due to triangle inequality and $(b)$ is due to H\"older's inequality. Observe that for any $m\in\calM$:
\begin{equation*}\begin{aligned}
   \frac{1}{\sum_{k=1}^Kw(k)\cdot \pi^P_{k,m}}
 \leq \frac{1}{\roundbrack{\inf\limits_{k:w(k)\neq 0}w(k)}\cdot \sum_{l=1}^K \pi^P_{l,m}}
 = \frac{1}{\inf\limits_{k:w(k)\neq 0}w(k)},
 \end{aligned}
\end{equation*}
as $\sum_{l=1}^K \pi^P_{l,m} = 1$, $\forall m\in\calM$. Hence, for any $m\in\calM$,
\begin{align}
      \sum_{y=1}^K\abs{\widehat{\pi}^{(w)}_{y,m}-\pi^{Q}_{y,m}} 
    \leq  2\cdot \frac{\sup_k w(k)}{\inf\limits_{k:w(k)\neq 0}w(k)}\cdot \sum_{y=1}^K\abs{\pi^{P}_{y,m}-\widehat{\pi}^{P}_{y,m}}. \label{eq:first_term_bound}
\end{align}
Now, consider the second term in~\eqref{eq:est_imp_w_triangle_ineq}. Observe that:
\begin{equation*}
\begin{aligned}
     & \abs{\widehat{\pi}^{(\widehat{w})}_{y,m}-\widehat{\pi}^{(w)}_{y,m}} \\
     =\quad & \abs{\frac{\widehat{w}(y)\cdot \widehat{\pi}^{P}_{y,m}}{\sum_{k=1}^K\widehat{w}(k)\cdot \widehat{\pi}^P_{k,m}} - \frac{w(y)\cdot \widehat{\pi}^{P}_{y,m}}{\sum_{l=1}^Kw(l)\cdot \widehat{\pi}^P_{l,m}}} \\
 =\quad &  \abs{\frac{\widehat{w}(y)}{\sum_{k=1}^K\widehat{w}(k)\cdot \widehat{\pi}^P_{k,m}} - \frac{w(y)}{\sum_{l=1}^Kw(l)\cdot \widehat{\pi}^P_{l,m}}} \cdot \widehat{\pi}^{P}_{y,m}\\
 =\quad &  \abs{\frac{\widehat{w}(y)}{\sum_{k=1}^K\widehat{w}(k)\cdot \widehat{\pi}^P_{k,m}} - \frac{w(y)-\widehat{w}(y)+\widehat{w}(y)}{\sum_{l=1}^Kw(l)\cdot \widehat{\pi}^P_{l,m}}} \cdot \widehat{\pi}^{P}_{y,m}\\
 \overset{(a)}{\leq} \quad &  \abs{\frac{1}{\sum_{k=1}^K\widehat{w}(k)\cdot \widehat{\pi}^P_{k,m}} - \frac{1}{\sum_{l=1}^Kw(l)\cdot \widehat{\pi}^P_{l,m}}}\cdot \widehat{\pi}^{P}_{y,m}\cdot\widehat{w}(y) 
      +  \frac{\widehat{\pi}^{P}_{y,m}\cdot\abs{ w(y)-\widehat{w}(y)}}{\sum_{l=1}^Kw(l)\cdot \widehat{\pi}^P_{l,m}},
\end{aligned}   
\end{equation*}
where $(a)$ is due to triangle inequality. Thus, 
\begin{equation*}
\begin{aligned}
     & \sum_{y=1}^K\abs{\widehat{\pi}^{(\widehat{w})}_{y,m}-\widehat{\pi}^{(w)}_{y,m}} \\
     \leq\quad &  \abs{\frac{1}{\sum_{k=1}^K\widehat{w}(k)\cdot \widehat{\pi}^P_{k,m}} - \frac{1}{\sum_{l=1}^Kw(l)\cdot \widehat{\pi}^P_{l,m}}}\cdot \sum_{y=1}^K \widehat{\pi}^{P}_{y,m}\cdot\widehat{w}(y)
      +  \frac{\sum_{y=1}^K\widehat{\pi}^{P}_{y,m}\cdot\abs{ w(y)-\widehat{w}(y)}}{\sum_{l=1}^Kw(l)\cdot \widehat{\pi}^P_{l,m}}\\
     =\quad &  \abs{1 - \frac{ \sum_{y=1}^K \widehat{w}(y)\cdot \widehat{\pi}^{P}_{y,m}}{\sum_{l=1}^Kw(l)\cdot \widehat{\pi}^P_{l,m}}}
      +   \frac{\sum_{y=1}^K\widehat{\pi}^{P}_{y,m}\cdot\abs{ w(y)-\widehat{w}(y)}}{\sum_{l=1}^Kw(l)\cdot \widehat{\pi}^P_{l,m}}\\
     = \quad &    \frac{\abs{ \sum_{y=1}^K \roundbrack{w(y)-\widehat{w}(y)}\cdot \widehat{\pi}^{P}_{y,m}}}{\sum_{l=1}^Kw(l)\cdot \widehat{\pi}^P_{l,m}} 
     +   \frac{\sum_{y=1}^K\widehat{\pi}^{P}_{y,m}\cdot\abs{ w(y)-\widehat{w}(y)}}{\sum_{l=1}^Kw(l)\cdot \widehat{\pi}^P_{l,m}}\\
     \leq \quad &  \frac{2 \norm{\infty}{\widehat{w}-w}}{\sum_{l=1}^Kw(l)\cdot \widehat{\pi}^P_{l,m}},\\
\end{aligned}
\end{equation*}
since $\sum_{k=1}^K \widehat{\pi}^P_{k,m} = 1$, $\forall m\in\calM$. Similarly, for any $m\in\calM$:
\begin{equation*}
\begin{aligned}
     \frac{1}{\sum_{k=1}^Kw(k)\cdot \widehat{\pi}^P_{k,m}}
    \leq  \frac{1}{\roundbrack{\inf_{l:w(l)\neq 0}w(l)}\cdot \sum_{k=1}^K \widehat{\pi}^P_{k,m}} 
    =  \frac{1}{\inf_{l:w(l)\neq 0}w(l)}.
\end{aligned}
\end{equation*}
Thus, we get that for any $m\in\calM$:
\begin{equation}
\label{eq:second_term}
    \sum_{y=1}^K\abs{\widehat{\pi}^{(\widehat{w})}_{y,m}-\widehat{\pi}^{(w)}_{y,m}} \leq \frac{2 \norm{\infty}{\widehat{w}-w}}{\inf_{l:w(l)\neq 0}w(l)}.
\end{equation}
Combining bounds~\eqref{eq:first_term_bound} and~\eqref{eq:second_term} with the bound~\eqref{eq:est_imp_w_triangle_ineq}, we obtain that for any $m\in\calM$:
\begin{equation*}
\begin{aligned}
     \sum_{y=1}^K\abs{\widehat{\pi}^{(\widehat{w})}_{y,m}-\pi^{Q}_{y,m}}
    \leq  2\kappa\cdot \sum_{y=1}^K\abs{\widehat{\pi}^{P}_{y,m}-\pi^{P}_{y,m}}
    +  \frac{2 \norm{\infty}{\widehat{w}-w}}{\inf_{l:w(l)\neq 0}w(l)},
\end{aligned}
\end{equation*}
which concludes the proof of the Theorem.
\end{proof}

\subsection{Simulation on real data}\label{appsubsec:calib_real_data}

For the simulation mentioned in Section~\ref{subsec:calib_label_shift} we use \texttt{wine quality} dataset~\citep{data_wine}. The original dataset contains ratings for white wines and we reduce it to a binary classification problem by treating wine as good if the corresponding rating is at least 7 on a 10-point scale. Logistic regression is used as an underlying predictor and for each pass the original dataset $\calD$ is, first, split into two disjoint and approximately equal sets $\calD_1$ and $\calD_2$. Label shift is simulated via resampling of $\widetilde{\calD}_1$ with class proportions $p=(0.8,0.2)$ and $\widetilde{\calD}_2$ with class proportions $(0.5,0.5)$. Final splitting resulted in $\approx  1350$ instances used for both training and calibration, $\approx 700$ and $\approx 400$  instances used for importance weights estimation on the source and the target respectively and $\approx 1100$ instances used for the test. Uniform-mass binning with 10 bins was used for calibration purposes. For 4 random data splits the resulting reliability curves are presented on Figure~\ref{fig:calib_real_data} illustrating that calibration with proper reweighting leads to approximate calibration on the target domain and uncorrected fails to do so.

\begin{figure*}
    \centering
    \begin{subfigure}{0.45\textwidth}
        \centering
        \includegraphics[width=\textwidth]{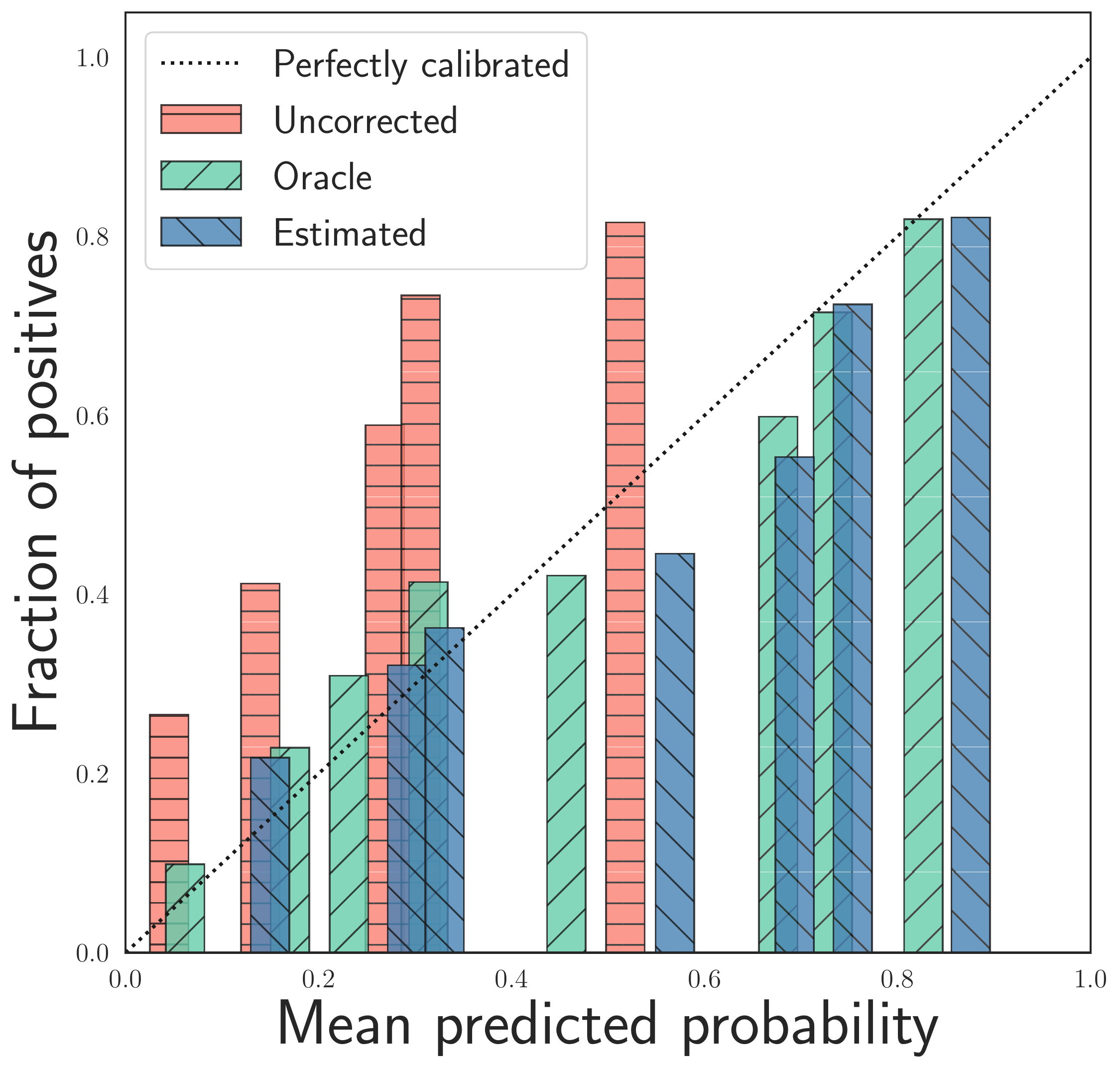}
    \caption{}
    \end{subfigure}%
    ~ 
    \begin{subfigure}{0.45\textwidth}
        \centering
        \includegraphics[width=\textwidth]{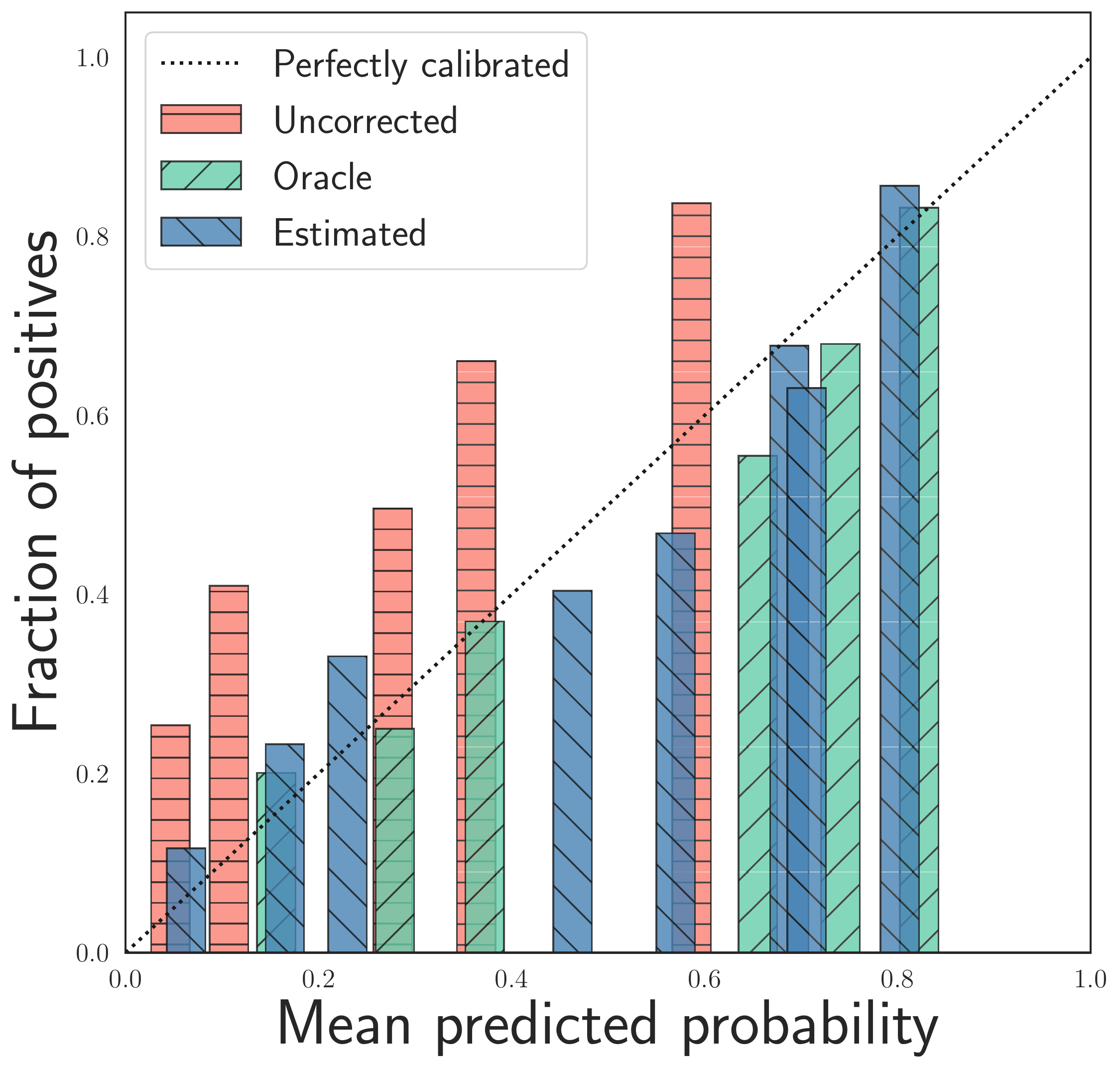}
    \caption{}
    \end{subfigure}
    ~ 
    \begin{subfigure}{0.45\textwidth}
        \centering
        \includegraphics[width=\textwidth]{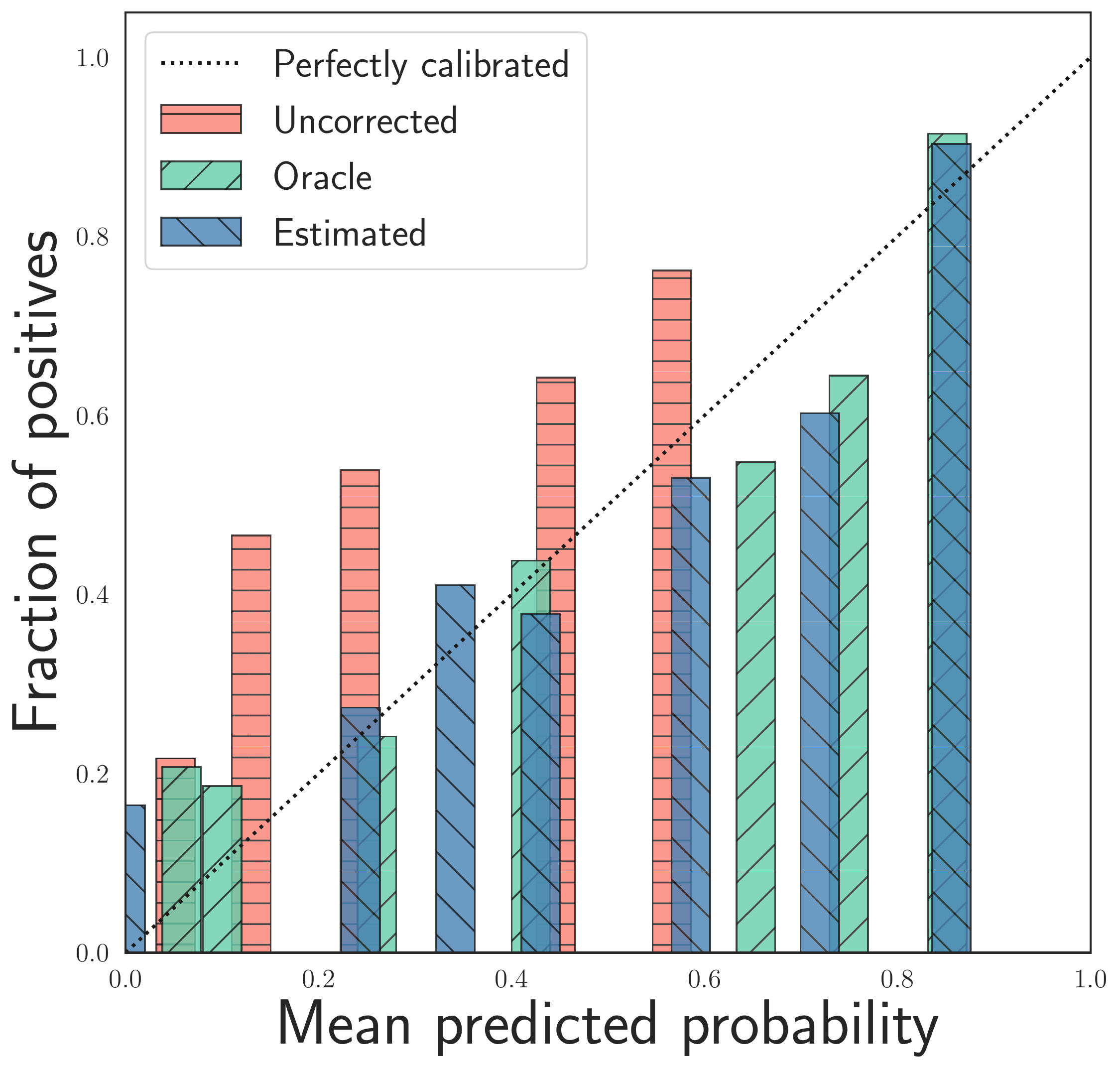}
    \caption{}
    \end{subfigure}
    ~ 
    \begin{subfigure}{0.45\textwidth}
        \centering
        \includegraphics[width=\textwidth]{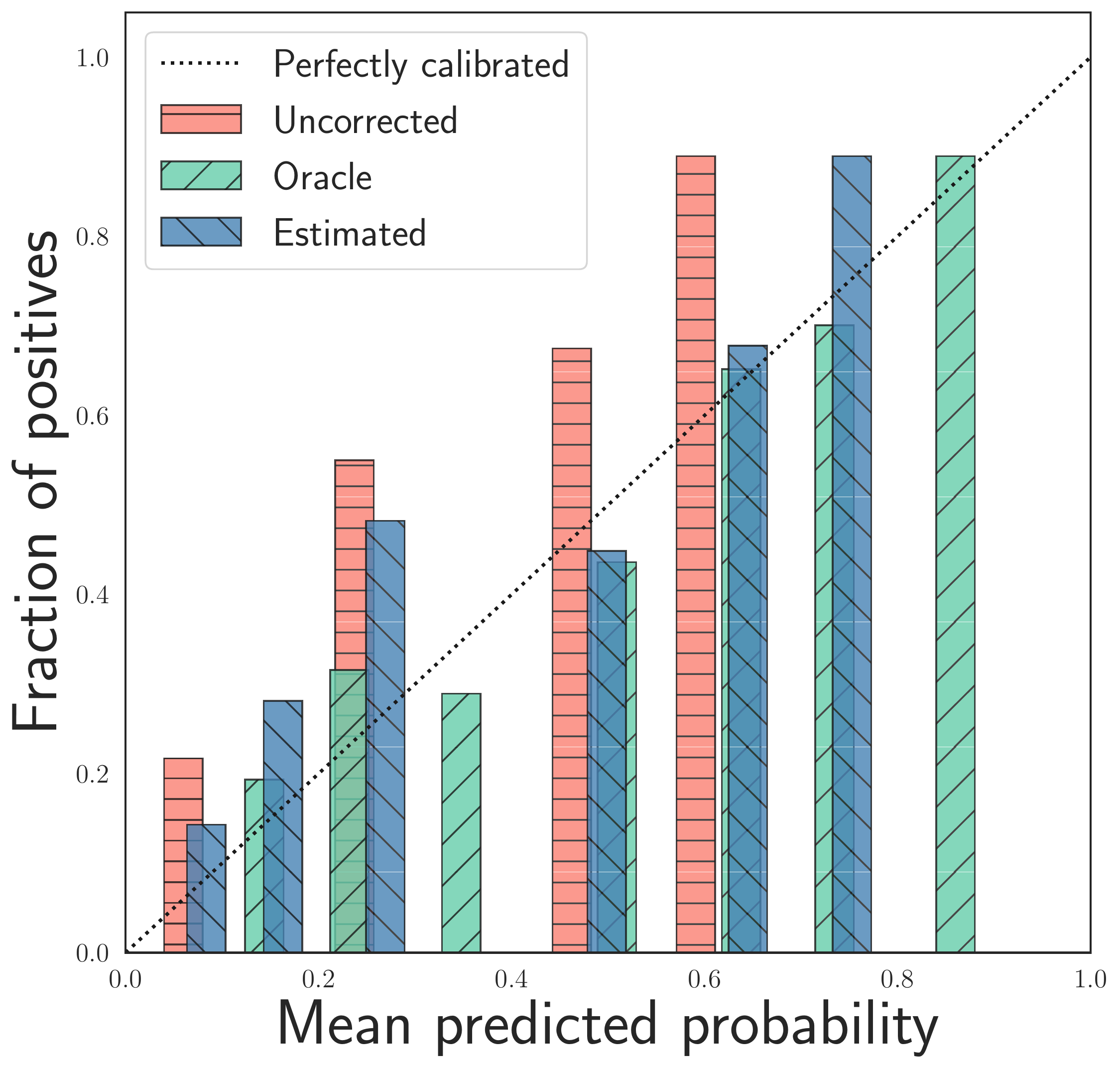}
    \caption{}
    \end{subfigure}
    \caption{Reliability curves for the simulation on the \texttt{wine quality} dataset obtained for several data splits. Notice that the bars indicating calibration using oracle and estimated importance weights  are quite similar to each other, but most importantly that both are very close to the ideal diagonal line (perfect calibration). In contrast, the uncorrected bars are poorly calibrated, demonstrating both the need for handling label shift and the relative success of our procedures in doing so. See Section~\ref{appsubsec:calib_real_data} for details.}
    \label{fig:calib_real_data}
\end{figure*}

\section{Auxiliary results}
Note Lemma~\ref{lem:quanile_lemma} and Lemma~\ref{lem:weight_quant_lemma} were originally formulated for possibly unbounded non-conformity scores. It is easy to see that we can safely replace point masses $\delta_\infty$ by $\delta_1$ in the conformal classification setting considered in this work.

\begin{theorem}[Bretagnolle-Huber-Carol inequality~\citep{vandervaart1996weak}]
\label{thm:bret_huber_carol}
If the random vector $(N_1,\dots,N_k)$ is multinomially distributed with parameters $n$ and $(p_1$, $\dots$, $p_k)$, then
\begin{equation*}
    \Prob\roundbrack{\sum_{i=1}^k\abs{N_i-np_i}\geq 2\sqrt{n}\lambda}\leq 2^ke^{-2\lambda^2},\quad \lambda>0.
\end{equation*}
\end{theorem}

\begin{lemma}[Lemma 1~\citep{tibs2019conf}]\label{lem:quanile_lemma}
Assume $Z_1,\dots,Z_{m+1}$ are exchangeable random variables supported on $[0,1]$. Then for any $\beta\in(0,1)$,
\begin{equation*}
    \Prob\roundbrack{Z_{m+1}\leq \quant_\beta \roundbrack{Z_{1:m}\cup \curlybrack{1}}}\geq \beta. \footnote{In this case, $\quant_\beta \roundbrack{Z_{1:m}\cup \curlybrack{1}}$ can be equivalently defined as the $\lceil \beta(m+1) \rceil$-th smallest element of the set $\curlybrack{Z_i}_{i=1}^m$ if $\beta\leq \frac{m}{m+1}$, and as $1$ otherwise.}
\end{equation*}
Moreover, if $Z_i$, $i=1,\dots,m+1$ are almost surely distinct, then the above probability is upper bounded by $\beta + \tfrac1{m+1}$.
\end{lemma}

\begin{lemma}[Lemma 3~\citep{tibs2019conf}]\label{lem:weight_quant_lemma}
Let $Z_{i}$, $i=1,\dots,n+1$ be weighted exchangeable random variables with weight functions $w_1$, $\dots$, $w_{n+1}$ and supported on $[0,1]$. Let $V_i=S\roundbrack{Z_i,Z_{-i}}$, where $Z_{-i}=Z_{1:(n+1)}\backslash \{Z_i\}$, $i=1,\dots,n+1$ and $S$ is an arbitrary score function. Define
\begin{equation}\label{eq:weight_coef}
    p_i^w(z_1,\dots,z_{n+1}) = \frac{\sum_{\sigma:\sigma(n+1)=i}\prod_{j=1}^{n+1}w_j(z_{\sigma(j)})}{\sum_{\sigma}\prod_{j=1}^{n+1}w_j(z_{\sigma(j)})},
\end{equation}
for $i=1,\dots,n+1$, where summations are taken over permutations $\sigma$ of $1,\dots,n+1$. Then for any $\beta\in (0,1)$,
\begin{equation*}
\begin{aligned}
      \Prob\roundbrack{V_{n+1} \leq  \quant_\beta\roundbrack{G_n}}\geq 1-\beta,
\end{aligned}
\end{equation*}
where the distribution $G_n$ is defined as
\begin{equation*}
    G_n :=  \sum_{i=1}^n p_i^w(Z_1,\dots,Z_{n+1})\delta_{V_i}+ p_{n+1}^w(Z_1,\dots,Z_{n+1})\delta_1.
\end{equation*}

\end{lemma}

\end{document}